\documentclass[final]{colt2026}

\usepackage{graphicx} 

\usepackage[utf8]{inputenc}
\usepackage{amsmath,amsfonts,amssymb}
\usepackage{mathtools}
\usepackage{stackrel}
\usepackage{enumitem}

\usepackage{tikz}
\usetikzlibrary{positioning}

\newcommand{\1}{\mathbf{1}}

\usepackage{color}

\newtheorem{question}{Question}

\usepackage[normalem]{ulem}
\usepackage{caption}

\newcommand{\bbE}{\mathbb{E}}
\newcommand{\bbP}{\mathbb{P}}
\newcommand{\bbR}{\mathbb{R}}
\newcommand{\pa}[1]{\left(#1\right)}
\newcommand{\ac}[1]{\left\{#1\right\}}
\newcommand{\cro}[1]{\left[#1\right]}

\title[Phase Transition in SBM]{Phase Transition for Stochastic Block Model with more than $\sqrt{n}$ Communities}
\usepackage{times}

\coltauthor{%
 \Name{Alexandra Carpentier} \Email{Alexandra.Carpentier@uni-potsdam.de}\\
 \addr Institut für Mathematik -- Universität Potsdam, Potsdam, Germany
 \AND
 \Name{Christophe Giraud} \Email{Christophe.Giraud@universite-paris-saclay.fr}\\
 \addr Laboratoire de Math\'ematiques d'Orsay, Universit\'e Paris-Saclay, CNRS, France%
 \AND
  \Name{Nicolas Verzelen} \Email{Nicolas.Verzelen@inrae.fr}\\
 \addr INRAE, Institut Agro, MISTEA, Univ. Montpellier, France
}

\begin{document}

\maketitle

\begin{abstract}
   Predictions from statistical physics postulate that recovery of the communities in the Stochastic Block Model (SBM) with a fixed number $K$ of communities is possible in polynomial time above, and only above, the Kesten-Stigum (KS) threshold. This conjecture has given rise to a rich literature,  proving that non-trivial community recovery is indeed possible in SBM above the KS threshold. Failure of low-degree polynomials (LDP) below the KS threshold was also proven, as long as $K\ll \sqrt{n}$, where $n$ is the number of nodes in the observed graph.
    
    When $K\geq \sqrt{n}$,~\cite{pmlr-v291-chin25a} recently proved that,  in a \emph{sparse regime},  community recovery  in polynomial time is possible below the KS threshold by counting non-backtracking paths. This breakthrough led them to postulate a new threshold for the many-communities regime $K\geq \sqrt{n}$. In this work, we provide evidence supporting their conjecture:\\
 1- We prove that, for \emph{any graph density}, LDP fail to recover communities below the threshold postulated by \cite{pmlr-v291-chin25a} ;\\
 2- We prove that community recovery is possible in polynomial time above the postulated threshold, not only in the \emph{sparse regime} considered in~\cite{pmlr-v291-chin25a}, but also in \emph{moderately sparse regimes}, by counting occurrences of some specific motifs inspired by the LDP analysis.\\
 In particular, counting  self-avoiding paths of length $\log(n)$—which is closely related to spectral algorithms based on the Non-Backtracking operator—is optimal only in the sparse regime. More complex motifs based on the blow-up of a cycle must be considered in denser regimes. 
\end{abstract}


\section{Introduction}

Network analysis  investigates random interactions among individuals or objects.  A typical network over $n$ individuals can be represented as an undirected graph with $n$ nodes, where edges encode the observed binary interactions between individuals. 
Examples of such networks include social networks (interactions being e.g., friendships, links, or followers), 
biological networks (e.g., gene---gene or gene---protein interactions), 
and information networks (e.g., email communication networks, internet, or citation networks).
The \emph{Stochastic Block Model} (SBM), introduced by Holland, Laskey, and Leinhardt~\cite{holland1983stochastic} is a popular statistical model for network analysis. In the SBM, the set of nodes is partitioned into $K < n$ disjoint groups, or \emph{communities}, according to a random assignment. Then, 
conditional on the (unobserved) community assignments, edges are generated independently, 
with probabilities determined solely by the community memberships of the nodes they connect. 
In its simplest version, the probability of connection is $p$ if both nodes belong to the same community, and $q$ otherwise. When $p>q$, nodes within the same community tend to exhibit stronger connectivity among themselves than with nodes outside their community. 
Because of its simple analytic structure, its ability to capture community formation, and its interesting 
theoretical properties, the SBM has become an iconic model for network data, and it has received significant attention from theoreticians in statistics, computer science and probability~\cite{AbbeReview2018}. 

The central task in SBM analysis is to recover the latent community memberships of the nodes 
from a single observed network. For a fixed  number of communities $K$, a community assignment uniformly at random,  a diverging population size $n$, and probabilities of connection $p$ and $q$ scaling as $1/n$, the seminal paper of Decelle, Krzakala, Moore and Zdeborov\'a~\cite{Decelle2011} conjectured with the replica heuristic from statistical physics, that, while non-trivial community recovery is possible above the information threshold
\begin{equation}\label{eq:inform}
{n\lambda^2 \over K\lambda  + K^2q} \gtrsim {\log(K)\over K},\quad \text{where}\ \lambda=p-q\ ;
\end{equation}
community recovery is possible in polynomial time  only above the Kesten–Stigum (KS) threshold
\begin{equation}\label{eq:KS}
\lambda>\lambda_{KS},\quad \text{with $\lambda_{KS}$ defined as the solution to}\quad
{n\lambda_{KS}^2 \over K\lambda_{KS}  + K^2q} =1\ .
\end{equation}
This conjecture then suggests that, for $K\geq 5$, there is an information-computation gap for community recovery in SBM in this sparse regime, which means that the minimal separation $\lambda$ required for community recovery in polynomial-time is larger than the minimal separation required without algorithmic constraints. This conjecture has raised a high interest among theoreticians, and a lot of efforts have been devoted to confirm this conjecture.

On the one hand, non-trivial recovery in polynomial time above the KS threshold was established in this setting in~\cite{massoulie2014community,mossel2018proof,bordenave2015non,AbbeSandon2015a,pmlr-v291-chin25a}, see~\cite{AbbeReview2018} for a more exhaustive review.
On the other hand, a strong effort was devoted to prove the impossibility to recover the communities in poly-time below the KS threshold. 
Due to the random nature of the observed graph in the SBM, the classical notions of worst-case hardness, such as P, NP, etc are not suitable for proving computational hardness in this framework. 
In such random settings, other notions are considered, such as reduction to other hard statistical problems, like planted clique~\cite{brennan2020reducibility,pmlr-v30-Berthet13,brennan2018reducibility}, or computational hardness in some specific models of computations, such as SoS \cite{HopkinsFOCS17,Barak19}, overlap gap property \cite{gamarnik2021overlap}, statistical query \cite{kearns1998efficient,brennan2020statistical}, and low-degree polynomials (LDP) \cite{hopkins2018statistical,KuniskyWeinBandeira,SchrammWein22,SohnWein25}. Among these, LDP bounds have gained attention for establishing state-of-the-art lower bounds in tasks like community detection~\cite{Hopkins17},  spiked tensor models~\cite{Hopkins17,KuniskyWeinBandeira}, sparse PCA~\cite{ding2024subexponential}, Gaussian clustering~\cite{Even25a}, and many other models~\cite{SurveyWein2025}. In this model of computation, only estimators that are multivariate polynomials of degree at most $D$ are considered. It is conjectured that, for many problems, failure of degree $D = O(\log n)$ polynomials enforces failure of any polynomial-time algorithm~\cite{KuniskyWeinBandeira} (LDP conjecture). 
The first LDP hardness results  for SBM were for the detection problem~\cite{Hopkins17,bandeira2021spectral,kunisky2024low}, which is the problem of testing $\lambda=0$ against $\lambda>0$.
LDP hardness below the KS threshold~\eqref{eq:KS} for community recovery was proven  in~\cite{luo2023computational, SohnWein25, ding2025low}.

The setting where the number $K$ of communities grows with $n$ has attracted much less attention until recently. A polynomial time algorithm performing non-trivial recovery close to the KS threshold was already proposed in~\cite{chen2014statistical}. Yet, for growing $K$, non-trivial community recovery in poly-time  just above the    KS threshold  was only obtained recently in \cite{stephan2024community,pmlr-v291-chin25a}. 
Lower bounds for growing $K$ have been derived in~\cite{luo2023computational, ding2025low,pmlr-v291-chin25a}, proving failure of LDP below the KS threshold~\eqref{eq:KS} when $K=o(\sqrt{n})$, both in  the sparse case discussed here ($q$ scaling like $1/n$), and in the denser case where probabilities of connection  grow faster than $1/n$. In particular, \cite{pmlr-v291-chin25a} proves that, conditional on the LDP conjecture, the prediction from \cite{Decelle2011} holds in the sparse case as long as $K=o(\sqrt{n})$. Indeed, in addition to proving  failure of LDP below the KS threshold in any regime, they also prove the existence of a poly-time algorithm based on non-backtracking statistics which succeeds to partially recover the communities above the KS threshold in the sparse regime (where $q$ scales like $1/n$).

The regime with a higher number of communities, $K\geq \sqrt{n}$ remains mostly not understood. Chin, Mossel, Sohn and Wein~\cite{pmlr-v291-chin25a} made an important breakthrough by proving that when $K\gg \sqrt{n}$, a poly-time algorithm based on non-backtracking statistics  succeeds to partially recover the communities below the KS threshold in the \emph{sparse regime} (where $q$ scales like $1/n$). A similar phenomenon had  already been observed in \cite{Even24} for the related problem of clustering in Gaussian mixture: When $K\geq \sqrt{n}$, a simple hierarchical clustering algorithm succeeds to recover the clusters below the KS threshold (of Gaussian mixture models), and it has been proved to be optimal  up to possible log factors~\cite{Even24,Even25a}.  For the SBM in the sparse regime  ($q$ scaling like $1/n$), \cite{pmlr-v291-chin25a} prove that, when $K\gg \sqrt{n}$, partial recovery of the communities is possible in polynomial time when
\begin{equation}\label{eq:new0}
\lambda \gtrsim_{\log} \pa{q+{\lambda \over K}}^{1-\log_{n}(K)},
\end{equation}
where $\log_{n}$ is the logarithm in base $n$ (i.e. $K=n^{\log_{n}(K)}$) and where $\gtrsim_{\log}$ hides a poly-log factor, which is proportional to the square of the natural logarithm of the average density. The Condition~\eqref{eq:new0} can be simplified as
\begin{equation}\label{eq:new}
\lambda \gtrsim_{\log} q^{1-\log_{n}(K)},
\end{equation}
for $q\geq 1/n$.
This upper-bound for the sparse regime led them to pose the following open questions (Question 1.9 in~\cite{pmlr-v291-chin25a}, that we transpose with our notation)
\begin{center}
{\it Let $\log_{n}(K)={1\over 2}+\Omega(1)$. Does the phase transition for efficient weak
recovery occur at $\lambda \gtrsim_{\log} \pa{q+{\lambda \over K}}^{1-\log_{n}(K)}$, where $\gtrsim_{\log}$  hides poly-log factors? Moreover, is there a sharp or coarse
transition in the recovery accuracy as $\lambda$ varies as a function of the average density?}
\end{center}
Let us (partially) rephrase these two questions for convenience of the discussion. When $K\gg \sqrt{n}$:
\begin{enumerate}\setlength{\itemsep}{0pt}
\item[Q1:]  Does the phase transition occur at \eqref{eq:new} in the sparse regime ($q$ scaling like $1/n$)?
\item[Q2:] Where does the phase transition occur in denser regimes?
\end{enumerate}
For the first question Q1, one may guess that the answer should be positive. Indeed,  non-backtracking statistics are tightly related to Belief-Propagation and the Bethe free-energy, the latter being known to be a valid approximation of the log-likelihood in the sparse regime, while the former is expected to be optimal again in the sparse regime~\cite{krzakala2013spectral,Moore2017}. Hence, as far as the analysis of~\cite{pmlr-v291-chin25a} is tight, we can expect that \eqref{eq:new} is the threshold where the phase transition occurs in the sparse regime. 

The answer to the second question Q2 seems much more open. First, the upper-bound of~\cite{pmlr-v291-chin25a} is for the sparse regime, and we are not aware of any result proving non-trivial recovery below the KS threshold in a denser regime. Second, even if the upper-bound of~\cite{pmlr-v291-chin25a} were to hold in a denser regime, there is no strong indication that non-backtracking statistics should be optimal in this denser regime. In particular, in the light of the results for the Gaussian clustering problem~\cite{Even24}, we may guess that some other algorithms could be more powerful in  denser regimes.\medskip

\subsection{Main contributions}
In this paper, we provide an answer to these two questions. We prove that, when $K\geq \sqrt{n}$
\begin{enumerate}[wide, labelindent=0pt]
\setlength{\itemsep}{0pt}
\item LDP fail for partial recovery when $\lambda \lesssim_{\log} q^{1-\log_{n}(K)}$, for all density regimes;
\item Community recovery in poly-time is possible in  moderately sparse regimes when  \eqref{eq:new} holds.
\end{enumerate}
These two results provide evidence that, when $K\geq \sqrt{n}$, the phase transition for community recovery in poly-time holds around the Threshold~\eqref{eq:new} conjectured by~\cite{pmlr-v291-chin25a}.
Let us present our two main contributions with more details.

\paragraph{Contribution 1: Low degree polynomial lower bound.}
Our first main result is a LDP lower bound for community recovery for any $K\leq n$. Low-degree polynomials are not well-suited for directly outputting a partition of the nodes, which is combinatorial by nature. Instead, as standard for clustering or community recovery~\cite{luo2023computational,Even24,SohnWein25,pmlr-v291-chin25a,Even25a}, we focus on the problem of estimating $x_{ij}=\mathbf{1}\ac{i,j\ \text{are in the same community}}-1/K$, which can be directly related to the problem of clustering~\cite{Even24,Even25a}. Theorem~\ref{thm:BI},
Appendix~\ref{sec:main}, states a precise version of the following.

\begin{theorem}[Informal]\label{thm:informal1}
When 
\vspace{-0.3cm}
\begin{equation}\label{eq:phase}
\lambda\lesssim_{\log} \lambda_{c},\quad \text{with $\lambda_{c}$ solution to}\quad \sup_{r\geq 1}\, {n\lambda_{c}^{2r}\over K\lambda_{c}^{r}+K^2 q^{r}}=1\ ,
\end{equation}
any $O(\log(n))$-degree polynomial $f$ fails to estimate $x_{ij}$ significantly better than $f=\bbE[x_{ij}]=0$.
\end{theorem}

At first sight, the Condition~\eqref{eq:phase} looks unrelated to the Threshold~\eqref{eq:new} identified by~\cite{pmlr-v291-chin25a}. 
Next lemma, proved in Appendix~\ref{sec:signal}, highlights 
 that the Condition~\eqref{eq:phase} is closely related both to the Kesten-Stigum threshold \eqref{eq:KS} when $K\leq \sqrt{n}$ and the Threshold~\eqref{eq:new} when $K\geq \sqrt{n}$.  

\begin{lemma}\label{lem:rstar}
Assume that $q=n^{-\alpha_{n}}$ with $\alpha_{n}\in(0,+\infty)$, and $K=n^{1+\delta_{n}\over 2}$ ---i.e. $\log_{n}(K)=(1+\delta_{n})/2$--- with $\delta_{n}\in(-1,1)$. Define 
$\bar r_{n}$ as the value $r\geq 1$ achieving the maximum in~\eqref{eq:phase}.
Then, for any small $\epsilon>0$, and $n$ large 
\begin{enumerate}
\item If $K\leq (1-\epsilon)\sqrt{n}$, then $\bar r_{n}=1$ and 
$\lambda_{c}=\lambda_{KS}:={K\over 2n} \pa{1+\sqrt{1+2nq}}.$
\item If $(1+\epsilon)\sqrt{n}\leq K\leq n^{1-\epsilon}$, then
\begin{itemize}[wide, labelindent=0pt]
\item If $q\leq {1-\delta_{n}^2\over 4\delta_{n}^2}\times{1-o(1)\over n}$, then $\bar r_{n}=1$ and $\lambda_{c}=\lambda_{KS}$;
\item $q\geq {1-\delta_{n}^2\over 4\delta_{n}^2}\times{1-o(1)\over n}$, then \ 
$\displaystyle{\bar r_{n}\sim \alpha_{n}^{-1} \pa{1-{\log\pa{1-\delta_{n}^2\over 4 \delta_{n}^2}\over \log(n)}-O\pa{1\over (\log (n))^2}}}$ and
\begin{align}\label{eq:lambdac}
\lambda_{c} &\sim q^{(1-\delta_{n})/2} \pa{(1+\delta_{n})^{1+\delta_{n}}\over (1-\delta_{n})^{1-\delta_{n}}(2\delta_{n})^{2\delta_{n}}}^{\alpha_{n}/2}.
\end{align}
\end{itemize}
\end{enumerate}
\end{lemma}

When $K\leq \sqrt{n}$, we have $\lambda_{c}=\lambda_{KS}$, and we recover that low-degree polynomials fail to estimate $x_{ij}$ better than 0 
below the Kesten-Stigum threshold, as already proved in~\cite{pmlr-v291-chin25a}. 
When $K\geq \sqrt{n}$ and $q\gtrsim n^{-1}$, then the threshold value~\eqref{eq:lambdac} matches the Threshold~\eqref{eq:new} since $(1-\delta_{n})/2=1-\log_{n}(K)$.

 
 Theorem~\ref{thm:informal1} establishes that the Threshold~\eqref{eq:new} identified by~\cite{pmlr-v291-chin25a} is correct, up to a poly-logarithmic factor, in the sparse regime $q\asymp 1/n$ with $K\gg \sqrt{n}$. Moreover, it serves as a lower bound on the minimal separation required for polynomial-time recovery in denser regimes. Our second main contribution is to leverage our LDP analysis for deriving novel optimal motif counts. These counts enable polynomial-time community recovery in the intermediate dense regime $q\asymp n^{-\alpha}$ (for $\alpha\in(0,1)\cap \mathbb{Q}$), provided that $K\geq \sqrt{n}$ and~\eqref{eq:new} are satisfied.

\paragraph{Contribution 2: Successful recovery above threshold $\lambda_{c}$ with motif counts.}
\cite{pmlr-v291-chin25a} prove that, in the sparse regime ($q$ scaling like $1/n$), partial recovery of the communities is possible above the threshold~\eqref{eq:new}  with non-backtracking statistics, when $K\gg \sqrt{n}$. Our second contribution --- Corollaries~\ref{cor:cliques} ;~\ref{cor:SAP} and~\ref{cor:blow-up}  Appendix~\ref{sec:positive} ---  is to prove that partial recovery of the communities is possible above the Threshold~\eqref{eq:lambdac}, with an  algorithm based on (weighted) motif counting.

\begin{theorem}[Informal]\label{thm:informal2}
Assume that $q\asymp n^{-\alpha}$ with $\alpha\in(0,1)\cap \mathbb{Q}$. When $\lambda\gtrsim_{\log} \lambda_{c}$, with $\lambda_{c}$ given in~\eqref{eq:lambdac},
community recovery is possible in polynomial time with an algorithm based on counting a motif, whose shape depends on $\alpha$. 
\end{theorem}

Combining Theorem~\ref{thm:informal1} with Theorem~\ref{thm:informal2} ---and the result of~\cite{pmlr-v291-chin25a} for the case $\alpha=1$---, we get that for  $q\asymp n^{-\alpha}$ with $\alpha\in (0,1]\cap \mathbb{Q}$,  the phase transition for weak recovery in poly-time occurs at level $\lambda_{c}$, up to possible $\log$ factors.  When $\alpha$ is irrational, it follows from our results that community recovery is possible either above the threshold $\lambda_{c}$ with a polynomial of degree $\log(n)$ ---but with $O(n^{\log(n)})$ algorithmic complexity---, or in polynomial time if $\lambda\geq n^{\epsilon} \lambda_{c}$, where $\epsilon>0$ can be arbitrarily small. See the discussion after Corollary~\ref{cor:blow-up} (page \pageref{cor:blow-up}) for details. 

As mentioned in  Theorem~\ref{thm:informal2}, the motifs involved in the log-optimal algorithm depends on the power $\alpha$ of the between group probability of connection $q=n^{-\alpha}$.
For some very specific values of $\alpha$, some log-optimal motifs are very simple. Indeed, we prove in Corollaries~\ref{cor:cliques} and~\ref{cor:SAP}, Appendix~\ref{sec:positive}, the following.

\begin{corollary}[Informal]\label{cor:cliques-SAP}
Assume that $K\geq \sqrt{n}$ and  $\lambda \gtrsim_{\log} \lambda_{c}$.\vspace{-0.2cm}
\begin{enumerate}[wide, labelindent=0pt]\setlength{\itemsep}{0pt}
\item When $q\asymp n^{-{2\over m+1}}$, with $m\in \ac{3,4,5,\ldots}$, community recovery is possible in polynomial time with an algorithm based on counting cliques of size $m$.
\item When  $q\asymp n^{-{m-2\over m-1}}$, with $m\in \ac{3,4,5,\ldots}$, community recovery is possible in polynomial time with an algorithm based on counting self-avoiding paths of length $m-1$.
\end{enumerate}
\end{corollary}

\begin{figure}[h!]
  \centering
  \includegraphics[width=0.9\linewidth]{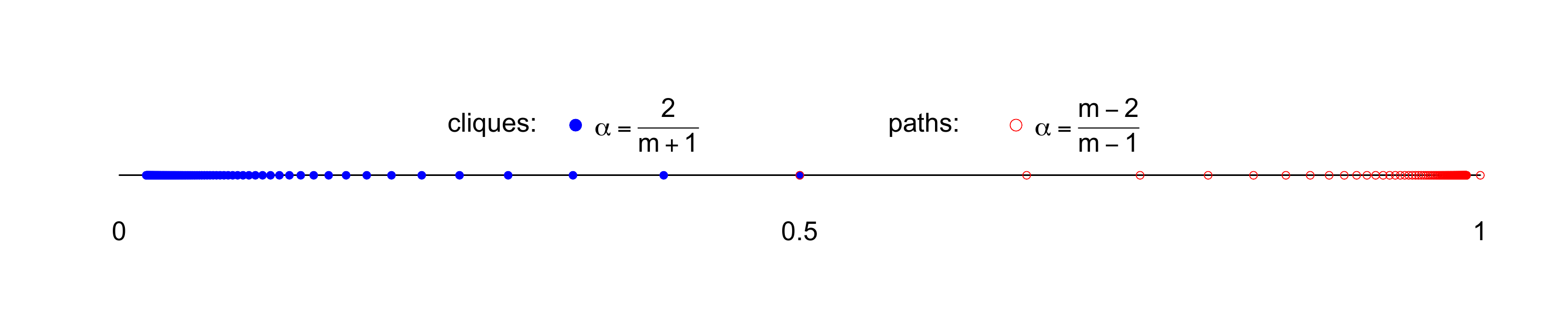}
  \vspace{-0.5cm}
\caption{Values of $\alpha\in \big\{{2\over m+1}:m=3,4,\ldots\big\}\cup\big\{{m-2\over m-1}:m=3,4,\ldots\big\}$ corresponding to densities $q\asymp {n^{-\alpha}}$ for which $m$-cliques counting (blue solid markers) and $(m-1)$-self-avoiding paths counting (red circles) succeed at the (log-)optimal threshold  $\lambda \gtrsim_{\log} \lambda_{c}$.}
\end{figure}

We notice that taking $m$ slowly diverging\footnote{When $m$ diverges with $n$, counting length-$(m-1)$ self-avoiding paths has a complexity super-polynomial in $n$. This issue can be overcome by counting non-backtracking paths instead~\cite{pmlr-v291-chin25a}, which can be done efficiently.} with $n$, we recover the success of self-avoiding paths counting above the threshold~\eqref{eq:new} when $q\asymp {1\over n}$, as already proved in~\cite{pmlr-v291-chin25a}. We underline yet, that counting length-$\log(n)$  self-avoiding paths ---which is closely related to spectral algorithms based on the Non-Backtracking operator--- is optimal only in the sparse regime $q\asymp {1\over n}$. Other motifs counts should be considered in Theorem~\ref{thm:informal2} in denser regimes. 

For general values of $\alpha$, our LDP analysis dictates a more complex family  of (log-)optimal motifs. 
Writing $\alpha^{-1}= \gamma + a  $, where $\gamma\geq 1$ is an integer  and $a$ is a rational in $(0,1)$, the motif is constructed by considering a $\gamma$-blow-up of a cycle of length $\kappa$, which is then linked evenly to the reference nodes $i$ and $j$ with  $a\kappa\gamma$ edges ---referred to as fasteners in the following. The length $\kappa$ of the cycle is chosen in such a way that $a\kappa\gamma$ is a sufficiently  large even integer. We refer to the discussion Section~\ref{sec:intro-motif} for some intuition on the choice of this motif, and to Section~\ref{sec:blow-up:def} page~\pageref{sec:blow-up:def} for the formal description.

It is worth mentioning that counting self-avoiding paths of length 2 or cliques of size $m=3$ is closely related to the hierarchical clustering algorithm in~\cite{Even24}, which is log-optimal for Gaussian mixture when $K\gg \sqrt{n}$. Yet, counting $3$-cliques is only log-optimal for SBM in a narrow regime of density where $q\asymp n^{-1/2}$. Indeed, contrary to the Gaussian case, the minimal separation $\lambda$ remains bounded, and in denser regimes, better results can be obtained by averaging the observations over a wider $m$-clique, to reduce the variance. 

\paragraph{Open questions.}
While our results, combined with those of~\cite{pmlr-v291-chin25a}, provide a  better understanding of the phase transition for community recovery when $K\geq \sqrt{n}$, some questions remain open.
A first question is relative to the exact location of the phase transition when $K\geq \sqrt{n}$. 
Indeed, while~\cite{pmlr-v291-chin25a} prove that the phase transition occurs exactly at $\lambda_{c}=\lambda_{KS}$ when $K\leq \sqrt{n}$ (and $q\asymp 1/n$ for the upper bound),   the upper and lower bounds for $K\geq \sqrt{n}$ involve (un-matching) poly-log factors.

\begin{question}\label{Q1}
Is the threshold $\lambda_{c}$ defined in~\eqref{eq:phase} ---equivalently~\eqref{eq:lambdac}--- the \emph{exact} threshold above which weak recovery is possible in polynomial time?
\end{question}

A second question is relative to optimal motifs for successful recovery. Indeed, while we show that, when $\lambda\gtrsim_{\log} \lambda_{c}$,  we can recover the communities with the family of blow-up  motifs introduced in Section~\ref{sec:blow-up:def}, the question of sharp optimality (in terms of  log factors) remains open. 

\begin{question}\label{Q2}
What are the optimal poly-time algorithms for weakly recovering the communities above the computational barrier (with sharp constants)?
\end{question}


\subsection{Proof technique, insights, and limitation}\label{sec:overview}
Let us denote by $Y^{*}\in\bbR^{n\times n}$ the adjacency matrix of the observed graph on $n$ vertices.
The graph being undirected and with no self-loops, the matrix $Y^*$ is symmetric with zero on the diagonal. We assume that $Y^*$ has been generated according to the stochastic block model: Each node $i$ is assigned uniformly at random to one among $K$ communities, and then nodes within a same community are connected with probability $p$, and those between different communities with probability $q$.

More formally, let $K$ be an integer in $[2,n]$, and $p,q\in (0,1)$, with $p>q$. We assume that $Y^*$ is sampled according to the following distribution:
\vspace{-0.2cm}
\begin{enumerate}[wide,labelindent=0mm]
\setlength{\itemsep}{0pt}
\item  $z_{1},\ldots,z_{n}\stackrel{i.i.d.}{\sim} \text{Uniform}\ac{1,\ldots,K}$;
\item Conditionally on $z_{1},\ldots,z_{n}$, the entries $(Y^*_{ij})_{i<j}$ are sampled independently, with $Y^*_{ij}$ distributed as a Bernoulli random variable with parameter $q+\lambda{\bf 1}\ac{z_{i}=z_{j}}$, where $\lambda=p-q$. 
\end{enumerate}
To get simpler formulas, in the remaining of the paper, we will work with the (almost)-recentered adjacency matrix
\begin{equation}\label{eq:definition:Y}
Y_{ij} = Y^{*}_{ij}- q \ , \text{ for any }1\leq i < j \leq n.
\end{equation}
We will also recurrently use the notation
\begin{equation}\label{eq:bar-proba}
\bar q=q(1-q),\quad \text{and}\quad \bar p= \bar q +\lambda (1-2q).
\end{equation}

\subsubsection{Low-degree lower bound}
Let $\bbP_{\lambda}$ denote the distribution of the SBM with a prescribed separation $\lambda$. Let us consider the problem of estimating $x_{12}=\mathbf{1}\ac{z_1=z_2}-1/K$. 
We say that degree-$D$ polynomials fail to estimate $x_{12}$ in $L^2(\bbP_{\lambda})$-norm,  when no degree-$D$ polynomial can estimate $x_{12}$ significantly better than the constant estimate $\bbE_{\lambda}[x_{12}]=0$. In mathematical language, proving  failure of degree-$D$ polynomials for estimating $x_{12}$ amounts to prove that
\begin{equation}\label{eq:MMSE}
MMSE_{D}:=\inf_{f: \mathrm{deg}(f)\leq D} \mathbb{E}_{\lambda}\cro{(f-x_{12})^2}=\mathbb{E}_{\lambda}[x_{12}^2]- \sup_{f: \mathrm{deg}(f)\leq D}\frac{\mathbb{E}_{\lambda}[fx_{12}]^2}{\mathbb{E}_{\lambda}[f^2]}=(1+o(1))MMSE_{0},
\end{equation}
where the $\inf$ and $\sup$ are over all the polynomials of the observation (the adjacency matrix in our case) with degree at most $D$. Since $MMSE_{0}=\bbE_{\lambda}[x_{12}^2]-\bbE_{\lambda}[x_{12}]^2=\bbE_{\lambda}[x_{12}^2]$, the problem reduces to proving that
\begin{equation}\label{eq:corr:goal}
   \mathrm{Corr}_{\leq D}^2 :=   \sup_{f: \mathrm{deg}(f)\leq D}\frac{\bbE_{\lambda}\cro{x_{12}f}^2}{{\bbE_{\lambda}\cro{f^2}}} \leq  o\pa{\bbE_{\lambda}\cro{x_{12}^2}}=o\pa{1\over K}.
\end{equation}

Low-degree polynomials theory has emerged in a sequence of works~\cite{HopkinsFOCS17,Hopkins17,hopkins2018statistical,KuniskyWeinBandeira} focusing on detection problems, where the goal is to perform a test with a simple null distribution (typically with independent entries). For SBM, the detection problem amounts to test $\lambda=0$ against  $\lambda>0$.  In practice, it reduces to proving $ \mathrm{Corr}_{\leq D}^2=o(1)$, with $\lambda=0$ and $x_{12}$ replaced by the likelihood ratio between $\bbP_{\lambda}$ and $\bbP_{0}$.  
The core strategy for detection problems is to expand the polynomials $f$ on a $L^2\pa{\bbP_{0}}$-orthonormal basis of low-degree polynomials, and then to solve the resulting optimization problem explicitly. This approach has been successfully implemented for SBM in~\cite{Hopkins17}. A similar approach cannot be implemented for estimation (i.e.\ recovery) problems, as no simple explicit $L^2\pa{\bbP_{\lambda}}$-orthonormal basis of low-degree polynomials is known for $\lambda>0$. Three strategies have been proposed to circumvent this issue:
\begin{enumerate}[wide, labelindent=0pt]\setlength{\itemsep}{0pt}
\item Schramm and Wein~\cite{SchrammWein22} schematically apply a partial Jensen inequality with respect to the latent variables (the community assignment in SBM), and expand the polynomials $f$ on a $L^2\pa{\bbP_{0}}$-orthonormal basis. Then, the optimization problem~\eqref{eq:corr:goal} can be upper-bounded recursively. This approach has been implemented by Luo and Gao~\cite{luo2023computational} for SBM.
\item To get tighter bounds, Sohn and Wein~\cite{SohnWein25} expand the polynomials $f$ over a basis which is orthonormal in an enlarged probability space including the latent variables. While the analysis is much more delicate, it leads to sharp results. It has been implemented for SBM in~\cite{SohnWein25,pmlr-v291-chin25a} to prove failure of low-degree polynomials below the KS threshold~\eqref{eq:KS} when $K=o(\sqrt{n})$. 
\item To provide a simple and more direct analysis, Carpentier, Giancola, Giraud and Verzelen~\cite{CGGV25} build a basis for permutation invariant polynomials which is almost $L^2\pa{\bbP_{\lambda}}$-orthonormal when the signal is low, i.e.\ typically when recovery is impossible. Compared to~\cite{SohnWein25}, a counterpart of the simplicity of the analysis, is the derivation of less tight lower bounds. Indeed, similarly as for the Schramm and Wein~\cite{SchrammWein22} approach, bounds derived from the approach of~\cite{CGGV25}  are typically tight only up to poly-logarithmic factors.
\end{enumerate}

To prove Theorem~\ref{thm:informal1},  we leverage the ideas introduced in~\cite{CGGV25}, that we adapt to our setting.
 We adjust the basis proposed in~\cite{CGGV25} by using a better renormalisation. On top of the proof strategy of~\cite{CGGV25}, the main technical result is Proposition~\ref{prop:christophe}, which controls the correlation between two basis elements.
We highlight below the main steps and ideas. All the details are provided in the appendices. We recall that we consider the problem of estimating $x_{12}={\bf 1}\ac{z_{1}=z_{2}}-{1\over K}$. 

\paragraph{Almost orthonormal basis for permutation-invariant polynomials.}
A first key idea is to notice that, since the distribution of $x_{12}$ and $Y$ is invariant by permutation of the indices --- except 1,2--- the maximizer of $\mathrm{Corr}^2_{\leq D}$ ---defined by \eqref{eq:corr:goal}---  must be a polynomial  invariant by permutation of the indices,  except 1,2. Hence, we only need to focus on such polynomials, we call \emph{invariant polynomials}.

We can build a simple family spanning the set of degree-$D$ invariant polynomials. Let $G=(V,E)$ be a graph on a node set $V=\ac{v_{1},\ldots,v_{D}}$, and 
$\Pi_{V}$ be the set of injections $\pi:V\to \ac{1,\ldots,n}$ fulfilling $\pi(v_{1})=1$ and $\pi(v_{2})=2$. Then, any degree-$D$ invariant polynomials can be decomposed on the family $(P_{G})_{G}$ where
$$P_G = \sum_{\pi\in \Pi_V} P_{G,\pi},\quad \text{with} \quad P_{G,\pi}(Y)= \prod_{(i,j)\in E} Y_{\pi(i),\pi(j)} .$$
We observe that, when applied to the vanilla adjacency matrix $Y^*$, the polynomial $P_{G}(Y^*)$ merely counts  the number of occurrence of the motif $G$ ``attached'' at 1,2. The family $(P_{G})_{G}$ itself is not an almost-orthonormal basis of degree-$D$ invariant polynomials. First, it is overcomplete, so we need to extract a maximal free subfamily $(P_{G})_{G\in \mathcal{G}_{\leq D}}$. Second, $(P_{G})_{G\in \mathcal{G}_{\leq D}}$ is neither almost-orthogonal, nor almost-normalized. To get almost-orthogonality, we apply some correction as in~\cite{CGGV25} (see Appendix~\ref{sec:basis}), and to get normality, we rescale the resulting basis by some normalizing factor, providing an almost-orthonormal basis $(\Psi_{G})_{G\in \mathcal{G}_{\leq D}}$. Improving the normalization process compared to~\cite{CGGV25} is decisive for deriving our results, as explained below.

\paragraph{Normalization of the polynomials.} 
For simplicity, let us drop the correction applied to $P_{G}$ to achieve approximate orthogonality. Our goal is then to normalize the polynomial $P_{G}$ by its standard deviation $\sqrt{{\rm var}(P_{G})}$. The second moment of $P_{G}$ can be expanded as
\begin{equation}\label{eq:moment1-init}
\bbE\cro{P_{G}^2}= \sum_{\pi^{(1)},\pi^{(2)}\in \Pi_{V}} \bbE\cro{P_{G,\pi^{(1)}}P_{G,\pi^{(2)}}}.
\end{equation}
Assume that the dominating terms in this sum are those for which $\pi^{(1)}(G)=\pi^{(2)}(G)$, where $\pi(G)=(\pi(V),\ac{(\pi(u),\pi(v)):(u,v)\in E})$ represents the graph $G$ with nodes labeled by $\pi$. This yields the simplification
\begin{equation}\label{eq:moment2-init}
\bbE\cro{P_{G}^2}\asymp \sum_{\pi \in \Pi_{V}} |{\rm Aut}(G)|\, \bbE\cro{P_{G,\pi}^2}\asymp n^{|V|} |{\rm Aut}(G)|\, \bbE\cro{P_{G,\pi}^2},
\end{equation}
where $|{\rm Aut}(G)|$ denotes the number of automorphisms of $G$. 
While \cite{CGGV25} uses $q^{|E|}$ as a simple proxy for the second moment $\bbE[P_{G,\pi}^2]$ to greatly simplify their analysis, this choice is not precise enough to derive tight bounds in our context.

To understand why the proxy $q^{|E|}$ is insufficient, let us consider a clique $G=(V,E)$, for which approximate computations can be performed explicitly. Let $|E^{\neq}|$ denote the number of edges in $\pi(G)$ between nodes of distinct communities, and let $\ell(G)$ denote the number of distinct communities in $\pi(G)$. Direct computation yields $\bbP[\ell(G)=\ell]\asymp K^{\ell-|V|}$ for $1\leq \ell\leq |V|$ and
\begin{align*}
\bbE\cro{P_{G,\pi}^2} &\asymp \lambda^{|E|} \bbE\cro{\pa{q\over \lambda}^{|E^{\neq}|}} \ \asymp \ \sum_{\ell=1}^{|V|} {\lambda^{|E|}\over K^{|V|-\ell}}\ \bbE\cro{\pa{q\over \lambda}^{|E^{\neq}|}\Big| \ell(G)=\ell}\\
&\asymp \sum_{\ell=1}^{|V|} {\lambda^{|E|}\over K^{|V|-\ell}}\ \pa{q\over \lambda}^{(\ell-1)(|V|-\ell/2)},
\end{align*}
where the last equivalence follows because $q<\lambda$, and, conditioned on $\ell(G)=\ell$, the exponent $|E^{\neq}|$ is minimized when all but $\ell-1$ vertices belong to the same community. Since the overall exponent in the sum is quadratic in $\ell$, the sum is typically dominated by its first ($\ell=1$) or last ($\ell=|V|$) terms, leading to
\begin{equation}\label{eq:moment3-init}
\bbE\cro{P_{G,\pi}^2} \asymp {\lambda^{|E|}\over K^{|V|-1}}+q^{|E|}.
\end{equation}
Hence, for a clique $G$, we observe an additional term $\lambda^{|E|}/ K^{|V|-1}$ compared to the proxy $q^{|E|}$, which cannot be neglected in our setting. 

Consequently, instead of normalizing the polynomials $P_{G}$ by the proxy $n^{|V|/2} |{\rm Aut}(G)|^{1/2}\, q^{|E|/2}$ as in~\cite{CGGV25}, we apply an implicit normalization by $n^{|V|/2} |{\rm Aut}(G)|^{1/2}\bbE[P_{G,\pi}^2]^{1/2}$. The price to pay for this nearly exact normalization is a significantly more involved analysis to prove the approximate orthonormality of the basis $(\Psi_{G})_{G\in \mathcal{G}_{\leq D}}$, particularly when controlling the correlation terms.


\paragraph{Controlling correlation terms.}
Since the polynomials are now normalized by $\bbE[P_{G,\pi}^2]^{1/2}$, a central challenge in proving the approximate orthonormality of the basis $(\Psi_{G})_{G\in \mathcal{G}_{\leq D}}$ is to tightly upper-bound the absolute value of the correlation
$${\bbE\cro{P_{G^{(1)},\pi^{(1)}}P_{G^{(2)},\pi^{(2)}}}\over
\sqrt{\bbE\cro{P_{G^{(1)},\pi^{(1)}}^2} \bbE\cro{P_{G^{(2)},\pi^{(2)}}^2}}}$$
between $P_{G^{(1)},\pi^{(1)}}(Y)$ and $P_{G^{(2)},\pi^{(2)}}(Y)$ for two given motifs $G^{(1)},G^{(2)}$ and two injections $\pi^{(1)},\pi^{(2)}\in\Pi_{V}$. 
This correlation is upper-bounded in Proposition~\ref{prop:christophe}, which forms the cornerstone of our analysis. With this bound in hand, the proof of the approximate orthonormality of $(\Psi_{G})_{G\in \mathcal{G}_{\leq D}}$ follows a path similar to that in~\cite{CGGV25}.

It is worth noticing that for proving success of counting a motif $G$, like e.g. a clique of size $m$, a key feature is to upper-bound a variant of the correlation
${\bbE[P_{G,\pi^{(1)}}P_{G,\pi^{(2)}}]/ \bbE[P_{G,\pi^{(1)}}]^2}$,
see Appendix~\ref{sec:S}. So, understanding the scalar product terms $\bbE[P_{G^{(1)},\pi^{(1)}}P_{G^{(2)},\pi^{(2)}}]$ appears to be crucial for understanding the nature of the phase transition in SBM. 

\paragraph{Identifying the optimal polynomials.}
Once we have proven that $(\Psi_{G})_{G\in \mathcal{G}_{\leq D}}$ is almost-orthonormal, we have
$$\mathrm{Corr}^2_{\leq D} \approx \sum_{G\in \mathcal{G}_{\leq D}} \mathbb{E}[x_{12}\Psi_G]^2.$$
Furthermore, setting $\bar q=q(1-q)$, we prove in Appendix~\ref{sec:proof:BI} that 
\begin{equation}\label{eq:optimalG}
 \mathbb E[x_{12}\Psi_G]^2     \leq   \frac{1}{n}\left[\frac{n \lambda^{2r(G)}}{(K\lambda^{r(G)}) \vee (K^2 \bar q^{r(G)})}\right]^{|V|-1},\qquad\text{with}\quad r(G):={|E|\over |V|-1}\ ,
\end{equation}
when $G$ is connected, and $\bbE[x_{12}\Psi_G]=0$ otherwise. 
Hence $\mathrm{Corr}^2_{\leq D}=O(1/n)$ when $\lambda \lesssim_{D} \lambda_{c}$, with $\lambda_{c}$ defined in~\eqref{eq:phase}.


\subsubsection{Optimal community recovery with motif counts.}\label{sec:intro-motif}
In addition to provide a simple and direct proof of the lower-bound, the construction of the almost orthonormal basis allows to get insights on the best degree-$D$ polynomials, and hence on the optimal poly-time procedures.
We underline that, schematically, 
\begin{enumerate}[wide, labelindent=0pt]\setlength{\itemsep}{0pt}
\item the Bound~\eqref{eq:optimalG} is valid only when the bracket in the right-hand side is small compared to 1;   
\item non-trivial recovery becomes possible when this bracket becomes larger than 1 for some $G$.
\end{enumerate}
The bound~\eqref{eq:optimalG} then provides some interesting information on the optimal motifs $G$. Indeed, taking for granted that~\eqref{eq:optimalG} is quite tight below the Threshold~\eqref{eq:new} --- it should be so---,  we observe that the motifs $G$ maximizing  $\mathbb{E}[x_{12}\Psi_G]^2$ are those being connected and maximizing
$$\left[\frac{n \lambda^{2r(G)}}{(K\lambda^{r(G)}) \vee (K^2 \bar q^{r(G)})}\right]^{|V|-1},\qquad\text{with}\quad r(G):={|E|\over |V|-1}\ .$$
If this remains true at the phase transition, where this quantity is close to 1, it provides insights on the choice of the motif $G$ for optimal recovery.

When $K\geq \sqrt{n}$ and $q\asymp n^{-\alpha}$,  Lemma~\ref{lem:rstar}  ensures that
 the value $\bar r$ achieving the maximum in~\eqref{eq:phase} fulfills $\bar r\sim (\alpha^{-1}\vee 1)$. 
Accordingly,
let us focus on the density regime $q\asymp n^{-1/r}$ with $r> 1$, and with signal condition $\lambda \geq_{\log} \lambda_{c}\asymp q^{1-\log_{n}(K)}$. In particular, 
\begin{equation}\label{eq:cond:lambda/q}
\lambda^r\geq_{\log} q^r K\asymp {K\over n}.
\end{equation}
Equation~\eqref{eq:optimalG} suggests to count the occurrence of a connected motif $G$ fulfilling 
$r=|E|/(|V|-1)$, where $r$ is such that $q\asymp n^{-1/r}$. In fact, since our objective is to test whether 1 and 2 belong to the same community rather than to estimate $x_{12}$, this ratio must be slightly adjusted to (see Equation~\eqref{eq:E12:intro} below)
\begin{equation}\label{def:r}
r={|E|\over |V|-2}\ .
\end{equation}
Introducing the notation $\bbP_{12}:=\bbP\cro{\cdot|z_{1}=z_{2}}$ and $\bbP_{\not 12}:=\bbP\cro{\cdot|z_{1}\neq z_{2}}$,
a direct computation gives $\bbE_{\not 12}\cro{P_{G}}=0$. So,
for testing  $z_{1}=z_{2}$ against $z_{1}\neq z_{2}$ with a small error, we seek for a connected motif $G$ fulfilling
\begin{equation}\label{eq:test:success}
\bbE_{12}\cro{P_{G}}^2 \gg  \mathrm{var}_{12}(P_{G}) \vee  \mathrm{var}_{\not 12}(P_{G}).
\end{equation}
The expectation can be simply evaluated (see Equations~\eqref{eq:comp} and~\eqref{eq:card-label})
\begin{equation}\label{eq:E12:intro}
\bbE_{12}\cro{P_{G}} \asymp \pa{n \lambda^{|E|/(|V|-2)}\over K}^{|V|-2}=\pa{n \lambda^{r}\over K}^{|V|-2},
\end{equation}
when $G$ is connected and \eqref{def:r} holds.
So, according to~\eqref{eq:cond:lambda/q}, we have $\bbE_{12}\cro{P_{G}}\gtrsim 1$. 
For analyzing the variance terms, we proceed as in~(\ref{eq:moment2-init}).
We recall that $|E^{\neq}|$ denotes the number of edges in $\pi(G)$  between nodes of distinct communities, 
$\bar q=q(1-q)$ and $\bar p= \bar q +\lambda (1-2q)$, with $\bar p\sim \lambda$ since $\lambda \geq_{\log} K^{1/r} q$ according to~\eqref{eq:cond:lambda/q}.
A direct computation gives (see Equation~\eqref{eqn:old basis proof ingredients})
\begin{align}
\bbE_{12}\cro{P_{G,\pi}^2} & =\bbE_{12}\Bigg[\prod_{\substack{(i,j)\in \pi(E) \\ z_{i} =z_{j}}}\bar p\prod_{\substack{(i,j)\in \pi(E) \\ z_{i} \neq z_{j}}}\bar q\Bigg]
\ \asymp\ \lambda^{|E|}\ \bbE_{12}\cro{\pa{q\over \lambda}^{|E^{\neq}|}}.\label{eq:moment2-intro}
\end{align}
We expect that the leading terms in the variance decomposition 
$\displaystyle{\mathrm{var}_{12}(P_{G})=\sum_{\pi,\pi'} \mathrm{cov}_{12}(P_{\pi,G},P_{\pi',G})}$,
are those for which $\pi(G)=\pi'(G)$.
 Writing $|\mathrm{Aut}(G)|\leq |V|^{|V|}$ for the number of automorphisms of $G$, and recalling \eqref{eq:moment2-intro} and $|E|=r(|V|-2)$, we then expect to have
\begin{align}
\mathrm{var}_{12}(P_{G})\vee  \mathrm{var}_{\not 12}(P_{G})&\stackrel{?}{\lesssim} \sum_{\pi\in \Pi_{V}} |\mathrm{Aut}(G)|\,\mathrm{var}_{12}(P_{G,\pi}) 
\ \lesssim\ \pa{n\lambda^r}^{|V|-2}\ \bbE_{12}\cro{(q/\lambda)^{|E^{\neq}|}},\label{eq:basic:var}
\end{align}
where $\lesssim$  hides some factors depending on $|E|$ and $|V|$ only. 

We recall that $\ell(G)$ denotes the number of distinct communities in $\pi(G)$.
Combining~\eqref{eq:basic:var} with~\eqref{eq:cond:lambda/q} and $\bbP_{12}[\ell(G)=\ell]\asymp K^{-(|V|-1-\ell)}$, we then get
\begin{align}
\mathrm{var}_{12}(P_{G})\vee  \mathrm{var}_{\not 12}(P_{G})&\lesssim  \pa{n\lambda^r\over K}^{|V|-2}\ \sum_{\ell=1}^{|V|-1} K^{\ell-1} \bbE_{12}\cro{K^{-|E^{\neq}|/r}\Big|\ell(G)=\ell}. \label{eq:V12:intro}
\end{align}
Taking for granted that the upper-bound in \eqref{eq:V12:intro} is valid up to constants or log factors, comparing~\eqref{eq:V12:intro} and~\eqref{eq:E12:intro},  a test based on $P_{G}$ 
will fulfill~\eqref{eq:test:success}, as soon as the sum in~\eqref{eq:V12:intro} remains small. This will occur in particular if the motif $G$ fulfills 
\begin{equation}\label{eq:Eneq}
|E^{\neq}| \geq r(\ell-1)\ ,
\end{equation} 
for any partition of the nodes of $G$ into $\ell$ communities, with $v_{1},v_{2}$ in the same community.

Cliques (after removing the edge $(v_{1},v_{2})$) and self-avoiding paths (SAP) between $v_{1}$ and $v_{2}$, are two families of motifs fulfilling the Condition~\eqref{eq:Eneq}. 
To prove Corollary~\ref{cor:cliques-SAP}, we provide a rigorous proof of the heuristic sketched above, and  we combine cliques/SAP counting with a Median-of-Means (MoM) post-processing step. The benefit of the MoM post-processing step is that we only need to control the expectation and the variance of the clique/SAP counting --- see Lemma~\ref{lem:cliques} and~\ref{lem:paths} --- in order to get good enough concentration inequalities. The main drawback is that we lose a constant factor in the variance, so we cannot expect our analysis to provide upper bounds with sharp constants.

Cliques and SAP counting only  cover a small subset of the exponents $\alpha\in(0,1)\cap \mathbb{Q}$. 
To prove Theorem~\ref{thm:informal2} for any rational $r>1$: 
\begin{itemize}[wide, labelindent=2pt, nosep]\setlength{\itemsep}{0pt}
\item we construct in Section~\ref{sec:blow-up:def} a connected motif $G$ ---based on the blow-up of a cycle---  fulfilling the Conditions~\eqref{eq:Eneq} and~\eqref{def:r}; 
\item we prove that when  $q\asymp n^{-1/r}$  and $\lambda\gtrsim_{\log} \lambda_{c}$ hold, we can recover the communities in the SBM with an algorithm based on $P_{G}$.
\end{itemize}

\subsubsection{Construction of the cycle blow-up  with fasteners}\label{sec:blow-up:def}

Consider any rational number $a\in (0,1)$, any positive integer $\gamma$. Fix $\kappa\geq 3\vee (2a^{-1})$ an integer such that $a\kappa \gamma$ is an even integer. In the following, we construct the cycle blow-up with fasteners $G_{\kappa,\gamma,a}=(V,E)$. 

Let $C_\kappa$ be a simple cycle on vertices $[\kappa]$. For each vertex $\omega\in [\kappa]$, define a \emph{layer}
\[
L_\omega := \{ v_{\omega,t}: t=1,\dots,\gamma\},
\]
consisting of $\gamma$ vertices that represent the ``blow-up'' of the cycle node $\omega$. The collection of all such vertices forms the \emph{cycle nodes} of the graph:
$V_{\mathrm{cyc}} := \bigcup_{\omega=1}^\kappa L_\omega$, with $|V_{\mathrm{cyc}}| = \kappa \gamma$.
Between every two consecutive layers $L_\omega$ and $L_{\omega+1}$ (with the convention $L_{\kappa+1}=L_1$), we insert a complete bipartite graph:
$
E_{\mathrm{cyc}} := \bigcup_{\omega=1}^\kappa \{(v_{\omega,t},v_{\omega+1,t'}) : 1\le t,t'\le \gamma\,\},
$
so that every vertex in a layer $L_\omega$ is connected to all $\gamma$ vertices in $L_{\omega-1}$ and all $\gamma$ vertices in $L_{\omega+1}$. This produces a total of $|E_{\mathrm{cyc}}| = \kappa \gamma^2$ edges and the graph induced by $G_{\mathrm{cyc}}=(V_{\mathrm{cyc}},E_{\mathrm{cyc}})$ is $2\gamma$-regular. We refer to $G_{\mathrm{cyc}}$ as the cycle part of the graph $G$. 

We then introduce the two additional \emph{distinguished} vertices $v_1$ and $v_2$, and define the set of nodes $V := V_{\mathrm{cyc}} \cup \{v_1,v_2\}$. 
We connect $v_{1},v_{2}$ to $V_{\mathrm{cyc}}$ in such a way that $|E|=r(|V|-2)=(\gamma+a)(|V|-2)$. Accordingly, we connect $v_{1}$ or $v_{2}$ to each node in $V_{\mathrm{cyc}}$ with frequency $a$. 
We denote by 
 $V_{\mathrm{fst}}\subseteq V_{\mathrm{cyc}}$ the cyclic vertices to be connected to $v_1$ or $v_2$,  which we call the \emph{fastener nodes}. Informally, the set $V_{\mathrm{fst}}$ whose cardinality is equal to $a \kappa \gamma$ is defined such that the nodes $V_{\mathrm{fst}}$ are as evenly spread as possible. More formally, for each layer \(\omega=1,\dots,\kappa\), we define the integers $s_\omega$ by
\begin{equation}\label{eq:definition:Fastener}
s_0=0,\qquad s_\omega = \lfloor a \omega \gamma \rfloor - (s_{0}+\ldots+s_{\omega-1})\ . 
\end{equation}
Note, that $s_{\omega}\in \{\lfloor a \gamma\rfloor, \lceil a \gamma\rceil\}$. In layer $L_\omega$, exactly $s_\omega$ vertices are selected into $V_{\mathrm{fst}}$. By enumerating $V_{\mathrm{fst}}$ in the lexicographic order and alternatively assigning these nodes to two sets $V_{\mathrm{fst},1}$ or to $V_{\mathrm{fst},2}$,  we partition $V_{\mathrm{fst}}$ into two subsets of size $a\kappa\gamma/2$ --which is  an integer, as we have assumed $a\kappa \gamma$ to be an even integer.

Finally, we introduce the collection $E_{\mathrm{fst}}$ of \emph{fastener edges} by 
\[
E_{\mathrm{fst}} := \{\, (v_1,v) : v\in V_{\mathrm{fst},1} \,\}\;\cup\; \{\, (v_2,v) : v\in V_{\mathrm{fst},2} \,\}\enspace .
\]
The complete edge set of the graph is then $E := E_{\mathrm{cyc}} \cup E_{\mathrm{fst}}$. This construction is illustrated in Figure~\ref{fig:illBU}.

\vspace{-0.5cm}
\begin{figure}[h]
\begin{center}
\begin{tikzpicture}[scale=1, every node/.style={circle, draw, minimum size=4mm, inner sep=1pt, font=\small}]

\def\kappa{6}
\def\gamma{2}
\def\radius{2}
\def\shift{0.2}

\foreach \w in {1,...,8}{
    \pgfmathsetmacro{\angle}{45*(\w-1)}
    \pgfmathsetmacro{\x}{\radius*cos(\angle)}
    \pgfmathsetmacro{\y}{\radius*sin(\angle)}
    \node[fill=blue!20] (n\w1) at ({\x+\shift}, {\y+\shift}) {$(\w,1)$};
    \node[fill=blue!20] (n\w2) at ({\x-\shift}, {\y-\shift}) {$(\w,2)$};
}

\foreach \w in {1,...,8}{
    \pgfmathtruncatemacro{\next}{mod(\w,8)+1}
    \foreach \t in {1,2}{
        \foreach \tp in {1,2}{
            \draw[thick,gray!70] (n\w\t) -- (n\next\tp);
        }
    }
}

\node[fill=orange!80] (v1) at (-4,0) {$v_1$};
\node[fill=orange!80] (v2) at (4,0) {$v_2$};

\foreach \w/\t in {2/1,4/1,6/1,8/1}{
    \node[fill=red!70] at (n\w\t) {$(\w,1)$};
}

\draw[thick, red] (v1) -- (n21);
\draw[thick, red] (v1) -- (n61);
\draw[thick, red] (v2) -- (n41);
\draw[thick, red] (v2) -- (n81);

\end{tikzpicture}
\end{center}
\caption{Blow-up graph with fasteners $G_{\kappa,\gamma,a}$, in the case where $\gamma = 2, \kappa = 8, a=0.25$. The distinguished nodes $v_1,v_2$ are in orange, the {\it fastener nodes} (in $V_{\mathrm{fst}}$) are in red, the other cycle nodes are in blue. The {\it fastener edges} (in $E_{\mathrm{fst}}$) are in red, while the other edges are in gray.}\label{fig:illBU}
\end{figure}

\vspace{-0.3cm}

Interestingly, we observe that $|V| = \kappa \gamma + 2$ and $|E| = \kappa \gamma^2 + a \kappa \gamma$,
so that the graph $G_{\kappa,\gamma, a}$ satisfies
$\frac{|E|}{|V|-2}= \gamma + a$. 
The following proposition, proved in Appendix~\ref{sec:key:blow-up}, ensures that the motif $G_{\kappa,\gamma,a}$ fulfills the Condition~\eqref{eq:Eneq} we are looking for. We recall that  $\kappa\geq (3\vee 2/a)$ is such $a \kappa\gamma$ is an even integer.

\begin{proposition}\label{prop:key:blow-up}
Fix any positive integer $I\leq |V_{\mathrm{cyc}}|$ and consider any partition of $V$ into $I+1$ communities such that both $v_1$ and $v_2$ are in the same community. Define $E^{\neq}\subset E$ as the set of edges between nodes of distinct communities. Then, we have
${|E^{\neq}|\ge I\, \frac{|E|}{|V|-2}}\enspace .$
\end{proposition}

\subsection{Organisation and notation.}
In Appendix~\ref{sec:main}, we describe more precisely the statistical setting, and the low-degree lower bound is precisely stated in Appendix~\ref{sec:LD-lower-bound}. Appendix~\ref{sec:positive} gathers the results on community recovery above the threshold $\lambda_{c}$:  cliques counting is analyzed in Appendix~\ref{sec:cliques-counting},  self-avoiding paths counting in Appendix~\ref{sec:SAP-counting}, and cycle blow-up counting in Appendix~\ref{sec:blow-up-count}. 
Proofs are deferred to the last Appendices.

We denote by $\log(\cdot)$ the natural logarithm, and by $\log_{n}(\cdot)$ the logarithm in base $n$, i.e.\ $K=n^{\log_{n}(K)}$.
We write ${\bf 1}\ac{A}$ for the indicator function ${\bf 1}_{A}$ of the set $A$
and
$[D]$ for the set $\ac{1,\ldots,D}$. 
In the discussion of the results, we use the symbol $a\lesssim_{D} b$ (respectively $a\lesssim_{\log} b$) to state that $a$ is smaller than $b$ up to a possible polynomial factor in $D$ (resp. up to a poly-log factor), and we use $a\asymp b$ to state that $a\lesssim b \lesssim a$. 

\subsection*{Acknowledgments}
The work of the  three authors has  been partially supported by ANR-21-CE23-0035 (ASCAI, ANR).
The work of A. Carpentier is partially supported by the Deutsche Forschungsgemeinschaft (DFG)- Project-ID 318763901 - SFB1294 ``Data Assimilation'', Project A03,  by the DFG on the Forschungsgruppe FOR5381 ``Mathematical Statistics in the Information Age~-~Statistical Efficiency and Computational Tractability'', Project TP 02 (Project-ID 460867398), and by  the DFG on the French-German PRCI ANR-DFG ASCAI CA1488/4-1 ``Aktive und Batch-Segmentierung, Clustering und Seriation: Grundlagen der KI'' (Project-ID 490860858), by the European Research council (ERC) through the ERC Consolidator Grant SOCE (Nr: 101229569).  The work of A.~Carpentier and N.~Verzelen is also supported by the Universit\'e franco-allemande (UFA) through the college doctoral franco-allemand CDFA-02-25 ``Statistisches Lernen für komplexe stochastische Prozesse''. The work of C.~Giraud has been partially supported by ANR-19-CHIA-0021-01 (BiSCottE, ANR).

\bibliography{biblio}

\appendix
\section{Low-degree lower bound}\label{sec:main}

\subsection{Setting}\label{sec:setting}

\paragraph{Community recovery.}
Let $\mathcal{C}^*=\ac{\mathcal{C}^*_{1},\ldots,\mathcal{C}^*_{K}}$ be the partition of $\ac{1,\ldots,n}$ induced by the community assignment $z_{1},\ldots,z_{n}$, i.e.
$$\mathcal{C}^*_{k}:=\ac{i\in [n]:z_{i}=k},\quad k=1,\ldots,K.$$
Let $\widehat{\mathcal{C}}$ be a partition estimating $\mathcal{C}^*$. Exact recovery corresponds to the event $\widehat{\mathcal{C}}=\mathcal{C}^*$. Non-trivial recovery corresponds to a recovery of $\mathcal{C}^*$ better than random guessing. More precisely, let us define the clustering error as
$$err(\widehat{\mathcal{C}},\mathcal{C}^*) := \min_{\pi \in \mathrm{Perm}\pa{\ac{1,\ldots,K}}} {1\over 2n} \sum_{k=1}^K |\mathcal{C}^*_{k}\triangle \widehat{\mathcal{C}}_{\pi(k)}|,$$
where $\triangle$ is the symmetric difference, and
where the permutation $\pi$ accounts for the non-identifiability of the labels $k$ of the partitions elements $\mathcal{C}^*_{k}$. 
A zero clustering  error $err(\widehat{\mathcal{C}},\mathcal{C}^*)=0$ corresponds to perfect recovery, while non trivial recovery holds when $err(\widehat{\mathcal{C}},\mathcal{C}^*)=o(1)$. 

\paragraph{Membership matrix estimation.}
Low degree polynomials cannot directly output a partition $\hat{\mathcal{C}}$. For this reason, the usual alternative objective  
to clustering or community recovery~\cite{luo2023computational,Even24,SohnWein25,pmlr-v291-chin25a,Even25a}, is to estimate the (centered) membership matrix
$$[x_{ij}]_{1\leq i<j \leq n}=[\mathbf{1}\ac{i,j\ \text{are in the same community}}-1/K]_{1\leq i<j \leq n}$$
in $L^2$-norm. As explained in~\cite{Even24,Even25a}, this estimation problem can be directly related to the problem of clustering / community recovery.  

The distribution of $Y$ being invariant by permutation of the indices, we have
\begin{align*}
\inf_{f_{ij}: \mathrm{deg}(f_{ij})\leq D} \mathbb{E}\cro{{2\over n(n-1)}\sum_{1\leq i < j\leq n}(f_{ij}(Y)-x_{ij})^2}&=\inf_{f_{12}: \mathrm{deg}(f_{12})\leq D} \mathbb{E}\cro{(f_{12}(Y)-x_{12})^2}\\
&= \mathbb{E}[x_{12}^2]- \sup_{f_{12}: \mathrm{deg}(f_{12})\leq D}\frac{\mathbb{E}[f_{12}(Y)x_{12}]^2}{\mathbb{E}[f_{12}(Y)^2]}.
\end{align*}
Hence, to prove that degree-$D$ polynomials cannot perform significantly better than degree-0 polynomials, all we need is to prove that the supremum is $o(\mathbb{E}[x_{12}^2])=o(1/K)$. To simplify the notation, in the remaining of the paper we write $x:=x_{12}$ and 
\begin{equation}\label{eq:corr:def}
\mathrm{Corr}_{\leq D}^2 := \sup_{f: \mathrm{deg}(f)\leq D}\frac{\mathbb{E}[f(Y)x]^2}{{\mathbb{E}[f(Y)^2]}}.
\end{equation}
Proving that $\mathrm{Corr}_{\leq D}^2=o(1/K)$  for $D$ on the order of $\log(n)$ provides strong evidence of computational hardness, as it is conjectured to imply that no polynomial-time algorithm in $n$ and $K$ can estimate the membership matrix significantly better than the trivial estimator $\hat x=0$.

\subsection{Low-degree lower bound}\label{sec:LD-lower-bound}

Our first main result is a low-degree lower bound for the problem of estimating the membership matrix $x$.
We recall that $\bar q=q(1-q)$.

\begin{theorem}\label{thm:BI}
    Let $c\geq 14$, $D\geq 2$, $q\leq 1/2$, $q+2\lambda \leq 1$, and  $K\leq n$. 
    Let $\bar \lambda = 2 D^{16c}\lambda$. 
    Assume that 
\begin{align}\label{eq:signal1bis} 
 \lambda \leq  2 D^{16c}\lambda_{c},\quad \text{with $\lambda_{c}$ solution to}\quad \sup_{r\geq 1}\, {n\lambda_{c}^{2r}\over K\lambda_{c}^{r}+K^2 \bar q^{r}}=1\ .
\end{align}
Then, $\mathrm{Corr}^2_{\leq D}$ defined by \eqref{eq:corr:def} fulfills
$$\mathrm{Corr}^2_{\leq D} \leq \frac{4}{n}D^{-15 c}\enspace \ .$$
\end{theorem}

 As mentioned in the proof sketch  Section~\ref{sec:overview}, the strategy to prove Theorem~\eqref{thm:BI} is to build an almost-orthonormal basis of degree-$D$ polynomials, and then to  essentially ``solve'' the optimization problem~\eqref{eq:corr:def}.   
 We refer to Appendix~\ref{sec:proof:BI} for all the details.
 
 Since, we ``solve'' almost exactly the optimization problem~\eqref{eq:corr:def}, we can provide some interpretation of the Condition~\eqref{eq:signal1bis} in terms of algorithms. Let us consider a connected graph $G$ made of $e$ edges and $v$ nodes. As sketch in Section~\ref{sec:overview}, schematically, the term with $r\approx e/v$ in Condition~\eqref{eq:signal1bis} ensures the failure of an algorithm based on  counting --- within the observed graph $Y$ ---  the occurrence of motifs $G$ involving the nodes $1$ and $2$. For example, the term $r\approx1$ is related to counting self-avoiding paths, while the term $r\approx m/2$ is related to counting $m$-cliques. We refer to Section~\ref{sec:overview} for a more detailed discussion of this point.

\paragraph{Remarks:} 
\begin{itemize}
\item Condition~\eqref{eq:signal1bis} with $r=1$ is the KS condition $\lambda^2\leq {K\over n}\lambda+{K^2\over n} q$, or equivalently $\lambda^2 \lesssim {K^2\over n^2} + {K^2\over n} q$. The ``first condition''  namely
 $$ \lambda \lesssim_{D} {K\over n},$$
 corresponds to the Information-Theoretic barrier for diverging $K$ ---see Theorem 1.7 in~\cite{pmlr-v291-chin25a}. All the other terms (including the second one for $r=1$) correspond to an additional computational barrier.
\item When $K\leq \sqrt{n}$, according to Lemma~\ref{lem:rstar}, (i) the maximum in~\eqref{eq:signal1bis} is achieved for $r=1$, which suggests that counting self-avoiding walks is optimal, which is already known from~\cite{pmlr-v291-chin25a}, (ii) our 
Condition~\eqref{eq:signal1bis}  is
 merely the KS threshold~\eqref{eq:KS} --- up to a poly-$D$ factor. Our result is yet not as tight as the lower bound from~\cite{pmlr-v291-chin25a}, which proves failure with sharp constants, while we are loosing a poly-$D$ factor.
 \item When $K\geq \sqrt{n}$, according to Lemma~\ref{lem:rstar}, (i) the supremum in  Condition~\eqref{eq:signal1bis} can be achieved for $r>1$, leading to a phase transition below the KS threshold, (ii) for $q\gtrsim 1/n$ the threshold $\lambda_{c}$ fulfills $\lambda_{c} \asymp  q^{1-\log_{n}(K)}$ as conjectured by~\cite{pmlr-v291-chin25a} for the sparse case $q\asymp 1/n$.
\end{itemize}

\section{Community recovery above the Threshold $\lambda_{c}$}\label{sec:positive}

\subsection{Community recovery above the threshold $\lambda_{c}$ with clique-counting}\label{sec:cliques-counting}
The proof of Theorem~\ref{thm:BI} suggests that counting cliques of size $m$ might be optimal in some regimes, where the supremum in~\eqref{eq:signal1bis} is achieved for $r\approx m/2$. 
In this subsection, we confirm this insight by proving that, for $q\asymp n^{-2\over m+1}$, we can recover the communities above the threshold $\lambda \gtrsim_{\log} \lambda_{c}$ by counting $m$-cliques. 

\paragraph{Counting $m$-cliques.}
We want to determine if $i$ and $j$ are in the same community, i.e. to estimate $x_{ij}={\bf 1}\ac{z_{i}=z_{j}}-1/K$.
 Let us modify the observed graph $Y^*$ by adding an edge between the nodes $i$ and $j$, if there is no such an edge in the initial graph. Schematically, our strategy consists in counting the number of $m$-cliques involving the nodes $i$ and $j$ in the modified graph. 

More precisely, let $V=\ac{v_{1},\ldots,v_{m}}$ be a node set, and 
 $G=(V,E)$ be a clique on $V$, where we remove the edge $(v_{1},v_{2})$. 
For $i< j$, we define $\Pi_{i,j}$ as the set of injections $\pi:V\to \ac{1\ldots,n}$ such that $\pi(v_{1})=i$ and $\pi(v_{2})=j$, and we set
\begin{equation}\label{eq:def:S}
S_{ij}=\sum_{\pi \in \Pi_{i,j}} P_{G,\pi}(Y),\quad \text{with}\quad P_{G,\pi}(Y)=\prod_{(v,v')\in E}Y_{\pi(v),\pi(v')},
\end{equation}
where $Y$ is the ``centered'' adjacency matrix~\eqref{eq:definition:Y}.
Should $P_{G,\pi}$ be applied to the initial adjacency matrix $Y^*$ instead of the ``centered'' one $Y$, the sum  $S_{ij}$ would be equal to $(m-2)!$ times the number of $m$-cliques involving the nodes $i$ and $j$ in the modified graph (where we have enforced an edge between $i$ and $j$). 
The time complexity to compute $S_{ij}$ is  $O(m^2(en/m)^m)$, so it can be computed in polynomial time, as long as $m$ is considered as a constant.

Our strategy to estimate whether $i$ and $j$ are in the same community is, essentially, to compare $S_{ij}$ to a  threshold to be determined. To analyse this strategy, we need to compute some concentration bounds on $S_{ij}$ conditionally on $z_{i}=z_{j}$ and $z_{i}\neq z_{j}$. A first step in this direction is to compute the conditional means and variances of $S_{ij}$ given these two events. It will be convenient to use the notation $\bar q=q(1-q)$,
\begin{equation}\label{eq:conditional:P}
\bbP_{ij}:=\bbP\cro{\cdot|z_{i}=z_{j}},\quad\text{and}\quad \bbP_{\not ij}:=\bbP\cro{\cdot|z_{i}\neq z_{j}}.
\end{equation}

\begin{proposition}\label{prop:mean:variance}
 Assume that $q\leq 1/2$, $q+2\lambda \leq 1$, and $3\leq m\leq K$.
Let $i<j$ and let $S_{ij}$ be defined by \eqref{eq:def:S}.
We have 
\begin{align}
\bbE_{ij}\cro{S_{ij}}&={(n-2)!\over (n-m)!} \pa{\lambda^{m+1\over 2} \over K}^{m-2}, \label{eq:mean:12:prop}\\
\bbE_{\not ij}\cro{S_{ij}}&= 0. \label{eq:mean:not12:prop}
\end{align}
In addition, if for some $\rho>0$
\begin{align}
{(n-2) \lambda^{{m+1 \over 2}}\over K} & \geq \rho (m-2)^2 \pa{1+{\bar q\over \lambda}}^{m+1\over 2} \label{eq:cond1}\\
 \text{and}\quad\quad{n-2\over K^2 }\pa{\lambda^2\over \bar q}^{{m+1\over 2}} & \geq \rho (m-2), \label{eq:cond2}
 \end{align}
then, we have
\begin{align}
\mathrm{var}_{ij}(S_{ij})  & \leq \bbE_{ij}\cro{S_{ij}}^{2}\cro{\pa{2\over \rho}^{m-2}+\pa{(m-2)\pa{1+{\bar q\over \lambda}}^{m+1}+{1\over \rho}}^{m-2}\pa{1+{\bar q\over \lambda}}^{1/2}\lambda^{1/2}}\label{eq:var:12:prop}\\
\mathrm{var}_{\not ij}(S_{ij}) &\leq \bbE_{ij}\cro{S_{ij}}^{2}\cro{\pa{2\over \rho}^{m-2}+\pa{{1\over \rho}}^{m-2}}.
\end{align}
\end{proposition}

\begin{proof}[Proof of Proposition~\ref{prop:mean:variance}]
The core of the proof is to provide some unconditional upper-bounds on the variance, the proof of which is deferred to Appendix~\ref{sec:S}.
\begin{lemma}\label{lem:cliques}
Assume that $q\leq 1/2$, $q+2\lambda \leq 1$, $3\leq m\leq K$, and set $\bar p=\bar q+\lambda(1-2q)$.
Then, for any $1\leq i< j\leq n$, we have 
\begin{align}
\bbE_{ij}\cro{S_{ij}}&={(n-2)!\over (n-m)!} \pa{\lambda^{m+1\over 2} \over K}^{m-2}, \label{eq:mean:12}\\
\bbE_{\not ij}\cro{S_{ij}}&= 0, \label{eq:mean:not12}\\
\mathrm{var}_{ij}(S_{ij}) & \leq {(n-2)! (m-2)!\over (n-m)!} \bar p^{(m+1)(m-2)\over 2} \pa{{m-2 \over K}+\pa{\bar q\over \bar p}^{{m+1\over 2}}}^{m-2} \nonumber\\
&\quad +{(n-2)!(m-2)!\over (n-m)!}{\bar p^{(m+1)(m-2)} \over K^{2m-4}}
\pa{n-m+{K\over \bar p^{m+1\over 2}}}^{m-2}\bar p^{1/2} \label{eq:var:12}\\
 \mathrm{var}_{\not ij}(S_{ij}) & \leq {(n-2)! (m-2)!\over (n-m)!} \bar p^{(m+1)(m-2)\over 2} \pa{\pa{{m-2 \over K}+\pa{\bar q\over \bar p}^{{m+1\over 2}}}^{m-2} \pa{\bar q\over \bar p}^{{m-1\over 2}}+ \pa{\bar q\over \bar p}^{{(m+1)(m-2)\over 2}} }.\label{eq:var:not12}
\end{align}
\end{lemma}
Let us upper-bound $\mathrm{var}_{ij}(S_{ij})$ and $\mathrm{var}_{\not ij}(S_{ij})$ under the conditions~\eqref{eq:cond1} and~\eqref{eq:cond2}. 
Since $\bar p \leq \lambda+\bar q$, we have
\begin{align*}
 \mathrm{var}_{ij}(S_{ij}) & \leq \pa{(n-2)!\over (n-m)!}^2 \pa{m-2\over n-2}^{m-2} \cro{\pa{\bar p^{m+1\over 2}{m-2\over K}+\bar q^{m+1\over 2}}^{m-2}+
 \pa{\bar p^{m+1}{n-m\over K^2}+{\bar p^{m+1\over 2}\over K}}^{m-2}\bar p^{1/2}}\\
 &\leq  \pa{(n-2)!\over (n-m)!}^2  \pa{\lambda^{m+1\over 2} \over K}^{2(m-2)}\cro{\pa{2\over \rho}^{m-2}+\pa{(m-2)\pa{1+{\bar q\over \lambda}}^{m+1}+{1\over \rho}}^{m-2}\pa{1+{\bar q\over \lambda}}^{1/2}\lambda^{1/2}}\\
& \leq \bbE_{ij}\cro{S_{ij}}^{2}\cro{\pa{2\over \rho}^{m-2}+\pa{(m-2)\pa{1+{\bar q\over \lambda}}^{m+1}+{1\over \rho}}^{m-2}\pa{1+{\bar q\over \lambda}}^{1/2}\lambda^{1/2}}.
\end{align*}
Similarly, since $\bar p\geq \bar q$,
\begin{align*}
 \mathrm{var}_{\not ij}(S_{ij}) & \leq \pa{(n-2)! \over (n-m)!}^2 \pa{m-2\over n-2}^{m-2}
 \cro{\pa{\bar p^{m+1\over 2}{m-2\over K}+\bar q^{m+1\over 2}}^{m-2}\pa{\bar q\over \bar p}^{{m-1\over 2}}+{\bar q^{(m+1)(m-2)\over 2}} }\\
&\leq  \pa{(n-2)! \over (n-m)!}^2  \pa{\lambda^{m+1\over 2} \over K}^{2(m-2)}\cro{\pa{2\over \rho}^{m-2}\pa{\bar q\over \bar p}^{{m-1\over 2}}+{\pa{{1\over \rho}}^{m-2}}}\\
& \leq \bbE_{ij}\cro{S_{ij}}^{2}\cro{\pa{2\over \rho}^{m-2}+\pa{{1\over \rho}}^{m-2}},
\end{align*}
which concludes the proof of Proposition~\ref{prop:mean:variance}.
\end{proof}

\paragraph{Adding a Median-of-Means post-processing step.}
In Proposition~\ref{prop:mean:variance}, we control the mean and the variance of the clique count $S_{ij}$. A concentration inequality based on a second moment Markov inequality would not be tight enough to ensure meaningful results. Indeed, with a second moment inequality, we can ensure a correct result only with a probability approximately $1-\sqrt{\lambda}$, which is worse than the probability $1-K^{-1}$ of correctness of the trivial estimator $\hat x_{ij}^{0}=-K^{-1}$. So we need a better concentration inequality.
To avoid computing higher moments, we add a ``Median-of-Means post-processing step'' to get  concentration bounds good enough for our purpose.

Assume that $m\leq n/(24 \log(n))$ and $1\leq i < j \leq n$. 
Let $\Lambda=24\log(n)$ and assume for simplicity\footnote{when $(n-2)/\Lambda$ is not an integer, we partition $\ac{1,\ldots,n}\setminus\ac{i,j}$ into $\Lambda$ disjoint sets of cardinality $\lfloor (n-2)/\Lambda\rfloor$ and $\lceil (n-2)/\Lambda\rceil$. The only change in Theorem~\ref{thm:clique} is that $(n-2)/(24 \log (n))$ is replaced by $\lfloor (n-2)/(24\log (n))\rfloor$.} that $(n-2)/\Lambda$ is a positive integer larger than $m-2$. We define $N=(n-2)/\Lambda+2$, and we partition the set of nodes $\ac{1,\ldots,n}\setminus\ac{i,j}$ into $L$ disjoint parts $J^{(1)},\ldots,J^{(\Lambda)}$ with the same cardinality $N-2$. For $\ell=1,\ldots, \Lambda$, we define $\Pi_{i,j}^{(\ell)}$ has the set of injections $\pi:V\to \ac{i,j}\cup J^{(\ell)}$, such that $\pi(v_{1})=i$ and $\pi(v_{2})=j$, and we introduce the partial clique count
\begin{equation}\label{eq:partial-clique}
S_{ij}^{(\ell)}:= \sum_{\pi \in \Pi_{i,j}^{(\ell)}} P_{G,\pi}(Y)\enspace .
\end{equation}
Finally, we define $M_{ij}$ as a median of the set $\ac{S^{(1)}_{ij},\ldots,S^{(\Lambda)}_{ij}}$
and we estimate $x_{ij}={\bf 1}_{z_{i}=z_{j}}-{1\over K}$ by 
\begin{equation}\label{eq:hatx}
\hat x_{ij}=\mathbf{1}\ac{M_{ij}> {(N-2)!\over 2(N-m)!} \pa{\lambda^{m+1\over 2} \over K}^{m-2}}-{1\over K},\quad \text{with}\quad N={n-2\over 24 \log(n)}+2.
\end{equation}
We can now state our second main result.
\begin{theorem}\label{thm:clique}
Assume that $q\leq 1/2$,  $N=2+(n-2)/(24\log(n))$ is an integer, and $3\leq m\leq K\wedge N$. When
\begin{align}
{(n-2) \lambda^{{m+1 \over 2}}\over  K \log(n)} & \geq 48 (32)^{1\over m-2} (m-2)^2 \pa{1+{\bar q\over \lambda}}^{m+1\over 2} \label{eq:cond1:MoM}\\
{n-2\over K^2 \log(n) }\pa{\lambda^2\over \bar q}^{{m+1\over 2}} & \geq 48 (32)^{1\over m-2} (m-2) \label{eq:cond2:MoM}\\
 \lambda& \leq 32^{-2}\pa{1+{\bar q\over \lambda}}^{-2(m+1)(m-2)-1}(2m-4)^{-(2m-4)}\ , \label{eq:lambda}
 \end{align}
we have for $\hat x_{ij}$ defined by \eqref{eq:hatx}
$$\bbP\pa{\hat x_{ij}=x_{ij}}\geq 1-n^{-3}.$$ 
\end{theorem}

Before proving Theorem~\ref{thm:clique}, let us comment this result. We recall that $m$ is a fixed integer.
Since $\bar q \leq \lambda$, the Condition~\eqref{eq:lambda} merely requires $\lambda$ to be smaller than the constant 
\begin{equation}\label{eq:kappa}
\eta_{m}:=2^{-11}\pa{2^{(m+2)}(m-2)}^{-(2m-4)},
\end{equation}
that  depends only on $m$. As for Conditions~\eqref{eq:cond1:MoM} and \eqref{eq:cond2:MoM}, they  roughly correspond to the opposite of Condition~\eqref{eq:signal1} (which is another formulation of Condition~\eqref{eq:signal1bis}) in 
Theorem~\ref{thm:BI:Appendix} for $r=(m+1)/2$. Let us relate these two conditions to the threshold $\lambda_{c}$.
First, in light of Theorem~\ref{thm:BI} and Conditions~\eqref{eq:cond1:MoM}-\eqref{eq:cond2:MoM}, we emphasize that $m$-cliques counting can  only be optimal if the maximum in Condition~\eqref{eq:signal1bis}  is achieved for $r= (m+1)/2$. 
According to Lemma~\ref{lem:rstar}, this happens when $ q \asymp n^{-{2\over m+1}}$.
The next (immediate) corollary of Theorem~\ref{thm:clique} confirms that $m$-cliques counting is successful above the threshold $\lambda_{c}$ in this density regime. We omit the proof, which is a straightforward check that Conditions~\eqref{eq:cond1:MoM} and \eqref{eq:cond2:MoM} hold under~\eqref{eq:clique:final}.
 
\begin{corollary}\label{cor:cliques}
Assume that $N=2+(n-2)/(24\log(n))$ is an integer,  $m\in \ac{3,4,\ldots, K\wedge N}$, $\bar q=n^{-2\over m+1}$, and $\lambda\leq \eta_{m}$, with $\eta_{m}$ defined in~\eqref{eq:kappa}. When
\begin{equation}\label{eq:clique:final}
\lambda \geq w_{m} \log(n)^{2\over m+1}  q^{1-\log_{n}(K)},\quad \text{with}\  w_{m}=2\pa{96(m-2)^2(32^{1\over m-2})}^{2\over m+1},
\end{equation}
the estimator $\hat x_{ij}$ defined by \eqref{eq:hatx} fulfills
$$\bbP\pa{\hat x_{ij}=x_{ij}}\geq 1-n^{-3}.$$ 
\end{corollary}
  
 We now turn to the proof of Theorem~\ref{thm:clique}.

\begin{proof}[Proof of Theorem~\ref{thm:clique}]
Without loss of generality,  we focus on the case $(i,j)=(1,2)$ to reduce the number of indetermined indices. 
We recall that $L=24\log(n)$ and  $N=(n-2)/\Lambda+2$.
We observe that the statistic  $S^{(\ell)}_{12}$ is simply the statistic~\eqref{eq:def:S}
applied to the graph $Y$ restricted to the node set  $\ac{1,2}\cup J^{(\ell)}$, whose cardinality is $N$. Since the graph $Y$ restricted to this node set follows a SBM with the same parameters $K,q,\lambda$,  the results proven for $S_{12}$
hold for  $S^{(\ell)}_{12}$ with $n$ replaced by $N$.
From Proposition~\ref{prop:mean:variance}, when \eqref{eq:cond1:MoM}, \eqref{eq:cond2:MoM}, and\eqref{eq:lambda} hold,  we have for all $\ell=1,\ldots,\Lambda$, 
\begin{equation}\label{eq:EN}
\bbE_{\not 12}\cro{S^{(\ell)}_{12}}=0,\quad \quad \bbE_{12}\cro{S^{(\ell)}_{12}}=E_{12}[N]:={(N-2)!\over (N-m)!} \pa{\lambda^{m+1\over 2} \over K}^{m-2}\ ,
\end{equation}
and
$$\mathrm{var}_{12}(S^{(\ell)}_{12}) \vee  \mathrm{var}_{\not 12}(S^{(\ell)}_{12})\leq V[N]:={E_{12}[N]^2\over 16}.$$
As a consequence, we have from Markov inequality
\begin{equation}\label{eq:markov}
\bbP_{\not 12}\cro{S^{(\ell)}_{12}\geq 2 \sqrt{V[N]}}\leq {1\over 4},\quad \text{and}\quad 
\bbP_{12}\cro{S^{(\ell)}_{12}\leq E_{12}[N] - 2 \sqrt{V[N]}}\leq {1\over 4}\ .
\end{equation}

A key feature of the partial clique counts $S^{(1)}_{12},\ldots,S^{(L)}_{12}$ is that they are independent both under $\bbP_{12}$ and $\bbP_{\not 12}$. Indeed, conditionally on $z_{1}=z_{2}$ (resp. $z_{1}\neq z_{2}$), the random variables $W_{1},\ldots,W_{\Lambda}$ defined by $W_{\ell}=\ac{Y_{1j},Y_{2j}}_{j\in J^{(\ell)}}\cup \ac{Y_{ij}}_{i<j\in J^{(\ell)}}$ are independent, and $S^{(\ell)}_{12}$ is $\sigma(W_{\ell})$-measurable.
Hence, the number of $S^{(\ell)}_{12}$ smaller than $E_{12}[N] - 2 \sqrt{V[N]}$ under $\bbP_{12}$ (resp. exceeding $2 \sqrt{V[N]}$ under $\bbP_{\not 12}$) is stochastically dominated by a binomial distribution with parameter $(\Lambda,1/4)$, and so $M_{12}$ being a median of the $\ac{S^{(1)}_{12},\ldots,S^{(\Lambda)}_{12}}$, we get
\begin{equation}\label{eq:MoM}
\bbP_{\not 12}\cro{M_{12}\geq 2 \sqrt{V[N]}}\leq e^{-\Lambda/8}={1\over n^3},\quad \text{and}\quad 
\bbP_{12}\cro{M_{12}\leq E_{12}[N] - 2 \sqrt{V[N]}}\leq e^{-\Lambda/8}={1\over n^3}\ .
\end{equation}
Since  $V[N]= \pa{E_{12}[N]/4}^2$, we conclude that 
\begin{equation}\label{eq:separation}
\bbP_{\not 12}\cro{M_{12}<{E_{12}[N]\over 2}}\geq 1-{1\over n^3},\quad \text{and}\quad 
\bbP_{12}\cro{M_{12}> {E_{12}[N]\over 2} }\geq 1-{1\over n^3}\ .
\end{equation}
The estimator
$$\hat x_{12}=\mathbf{1}\ac{M_{12}> {E_{12}[N]\over 2}}-{1\over K},$$
with $E_{12}[N]$ defined in \eqref{eq:EN}, then fulfills $\hat x_{12}=x_{12}$ with probability at least $1-n^{-3}$.
\end{proof}
\medskip

\noindent{\bf Remark:} for the simplicity of the exposition, we have assumed that $(n-2)/(24\log(n))$ is an integer. In general, we can set $N=\left\lfloor {n-2\over 24 \log(n)}\right\rfloor+2$ and splits the set $\ac{1,\ldots,n}\setminus\ac{i,j}$ into $\Lambda=24\log(n)$ subsets with cardinality $N$ or $N+1$. Then, all the results hold with ${n-2\over 24 \log(n)}$ replaced by $N-2$.

\paragraph{Successful recovery.} 
Theorem~\ref{thm:clique} readily gives a non-trivial bound for the estimation of $x_{ij}$ by $m$-cliques couting. We recall that $MMSE_{0}=\bbE[x_{ij}^2]=K^{-1}(1-K^{-1})$. Corollary~\ref{cor:cliques} ensures that $\bbE\cro{\pa{\hat x_{ij}-x_{ij}}^2}\leq n^{-3}=o(K^{-1})$, when $\bar q\asymp n^{-2\over m+1}$ and~\eqref{eq:clique:final} holds. 

Once we have estimated $x_{ij}$, we still need a last step to output communities. 
Let us define the matrix $\widehat X\in \ac{0,1}^{n\times n}$ by $\widehat X_{ii}=0$ for $i=1,\ldots,n$, and \label{clust:scheme}
$$\widehat X_{ij}=\widehat X_{ji}=\hat x_{ij}+{1\over K},\quad \text{for}\ i<j.$$
 Seeing $\widehat X$ as the adjacency matrix of a graph $\bar X$, we estimate the communities by the connected components of $\bar X$. The overall complexity is $O(n^2m^2(en/m)^m)$.
When $\bar q\asymp n^{-2\over m+1}$ and~\eqref{eq:clique:final} holds, Corollary~\ref{cor:cliques}
ensures that we recover the communities with probability at least $1-1/n$.


\subsection{Counting self-avoiding paths}\label{sec:SAP-counting}

Our analysis suggests that counting self-avoiding paths of length $m-1$ should be optimal for the  density $q\asymp n^{-{m-2\over m-1}}$. In this subsection, we confirm this prediction. Our analysis follows the same lines as the one for $m$-cliques counting. We expose it more succinctly. 
 
Let $V=\ac{v_{1},\ldots,v_{m}}$ be a node set and   $G=(V,E)$ be the self-avoiding path $v_1, v_3,v_4,\ldots, v_{m},v_2$ on $V$. Define
\begin{equation}\label{eq:def:T}
T_{ij}=\sum_{\pi \in \Pi_{i,j}} P_{G,\pi}(Y),\quad \text{with}\quad P_{G,\pi}(Y)=\prod_{(v,v')\in E}Y_{\pi(v),\pi(v')}\enspace .
\end{equation}
Should $P_{G,\pi}$ be applied to the initial adjacency matrix $Y^*$ instead of the ``centered'' one $Y$, the sum  $T_{ij}$ would be equal to the number of self-avoiding path of length $m-1$ with end nodes  $i$ and $j$. 
The time complexity to compute $T_{ij}$ is  $O(m(en/m)^m)$, so it can be computed in polynomial time, as long as $m$ is considered as a constant.

As for the clique counting, our strategy to estimate whether $i$ and $j$ are in the same community is, essentially, to compare $T_{ij}$ to a  threshold to be determined. To analyse it, we first compute the conditional means and variances of $T_{ij}$ given the two events $z_{i}=z_{j}$ and $z_{i}\neq z_{j}$.
We remind the Definition~\eqref{eq:conditional:P} of $\bbP_{ij}$ and $\bbP_{\not ij}$, as well as the notation $\bar q=q(1-q)$.

\begin{proposition}\label{prop:mean:variance:path}
 Assume that $q\leq 1/2$, $q+2\lambda \leq 1$, and $3\leq m\leq (K\wedge n/2)$.
Let $i<j$ and let $T_{ij}$ be defined by \eqref{eq:def:T}.
We have 
\begin{align}
\bbE_{ij}\cro{T_{ij}}&={(n-2)!\over (n-m)!} \cdot \frac{\lambda^{m-1}}{K^{m-2}}\ , \quad \quad \nonumber
\bbE_{\not ij}\cro{T_{ij}}= 0\ . 
\end{align}
In addition, if for some $\rho>1$, we have $n\geq 6\rho K m^3$,  
\begin{align}
    \lambda &\geq 8\left[1+ \frac{\bar{q}}{\lambda}\right]  \left[2\rho \frac{m^3K}{n}+ e^3m^2 \rho^{1/(m-1)}\left(\frac{K}{n}\right)^{1-1/(m-1)}\right] \\ 
    \frac{\lambda^2}{\bar{q}} & \geq 4\rho \frac{Km^3}{n}+ 4 \rho^{1/(m-1)}e^3m^2 \left(\frac{K^2}{n}\right)^{1-1/(m-1)}\ ,   
\nonumber 
\end{align}
then we have $\mathrm{var}_{ij}(T_{ij})\vee \mathrm{var}_{\not ij}(T_{ij})   \leq 6 \rho^{-1}\bbE_{ij}^2\cro{T_{ij}}$.
\end{proposition}
As for the clique counting problem, relying on $T_{ij}$ alone together with a Markov type bound is not sufficient. For this reason, we rely again on a Median-of-Means post-processing step and we use the same notation as in that section.  As in the previous subsection, fix $\Lambda= 24\log(n)$ and we assume for simplicity that $(n-2)/\Lambda$ is an integer. Recall 
the  definition of $N=(n-2)/\Lambda+2$ and let $J^{(1)},\ldots,J^{(\Lambda)}$ be a partition of $\ac{1,\ldots,n}\setminus\ac{i,j}$ into $L$ disjoint parts. Then, as in~\eqref{eq:partial-clique} for clique counts, we introduce the partial self-avoiding path count  $T_{ij}^{(\ell)}:= \sum_{\pi \in \Pi_{i,j}^{(\ell)}} P_{G,\pi}(Y)$, and we define $M_{ij}$ as a median of the set $\ac{T^{(1)}_{ij},\ldots,T^{(\Lambda)}_{ij}}$. We estimate  $x_{ij}={\bf 1}_{z_{i}=z_{j}}-{1\over K}$ by 
\begin{equation}\label{eq:hatx:path}
\hat x_{ij}=\mathbf{1}\ac{M_{ij}> {(N-2)!\over 2(N-m)!} {\lambda^{m-1} \over K^{m-2}}}-{1\over K},\quad \text{where}\quad N={n-2\over 24 \log(n)}+2.
\end{equation}
We can now state our second main result.  
\begin{theorem}\label{thm:self_avoiding_path}
Assume that $q\leq 1/2$, $q+2\lambda \leq 1$,  $N=2+(n-2)/(24\log(n))$ is an integer, and $3\leq m\leq K\wedge N/2$. When $n \geq c_0 Km^3\log(n)$, 
\begin{align*}
{n \lambda^{1+ 1/(m-2)}\over  K \log(n)} & \geq c_1 m^3 \pa{1+{\bar q\over \lambda}}^{1+1/(m-2)} \ ; \\
{n \over K^2 \log(n) }\pa{\lambda^2\over \bar q}^{1+ 1/(m-2)} & \geq c_2 m^{2} \  , 
 \end{align*}
we have for $\hat x_{ij}$ defined by \eqref{eq:hatx:path} that $\bbP\pa{\hat x_{ij}=x_{ij}}\geq 1-n^{-3}$. 
\end{theorem}

\begin{proof}[Proof of Theorem~\ref{thm:self_avoiding_path}]
The proof follows exactly the same lines as that of Theorem~\ref{thm:clique} to the difference that we build upon Proposition~\ref{prop:mean:variance:path} instead of Proposition~\ref{prop:mean:variance}. We skip the details.
\end{proof}



\begin{corollary}\label{cor:SAP}
Let  $\bar q=n^{-{m-2\over m-1}}$ for some  $m\in \ac{3,4,\ldots, K\wedge N/2}$. When $n\geq c_0Km^3\log(n)$, $\lambda \leq 1-2q$, and 
\begin{equation}\label{eq:SAP:final}
\lambda \geq w'_{m} \log(n)^{m-2\over m-1} \, \bar q^{1-\log_{n}(K)},
\end{equation}
with $w'_{m}$ depending only on $m$, then the  estimator $\hat{x}$ based on the number of self-avoiding path recovers the communities with probability at least $1-1/n$. 
\end{corollary}

\begin{proof}[Proof of Corollary~\ref{cor:SAP}]
Define $r = (m-1)/(m-2)$. In this corollary, we consider the regime $\bar{q}= n^{-1/r}$. It follows from Theorem~\ref{thm:self_avoiding_path} that $\hat{x}$ recovers the communities with probability at least $1-1/n$ as long as $q+ 2\lambda \leq 1$ and 
\[
\lambda \geq \eta'_{m} \log(n)^{1/r} \left(\frac{K}{n}\right)^{1/r}= \eta'_m  \log(n)^{1/r} \bar{q}^{1-\log_n(K)} \  ,
\]
where $\eta'_m$ depends only on $m$. The result follows.
\end{proof}

\begin{proof}[Proof of Proposition~\ref{prop:mean:variance:path}]
Similarly as for the proof of Proposition~\ref{prop:mean:variance}, the proof of Proposition~\ref{prop:mean:variance:path} is a direct consequence of the following lemma proved in Appendix~\ref{sec:lem:paths}.
\begin{lemma}\label{lem:paths}
Assume that $q\leq 1/2$, $q+2\lambda \leq 1$, $3\leq m\leq K$, and set $\bar p=\bar q+\lambda(1-2q)$.
Then, for any $1\leq i< j\leq n$, we have 
\begin{align}
\bbE_{ij}\cro{T_{ij}}&={(n-2)!\over (n-m)!} \pa{\lambda^{m-1} \over K^{m-2}}, \label{eq:mean:12:path}\\
\bbE_{\not ij}\cro{T_{ij}}&= 0, \label{eq:mean:not12:path}\\ \nonumber
\mathrm{var}_{ij}(T_{ij}) & \leq \left(\frac{(n-2)!}{(n-m)!}\right)^2  \frac{\lambda^{2(m-1)}}{K^{2m-4}}m \Bigg[\frac{8Km^2 \bar p}{ (n-m) \lambda^2}+  \frac{2m^2K\bar{q}}{(n-m) \lambda^2} + 3m^2\frac{K}{n-m}+ 3\left(\frac{m^2 K}{n-m}\right)^{m-2}\\ &\hspace{3cm} +  \frac{2nm^2}{K} \left(\frac{4Km^2\bar p}{(n-m)\lambda^2}\right)^{m-1} + \frac{n m^2}{K^2}\left(\frac{2m^2K^2\bar{q}}{(n-m)\lambda^2} \right)^{m-1}
\Bigg] \ . 
 \label{eq:var:12:path}\\ \nonumber
 \mathrm{var}_{\not ij}(T_{ij}) & \leq  \left(\frac{(n-2)!}{(n-m)!}\right)^2 \frac{\lambda^{2(m-1)}}{K^{2m-4}}m \Bigg[\frac{4m^2 \bar p}{ (n-m) \lambda^2}+  \frac{2m^2K\bar{q}}{(n-m) \lambda^2} + 2m^2\frac{K}{n-m}+ 2\left(\frac{m^2 K}{n-m}\right)^{m-2}\\ &  \hspace{3cm}+ \frac{nm^2}{K} \left(\frac{4Km^2\bar p}{(n-m)\lambda^2}\right)^{m-1}+ \frac{n m^2}{K^2}\left(\frac{2m^2K^2\bar{q}}{(n-m)\lambda^2} \right)^{m-1}
\Bigg] \ . \label{eq:var:not12:path}
\end{align}
\end{lemma}

\end{proof}

\subsection{Counting Blow-up Motifs}\label{sec:blow-up-count}
Consider a cycle blow-up  with fasteners $G=G_{\kappa,\gamma , a }$, as defined in Section~\ref{sec:blow-up:def}. For $i< j$, we remind that $\Pi_{i,j}$ is the set of injections $\pi:V\to \ac{1\ldots,n}$ such that $\pi(v_{1})=i$ and $\pi(v_{2})=j$, and we set
\begin{equation}\label{eq:def:R}
R_{ij}=\sum_{\pi \in \Pi_{i,j}} P_{G,\pi}(Y),\quad \text{with}\quad P_{G,\pi}(Y)=\prod_{(v,v')\in E}Y_{\pi(v),\pi(v')},
\end{equation}
where $Y$ is the ``centered'' adjacency matrix~\eqref{eq:definition:Y}.

In the following proposition, we control the mean and the variance of $R_{ij}$ both when $z_{i}=z_{j}$ and when $z_i\neq z_j$. We recall that the conditional probabilities
$\bbP_{ij}$ and $\bbP_{\not ij}$ are defined in~\eqref{eq:conditional:P}.

\begin{proposition}\label{prop:mean:variance:blow_up}
 Assume that $q\leq 1/4$, $q+2\lambda \leq 1$, and $n\geq 2\kappa\gamma +4$. We also assume that $a\kappa \gamma$ is an even integer and that $\kappa\geq 3\vee 2/a$. 
Let $i<j$ and let $R_{ij}$ be defined by \eqref{eq:def:R}.
We have 
\begin{align}
\bbE_{ij}\cro{R_{ij}}&={(n-2)!\over (n-\kappa\gamma-2)!} \left({\lambda^{ \gamma + a } \over K}\right)^{\kappa \gamma }, \label{eq:mean:12:prop:blow_up} \quad \quad 
\bbE_{\not ij}\cro{R_{ij}} = 0 \enspace . 
\end{align}
In addition, if for some $\rho>1$, 
\begin{align}
\left(\frac{\lambda^2}{2\bar q}  \right)^{\gamma + a} \geq \frac{2K^2(\kappa \gamma)^5 \rho}{n}\ ; \quad \quad    \lambda^{\gamma + a}\geq 
 \frac{2K(\kappa \gamma)^5 \rho}{n} \label{eq:cond1:blow_up}
 \end{align}
then, we have
\begin{align}
\mathrm{var}_{ij}(R_{ij})\bigvee  \mathrm{var}_{\not ij}(R_{ij})  & \leq \frac{4}{\rho} \cdot \bbE^2_{ij}\cro{R_{ij}} \label{eq:var:12:prop:blow_up}\enspace . 
\end{align}
\end{proposition}
As for the clique counting or self-avoiding path counting, relying on $R_{ij}$ alone together with a Markov type bound is not sufficient. For this reason, we add again a Median-of-Means post-processing and we use the same notation as before. In particular, we fix $\Lambda= 24\log(n)$ and we assume for simplicity that $(n-2)/\Lambda$ is an integer. Recall 
the  definition of $N=(n-2)/\Lambda+2$ and let $J^{(1)},\ldots,J^{(\Lambda)}$ be a partition of $[n]\setminus\ac{i,j}$ into $\Lambda$ disjoint parts. For $\ell=1,\ldots, \Lambda$, we define $\Pi_{i,j}^{(\ell)}$ has the set of injections $\pi:V\to \ac{i,j}\cup J^{(\ell)}$, such that $\pi(v_{1})=i$ and $\pi(v_{2})=j$. Then, we introduce the blow-up count  $R_{ij}^{(\ell)}:= \sum_{\pi \in \Pi_{i,j}^{(\ell)}} P_{G,\pi}(Y)$, and we define $M_{ij}$ as a median of the set $\ac{R^{(1)}_{ij},\ldots,R^{(\Lambda)}_{ij}}$. We estimate  $x_{ij}={\bf 1}_{z_{i}=z_{j}}-{1\over K}$ by 
\begin{equation}\label{eq:hatx:blow-up}
\hat x_{ij}=\mathbf{1}\ac{M_{ij}> {(N-2)!\over 2(N-\kappa\gamma -2)!} \left({\lambda^{ \gamma + a } \over K}\right)^{\kappa \gamma }}-{1\over K},\quad \text{where}\quad N={n-2\over 24 \log(n)}+2.
\end{equation}
We can now state our main result.
\begin{theorem}\label{thm:blow-up}
There exists a numerical constant $c$ such that the following holds. 
Assume that $q\leq 1/4$, $q+2\lambda \leq 1$,  that $N:=2+(n-2)/(24\log(n))$ is an integer, and $N\geq 2\kappa \gamma + 4$, and that $\kappa \gamma a$ is an an even integer.  Provided that 
\begin{align*}
\left(\frac{\lambda^2}{2\bar q}  \right)^{\gamma + a} \geq c\frac{K^2(\kappa \gamma)^5 \log(n)}{n}\ ; \quad \quad    \lambda^{\gamma + a}\geq 
 c \frac{K(\kappa \gamma)^5\log(n) }{n}\enspace  , 
 \end{align*}
we have for $\hat x_{ij}$ defined by \eqref{eq:hatx:blow-up} that $\bbP\pa{\hat x_{ij}=x_{ij}}\geq 1-n^{-3}$. 
\end{theorem}

\begin{proof}[Proof of Theorem~\ref{thm:blow-up}]
The proof follows exactly the same lines as that of Theorem~\ref{thm:clique} 
to the difference that we build upon Proposition~\ref{prop:mean:variance:blow_up} instead of Proposition~\ref{prop:mean:variance}. We skip the details.
\end{proof}

\begin{corollary}\label{cor:blow-up}
Let  $\bar q=n^{-1/r}$ for some  $r= \gamma + \theta /\beta $ where  $\gamma$, $\theta $, and $\beta$ are positive integers with $\theta<\beta$. Consider the blow-up graph with fasteners $G_{\kappa,\gamma,a }$ with $\kappa = 2\beta\gamma$  and $a=\theta /\beta $. When $n\geq c_0 \beta\gamma^2  \log(n)$, $\lambda \leq 1-2q$, and 
\begin{equation}\label{eq:blow-up:final}
\lambda \geq w'_{r} \log^{1/r}(n) \,\bar q^{1-\log_{n}(K)},
\end{equation}
with $w'_{r}$ depending only on $r$, then the  estimator $\hat{x}$ based on the number of blow-up motifs recovers the communities with probability at least $1-1/n$. 
\end{corollary}

If $q=n^{-1/r}$, with $r$ a rational number, recovering the communities above the threshold~\eqref{eq:new} is feasible with a polynomial of fixed degree, only depending on $r$. 
When $r>1$ is not a rational number, we can still count blow-up motifs for some rational number $\bar{r}$ close to $r$. Indeed, consider any $\epsilon < 1$. There exists an integer $\overline{\beta}\leq 2/[\epsilon r^2]\vee 1$ and a rational number $\overline{r}= \overline{\gamma}+ \overline{\alpha}/\overline{\beta}$ such that $|\overline{r}-r|\leq \epsilon r^2/2 $. Here, we have $\bar \gamma = \lfloor r \rfloor$. Choosing $\overline{\kappa}= 2\overline{\beta}\overline{\gamma}$, we 
deduce from Theorem~\ref{thm:blow-up}, that  the  estimator $\hat{x}$ based on counting occurrence of blow-up motifs $G_{\bar \kappa,\bar \gamma, \bar \alpha/\bar \beta}$ recovers the communities with probability at least $1-1/n$ as long as 
\[
\lambda \geq w''_{r} \epsilon^{-5} \log^{1/r}(n) \,\bar q^{1-\log_{n}(K)} n^{\epsilon}\ . 
\]
The corresponding polynomial has a degree of the order $\epsilon^{-1} r$. In particular, if we take $\epsilon=\log(\log(n))/\log(n)$, we establish that, for any irrational $r$, it is possible to recover the communities above the threshold~\eqref{eq:new} with a polynomial of degree $O(\log(n)/\log(\log(n)))$. Whether this polynomial can be computed (or well-approximated) in polynomial-time remains an open question.

%

\section{Construction of the almost orthonormal basis and preliminary results}\label{sec:basis}

\subsection{Construction of the permutation-invariant basis}

In this section, we closely follow the construction of the polynomial basis in Section 3 of~\cite{CGGV25} up to a few (but important) changes.

\paragraph{Definition of invariant polynomials.}

Given a permutation $\sigma: [n]\mapsto [n]$, we define the matrix $Y_{\sigma}$ by $(Y_{\sigma})_{ij}= Y_{\sigma(i)\sigma(j)}$. A function $f$ is said to invariant by permutation of $[n]$ up to $1$ and $2$, if for any permutation $\sigma:[n]\mapsto [n]$ such that $\sigma(1)=1$ and $\sigma(2)=2$, we have $f(Y)= f(Y_{\sigma})$. 

\begin{lemma}\label{lem:reduction:permutation}
Fix any any degree $D>0$. Then, the minimum low-degree risk  $\mathrm{Corr}_{\leq D}$ is achieved by a function $f$ that is invariant by permutation up to individuals $1$ and $2$.
\end{lemma}

\begin{proof}
This result is established is the proof of Lemma 3.5 in~\cite{CGGV25}. It is a consequence of the permutation invariance of the distribution $\mathbb{P}$. 
\end{proof}

As a consequence, we need to build a suitable basis of invariant polynomials. We follow the same approach as in Section 3.2 of~\cite{CGGV25}.
 In what follows, we consider simple undirected graphs $G= (V,E)$ where $V=\{v_1,\ldots v_{r}\}$ is the set of nodes and where $E$ is the set of edges.

Let $G^{(1)}=(V^{(1)},E^{(1)})$ and $G^{(2)}= (V^{(2)},E^{(2)})$ be two graphs. We say $G^{(1)}$ and $G^{(2)}$ are equivalent if there exist a bijection $\sigma: V^{(1)} \mapsto V^{(2)}$ such that that $\sigma(v^{(1)}_1)= v_1^{(2)}$, $\sigma(v^{(1)}_2)= v_2^{(2)}$, and $\sigma$ preserves the edges. In other words, the graphs $G^{(1)}$ and $G^{(2)}$ are isomorphic with the additional constraint that the corresponding bijection maps the two first nodes.  
\begin{definition}[Collection $\mathcal{G}_{\leq D}$]
Let $\mathcal{G}_{\leq D}$ be any maximum collection of graphs $G=(V,E)$ such that (i) $|V|\geq 2$, (ii) $G$ does not contain any isolated node to the possible exceptions of $v_1,v_2$,  (iii) $1\leq |E|\leq D$, and (iv) no graphs in $\mathcal{G}_{\leq D}$ are equivalent.
\end{definition}

The collection $\mathcal{G}_{\leq D}$ corresponds to the collections of equivalence classes of all graphs with at most $D$ edges and at least $2$ nodes, and without isolated nodes (except maybe the first two nodes), if we keep the first two nodes fixed. Henceforth, we refer to $\mathcal{G}_{\leq D}$ as the collection of \emph{templates}. In fact, $\mathcal{G}_{\leq D}$ corresponds to $\mathcal{G}^{(1,2)}_{\leq D}$ in Section 3.2~\cite{CGGV25}-- here we drop the exponent $(1,2)$ in the notation of the templates and of the polynomials because we only consider a basis for estimation.

Consider a template $G= (V,E)\in \mathcal{G}_{\leq D}$. We define $\Pi_V$ the set of injective mappings from $V\rightarrow [n]$ that satisfy  $\pi(v_1) = 1, \pi(v_2) = 2$. An element $\pi\in \Pi_V$ corresponds to a labeling of the generic nodes in $V$ by elements in $[n]$. For $\pi\in \Pi_V$, we define the polynomials
\begin{equation}\label{eq:definition:P_G}
P_{G,\pi}(Y)= \prod_{(i,j)\in E} Y_{\pi(i),\pi(j)} \quad \text{and}\quad \quad \quad P_G = \sum_{\pi\in \Pi_V} P_{G,\pi}.
\end{equation}
 For short, we sometimes write $P_G$ for $P_G(Y)$ when there is no ambiguity. For the invariant polynomials $P_G$, we say that $G$ is the {\bf template} (graph) that indexes the polynomial.

Consider a template $G\in \mathcal{G}_{\leq D}$ with $c$ connected components $(G_1,G_2,\ldots, G_c)$ that contain at least one edge. To improve the orthogonality of the family $(P_{G})_{\mathcal{G}_{\leq D}}$,  
we apply a correction as in \cite{CGGV25}
\begin{equation*} 
\overline{P}_{G} := \sum_{\pi\in \Pi_V} \overline{P}_{G,\pi} \ ; \quad  \overline{P}_{G,\pi}:=\prod_{l=1}^c \left[P_{G_l,\pi} - \mathbb{E}[P_{G_l,\pi} ]\right] \ . 
\end{equation*}
Define $\mathrm{Aut}(G)$ is group of automorphisms of the graph $G$ that let $v_1$ and $v_2$ fixed. Then, given $\pi\in \Pi_V$, we will normalize $\bar P_{G}$ with the 
 variance proxy
\begin{equation}~\label{eqn:variance of graph:estimation}
\mathbb V(G) = \frac{(n-2)!}{(n-|V|)!} |\mathrm{Aut}(G)|  \mathbb E[P_{G,\pi}^2]\enspace ,
\end{equation}
where, by permutation invariance, $\mathbb E[P_{G,\pi}^2]$ does not depend on the specific choice of $\pi$. 
Importantly, the definition of $\mathbb V(G)$ in~\eqref{eqn:variance of graph:estimation} is the only difference with the original construction in Sect.3.5 of~\cite{CGGV25}. In the latter work, we used a smaller variance proxy which turns out to be a lose lower bound of $\mathbb{E}[\overline{P}^2_G]$. The rationale with this new choice \eqref{eqn:variance of graph:estimation} of  $\mathbb V(G)$ is that $\mathbb{E}[\overline{P}^2_G]$ turns out to be of the same order as $\mathbb{E}[P^2_G]= \sum_{\pi_1\pi_2}\mathbb{E}[P_{G,\pi_1}P_{G,\pi_2}]$. The largest $\mathbb{E}[P_{G,\pi_1}P_{G,\pi_2}]$ are achieved for labelings $(\pi_1,\pi_2)$ such that $P_{G,\pi_1}= P_{G,\pi_2}$. As there are 
\begin{equation}\label{eq:card-label}
|\Pi_V||\mathrm{Aut}(G)|= \frac{(n-2)!}{(n-|V|)!} |\mathrm{Aut}(G)| 
\end{equation} such labelings, we arrive at~\eqref{eqn:variance of graph:estimation}.

Finally, we introduce the polynomial $\Psi_G$ by $\Psi_{G} := \frac{\overline{P}_G}{\sqrt{\mathbb V(G)}}$. Since the $(1,(\Psi_G)_{G\in \mathcal{G}_{\leq D}})$ is a basis of permutation-invariant (up to $1$ and $2$) polynomials --see~\cite{CGGV25} for details. We readily deduce from Lemma~\ref{lem:reduction:permutation} the following result. 

\begin{lemma}\label{lem:reduction:degree_pair2}
We have 
\begin{align*}
    \mathrm{Corr}^2_{\leq D} &= \sup_{\alpha_{\emptyset}, (\alpha_G)_{G\in \mathcal G_{\leq D}}} \frac{\mathbb E\left[ x\left(\alpha_{\emptyset}+ \sum_{G \in \mathcal G_{\leq D}} \alpha_G \Psi_{G}\right)\right]^2}{{\mathbb E\left[\left[\alpha_{\emptyset}+ \sum_{G\in \mathcal G_{\leq D}} \alpha_G \Psi_{G}\right]^2\right]}}\enspace .
\end{align*}
\end{lemma}

\subsection{Some central notation and definition}\label{sec:graph:definition}

The crux of the proof is to establish that, in relevant regimes, the basis $(1,\Psi_G)$ is almost orthogonal. For that purpose, we need to introduce some notation. Those are similar to those in Section 5 in~\cite{CGGV25}. 

\paragraph{Labeled graph.} For a template $G=(V,E)$ and a labeling $\pi\in\Pi_{V}$, we define
the labeled graph $\pi(G)$ as the graph with node set $\ac{\pi(v):v\in V}$ and edge set $\ac{(\pi(v),\pi(v')):(v,v')\in E}$. 

\paragraph{Matching of nodes.} Consider two templates $G^{(1)}=(V^{(1)},E^{(1)})$ and $G^{(2)}=(V^{(2)},E^{(2)})$. Given labelings $\pi^{(1)}$ and $\pi^{(2)}$, we say  that two nodes $v^{(1)}$ and $v^{(2)}$ are matched if $\pi^{(1)}(v^{(1)})=  \pi^{(2)}(v^{(2)})$. More generally,  a matching $\mathbf M$ stands for a set of pairs of nodes $(v^{(1)},v^{(2)})\in V^{(1)}\times V^{(2)}$ where no node in $V^{(1)}$ or $V^{(2)}$ appears twice. We denote $\mathcal M$ for the collection of all possible node matchings. For $\mathbf M \in \mathcal M$, we define the collection of labelings that are compatible with $\mathbf M$ by 
\begin{multline*}
\lefteqn{\Pi(\mathbf M) = \Big\{\pi^{(1)}\in \Pi_{V^{(1)}},\pi^{(2)} \in \Pi_{V^{(2)}}: \forall (v^{(1)},v^{(2)}) \in V^{(1)}\times V^{(2)},}\\ \{\pi^{(1)}(v^{(1)}) = \pi^{(2)}(v^{(2)})\}\Longleftrightarrow \{(v^{(1)},v^{(2)}) \in \mathbf M\}\Big\}\enspace .
\end{multline*}
Importantly, as $\mathbb{P}$ is permutation invariant, $\mathbb{E}[P_{G^{(1)},\pi^{(1)}}P_{G^{(2)},\pi^{(2)}}]$ is the same for all $(\pi^{(1)},\pi^{(2)})$ in $\Pi(\mathbf{M})$. 
Given a matching $\mathbf{M}$, we write that two edges $e\in E^{(1)}$ and $e'\in E^{(2)}$ are matched if the corresponding incident nodes are matched.

\paragraph{Merged graph $G_\cup$, intersection graph $G_{\cap}$, and symmetric difference graph $G_{\Delta}$.} Consider two templates $G^{(1)}$ and $G^{(2)} \in \mathcal G_{\leq D}$ and two labelings $\pi^{(1)}$ and $\pi^{(2)}$. Then, the merged graph $G_{\cup}=(V_{\cup},E_{\cup})$ is defined as the union of $\pi^{(1)}(G^{(1)})$ and $\pi^{(2)}(G^{(2)})$, with the convention that two same edges are merged into a single edge. Similarly,  we define the intersection graph 
$G_{\cap}=(V_{\cap},E_{\cap})$ and the symmetric difference graph $G_{\Delta}=(V_{\Delta},E_{\Delta})$ so that $E_{\Delta}=E_{\cup}\setminus E_{\cap}$. Here, $V_{\cap}$ (resp. $V_{\Delta}$) is the set of nodes induced by the edges $E_{\cap}$ (resp. $E_{\Delta}$)
so that $G_{\cap}$ (resp. $G_{\Delta}$) does not contain any isolated node.  We also have $|E_{\cup}| = | E^{(1)}| +|E^{(2)}| - |E_\cap|$ and $|V_{\cup}| = |V^{(1)}|+|V^{(2)}|-|\mathbf{M}|$ for $(\pi^{(1)},\pi^{(2)})\in \Pi(\mathbf M)$. Note that, for a fixed matching $\mathbf{M}$, all graphs $G_{\cup}$ (resp. $G_{\cap}$, $G_\Delta$) are isomorphic for $(\pi^{(1)},\pi^{(2)})\in \Pi(\mathbf M)$ and, we shall refer to quantities such as $|E_{\Delta}|$, $|V_{\Delta}|$,\ldots associated to a matching $\mathbf{M}$.  
Finally, we write $\#\mathrm{CC}_{\Delta}$ for the number of connected components in $G_{\Delta}$.

\paragraph{Sets of unmatched nodes and of semi-matched nodes.} 
 Write $U^{(1)}$, resp.~$U^{(2)}$ for the set of nodes in $\pi^{(1)}(G^{(1)})$, resp.~$\pi^{(2)}(G^{(2)})$ that are not matched, namely the {\bf unmatched nodes}, that is 
\begin{align*}
U^{(1)} =\pi^{(1)}(V^{(1)})\setminus \pi^{(2)}(V^{(2)})\ ; \quad \quad\quad 
U^{(2)} =\pi^{(2)}(V^{(2)})\setminus \pi^{(1)}(V^{(1)})\ . 
\end{align*}
Again, $|U^{(1)}|$ and $|U^{(2)}|$ only  depend on $(\pi^{(1)},\pi^{(2)})$ through the matching $\mathbf{M}$. 
We have, for $i\in \{1,2\}$,
\begin{align}\label{eq:unmatched}
    |V^{(i)}| = |\mathbf M| + |U^{(i)}|.
\end{align}
Write also $\mathbf M_{\mathrm{SM}} = \mathbf M_{\mathrm{SM}}(\mathbf M) \subset \mathbf M$, 
for the set of node matches of $(G^{(1)},G^{(2)})$ that are matched, and yet that are not pruned when creating the symmetric difference graph $G_{\Delta}$.
This is the set of {\bf semi-matched nodes} - i.e.~at least one of their incident edges is not matched.  The remaining pairs of nodes $\mathbf M \setminus \mathbf M_{\mathrm{SM}}$ are said to be {\bf perfectly matched} as all the edges incident to them are matched. We write  $\mathbf M_{\mathrm{PM}} = \mathbf M_{\mathrm{PM}}(\mathbf M)$ for the set of perfect matches in $\mathbf M$. Note that 
\begin{equation}\label{eq:PMSM}
     |V^{(1)}| +|V^{(2)}| =   |V_{\Delta}| + |\mathbf M_{\mathrm{SM}}| + 2|\mathbf M_{\mathrm{PM}}|\enspace  .
\end{equation}

\paragraph{Definition of some relevant sets of nodes matchings.} Given $(G^{(1)},G^{(2)})$ and a matching $\mathbf M$, we define the pruned matching $\mathbf{M}^-$ as the matching $\mathbf{M}$ to which we remove $(v_1^{(1)},v_1^{(2)})$ (resp. $(v_2^{(1)},v_2^{(2)})$) if either $v_1^{(1)}$ (resp. $v_2^{(1)})$ is isolated in $G^{(1)}$ or $v_1^{(2)}$ (resp. $v_2^{(2)}$) is isolated in $G^{(2)}$. This definition accounts for the fact that, when isolated, the nodes $v_i^{(j)}$, for $i=1,2$ or $j=1,2$ do not play a role in the corresponding polynomials. 
Then, we define  $\mathcal M^\star\subset \mathcal{M}$ for the collection of matchings $\mathbf{M}$ such that all connected components of $G^{(1)}$ and of $G^{(2)}$ intersect with $\mathbf{M}^-$.
Finally, we introduce  $\mathcal M_{\mathrm{PM}}\subset \mathcal{M}$ for 
  the collection of perfect matchings, that is matchings $\mathcal M$ such that all the nodes in $V^{(1)}$ and $V^{(2)}$ are {\bf perfectly matched}. Note that, if $\mathbf M\in \mathcal M_{\mathrm{PM}}$, then $G_{\Delta}$ is the empty graph (with $E_{\Delta} = \emptyset$). Besides, $\mathcal{M}_{\mathrm{PM}} \neq \emptyset$ if and only $G^{(1)}$ and $G^{(2)}$ are isomorphic, which is equivalent to $G^{(1)}= G^{(2)}$ when $G^{(1)}, G^{(2)} \in \mathcal G_{\leq D}$.

  \paragraph{Shadow matchings.} Given two sets $W^{(1)}\subset V^{(1)}, W^{(2)} \subset V^{(2)}$ and a set of node matches $\underline{\mathbf M} \subset \mathcal M$, we define $\mathcal M_{\mathrm{shadow}}(W^{(1)}, W^{(2)}, \underline{\mathbf M} )$ as the collection of matchings $\mathbf{M}$ satisfying 
 \begin{equation}\label{eq:definition:mshadow}
  \mathbf M_{\mathrm{SM}}(\mathbf M) = \underline{\mathbf M}, \quad  (\pi^{(1)})^{-1}(U^{(1)})= W^{(1)}\quad \text{ and }\quad (\pi^{(2)})^{-1}(U^{(2)}) = W^{(2)} \enspace  
 \end{equation}
  for any  $(\pi^{(1)},\pi^{(2)})\in \Pi(\mathbf{M})$. Note that, as long as~\eqref{eq:definition:mshadow} is satisfied for one labeling $(\pi^{(1)},\pi^{(2)})\in \Pi(\mathbf{M})$, it it satisfied for all such $(\pi^{(1)},\pi^{(2)})$. In fact, $\mathcal M_{\mathrm{shadow}}(W^{(1)}, W^{(2)}, \underline{\mathbf M} )$ is the collection of all matchings that lead to the set $\underline{\mathbf M} $ of semi-matched nodes and such that  $W^{(1)}, W^{(2)}$ correspond  to unmatched nodes in resp.~$G^{(1)}, G^{(2)}$. We say that these matchings satisfy a given {\bf shadow} $(W^{(1)}, W^{(2)}, \underline{\mathbf M} )$. The only thing that can vary between two elements of $\mathcal M_{\mathrm{shadow}}(W^{(1)}, W^{(2)}, \underline{\mathbf M} )$ is the matching of the nodes that are not in $W^{(1)}, W^{(2)}$, or part of a pair of nodes in $\underline{\mathbf M}$. This matching must however ensure that all of these nodes are perfectly matched.

\paragraph{Edit Distance between graphs.} For any two templates $G^{(1)}$ and $G^{(2)}$, we define  the so-called edit-distance
\begin{equation}\label{eq:definition:edit:distance}
d(G^{(1)}, G^{(2)}) := \min_{\mathbf M \in \mathcal M} |E_{\Delta}| \enspace .
\end{equation}
Note that $d(G^{(1)}, G^{(2)})=0$ if and only if $G^{(1)}$ and $G^{(2)}$ are isomorphic. As a consequence, if $G^{(1)}$ and $G^{(2)}$ are in $\mathcal{G}_{\leq D}$, the edit distance is equal to $0$ if and only if $G^{(1)}=G^{(2)}$.

\section{Proof of Theorem~\ref{thm:BI}}\label{sec:proof:BI}
We actually prove a slightly stronger version (in terms of log factors) of Theorem~\ref{thm:BI}.
\begin{theorem}\label{thm:BI:Appendix}
    Let $c_{\texttt{s}}\geq 14$, $D\geq 2$, $q\leq 1/2$, $q+2\lambda \leq 1$, and  $K\leq n$. 
    Assume that 
\begin{align}\label{eq:signal1} 
\sup_{1\leq r\leq D}\ac{ D^{8c_{\texttt{s}}r} \times \cro{\left(\frac{\sqrt{n}}{K} \left(\frac{\lambda}{\sqrt{\bar q}}\right)^{r}\right) \land \left(\sqrt{n\over K}\lambda^{r/2}} \right)}  \leq   1.
\end{align}
Then, $\mathrm{Corr}^2_{\leq D}$ defined by \eqref{eq:corr:def} fulfills
$$\mathrm{Corr}^2_{\leq D} \leq \frac{4}{n}D^{-15 c_{\texttt{s}}}\enspace \ .$$
\end{theorem}

Theorem~\ref{thm:BI} readily follows from Theorem~\ref{thm:BI:Appendix} since Condition~\eqref{eq:signal1bis} implies Condition~\eqref{eq:signal1}.

Let us prove Theorem~\ref{thm:BI:Appendix}.
We first prove that the family $(1,(\Psi_G)_{G\in \mathcal G_{\leq D}})$ is almost-orthonormal - see Proposition~\ref{prop:ortho}. In turn, this allows us bound $\mathrm{Corr}_{\leq D}$ by simply controlling $\left\|(\mathbb E[\Psi_G x])_{G\in \mathcal G_{\leq D}}\right\|_2$ - see Equation~\eqref{eq:advat}. 
\medskip


\noindent
{\bf Step 1: Proving that $(\Psi_G)_{G\in \mathcal D}$ is almost-orthonormal.}
By definion of $\Psi_G$ and $\overline{P}_G$, we have $\mathbb{E}[1\cdot \Psi_G]= 0$ for any $G\in \mathcal{G}_{\leq D}$. As a consequence, we only have to prove that $(\Psi_G)_{G\in \mathcal D}$ is almost-orthonormal.

It is convenient to use the short notation $X_{ij}={\bf 1}\ac{z_{i}=z_{j}}$.
We observe that $\bbE\cro{Y_{ij}|X}=\lambda X_{ij}$ and 
$\bbE\cro{Y_{ij}^2|X}= \bar q+ \lambda X_{ij}(1-2q)$. 
Consider two templates $G^{(1)}, G^{(2)}$, some node matching $\mathbf M \in \mathcal M$ and two injections $(\pi^{(1)}, \pi^{(2)})\in \Pi(\mathbf{M})$. Since the $Y_{ij}$ are conditionally independent given $X$, we have 
\begin{align}
   \mathbb{E}\left[P_{G^{(1)}, \pi^{(1)}}P_{G^{(2)}, \pi^{(2)}}\right] 
    & = \bbE\cro{\prod_{(i,j)\in E_{\Delta}}Y_{ij}\prod_{(i,j)\in E_{\cap}}Y_{ij}^2}\nonumber\\
    &=   \bbE\cro{\prod_{(i,j)\in E_{\Delta}} (\lambda X_{ij}) \prod_{(i,j)\in E_{\cap}}[\bar q+\lambda X_{ij}(1-2q)]}.\label{eqn:old basis proof ingredients}
\end{align}
In particular
\begin{align}\label{eq:comp}
    \mathbb E\left[P_{G^{(1)}, \pi^{(1)}}\right] =   \lambda^{|E^{(1)}|} \frac{1}{K^{|V^{(1)}| - \#\mathrm{CC}_{G^{(1)}}}}\ , 
\end{align}
where $\#\mathrm{CC}_{G^{(1)}}$ is the number of connected components of $G^{(1)}$. 

Given $T\subset [n]$, write $Y_{T}= (Y_{ij})_{i\in T,j\in T}$.  By independence of the $z_i$ as well as the conditional independence of the $Y_{ij}$ given $X$, it follows 
that for any set $T_1,T_2\subset [n]$ such that $T_1\cap T_2=\emptyset$, 
for any functions $f^{(1)},f^{(2)}$ 
\begin{equation}\label{eq:indep}
\mathbb E\left[f^{(1)}(Y_{T_1})f^{(2)}(Y_{T_2}) \right] = \mathbb E\left[f^{(1)}(Y_{T_1})\right]\mathbb E\left[f^{(2)}(Y_{T_2}) \right].
\end{equation}
As a consequence, when the vertex sets $\pi^{(1)}(V^{(1)})$ and $\pi^{(2)}(V^{(2)})$ are disjoint, the expectation factorizes in  (\ref{eqn:old basis proof ingredients}).
When the vertex sets overlap, we have the following upper-bound.

 \begin{proposition}\label{prop:christophe}
For any templates $(G^{(1)},G^{(2)})$, any node matching $\mathbf{M}\in \mathcal{M}^\star$, and any labeling $(\pi^{(1)},\pi^{(2)})\in \Pi(\mathbf{M})$, we have
\begin{align}\label{eq:christophe1}
\left|\frac{\mathbb{E}\left[P_{G^{(1)}, \pi^{(1)}}P_{G^{(2)}, \pi^{(2)}}\right]}{\sqrt{\mathbb E [P_{G^{(1)}, \pi^{(1)}}^2]\mathbb E [P_{G^{(2)}, \pi^{(2)}}^2]}}\right| & \leq \left(\frac{1}{K^{|U^{(1)}|+ |U^{(2)}|}} \left(\frac{\lambda}{\sqrt{\bar q}}\right)^{|E_{\Delta}|}\right) \land \left(\frac{\lambda^{|E_{\Delta}|/2}}{K^{(|U^{(1)}|+ |U^{(2)}|)/2}} \right)\enspace .
\end{align}
Besides, for any connected template  $G$, and any labeling $\pi$, we have 
\begin{align}
\label{eq:christophe2}
\left|\frac{\mathbb{E}\left[P_{G, \pi}\right]}{\sqrt{\mathbb E [P_{G, \pi}^2]}}\right| & \leq \left(\frac{1}{K^{|V|-1}} \left(\frac{\lambda}{\sqrt{\bar q}}\right)^{|E|}\right) \land \left(\frac{\lambda^{|E|/2}}{K^{(|V|-1)/2}} \right)\enspace . 
\end{align}

 \end{proposition}


A first key observation is that the Gram matrix $\left(\mathbb E\left[\overline{P}_{G^{(1)},\pi^{(1)}}\overline{P}_{G^{(2)},\pi^{(2)}}\right]\right)_{G^{(1)}, G^{(2)}\in \mathcal G_{\leq D}, \pi^{(1)} \in \Pi_{V^{(1)}}, \pi^{(1)} \in \Pi_{V^{(2)}}}$ associated to $(\overline{P}_{G,\pi})_{G\in \mathcal G_{\leq D}, \pi \in \Pi_V}$ is quite sparse - unlike the one associated to $(P_{G,\pi})_{G\in \mathcal G_{\leq D}, \pi \in \Pi_V}$ - and that  the non-zero entries are quite close of those associated to $(P_{G,\pi})_{G\in \mathcal G_{\leq D}, \pi \in \Pi_V}$. 
\begin{proposition} \label{prop:boundbar}
     We have 
    \begin{enumerate}
    \item if $\mathbf{M}\notin\mathcal{M}^{\star}$ we have $\mathbb{E}\left[\overline{P}_{G^{(1)}, \pi^{(1)}}\overline{P}_{G^{(2)}, \pi^{(2)}}\right] = 0$; 
    \item if $\mathbf{M}\in \mathcal{M}^\star$ we have
    \begin{equation*}
        \left|\dfrac{\mathbb{E}\left[\overline{P}_{G^{(1)}, \pi^{(1)}}\overline{P}_{G^{(2)}, \pi^{(2)}}\right] - \mathbb{E}\left[P_{G^{(1)}, \pi^{(1)}}P_{G^{(2)}, \pi^{(2)}}\right]}{\mathbb{E}\left[P_{G^{(1)}, \pi^{(1)}}P_{G^{(2)}, \pi^{(2)}}\right]}\right| \leq 
        2^{|\mathbf M_{\mathrm{SM}}|}D^{-4\texttt{c}_s} + \mathbf{1}_{|\mathbf M_{\mathrm{SM}}|>0} 2^{|\mathbf M_{\mathrm{SM}}|}\ . 
    \end{equation*}
    \end{enumerate} 
\end{proposition}
We then consider the family of invariant polynomials $(\Psi_G)_{G\in \mathcal G_{\leq D}}$. We prove that the associated Gram matrix  $\left(\Gamma_{G^{(1)}, G^{(2)}}\right)_{G^{(1)}, G^{(2)}\in \mathcal G_{\leq D}} = \left(\mathbb E\left[\Psi_{G^{(1)}}\Psi_{G^{(2)}}\right]\right)_{G^{(1)}, G^{(2)}\in \mathcal G_{\leq D}}$,
is close to the identity.
\begin{proposition}\label{prop:scalprod}
Let $D\geq 2$.
    Consider two templates $G^{(1)},G^{(2)} \in \mathcal G_{\leq D}$. We have
    \begin{equation*}
     |\Gamma_{G^{(1)}, G^{(2)}} -1|=    \left|\mathbb E\left[\Psi_{G^{(1)}}\Psi_{G^{(2)}}\right] - \mathbf 1\{G^{(1)} = G^{(2)}\}\right|\leq 3D^{-c_{\texttt{s}} d (G^{(1)}, G^{(2)})\lor 1}\enspace .
    \end{equation*}
\end{proposition} 
In turn, we deduce from the above bound that  $\Gamma$ is close to the identity matrix in operator norm $\|.\|_{op}$.
\begin{proposition}\label{prop:ortho}
Let $D\geq 2$.
We have
    \[
\|\Gamma - \mathrm{Id}\|_{op}\leq 6 D^{-c_{\texttt{s}}/2}\enspace . 
\]
\end{proposition}

 From there, we can easily conclude on the bound on $\mathrm{Corr}_{\leq D}$. Note first that the previous proposition together with Lemma~\ref{lem:reduction:degree_pair2} and  $\mathbb{E}[x]=0$ implies 
\begin{align}\label{eq:advat}
    \mathrm{Corr}^2_{\leq D} &= \sup_{\alpha_{\emptyset}, (\alpha_G)_{G\in \mathcal G_{\leq D}}} \frac{\mathbb E\left[ x\left(\alpha_{\emptyset} +\sum_{G \in \mathcal G_{\leq D}} \alpha_G \Psi_{G}\right)\right]^2}{\mathbb E\left[\left[\alpha_{\emptyset}+ \sum_{G\in \mathcal G_{\leq D}} \alpha_G \Psi_{G}\right]^2\right]} \leq \frac{\sum_{G\in \mathcal{G}_{\leq D}}\mathbb E[x\Psi_G]^2}{1- 6 D^{-c_{\texttt{s}}/2}}\enspace .
\end{align} 
So, we just need to bound $\mathbb E[x\Psi_G]$ for each $G\in \mathcal{G}_{\leq D}$.

\medskip

\noindent
{\bf Step 2: Bounding $\mathbb E [x\Psi_G]$.} For any $G\in \mathcal G_{\leq D}$, write $c$ its number of connected components with a least one edge and write $(G_1,\ldots, G_{c})$ its decomposition in such connected components. It follows from~\eqref{eq:indep} that, for any $\pi \in \Pi_V$, we have 
\begin{equation}
\mathbb E[x\Psi_G] = \frac{(n-2)!}{(n-|V|)!}\frac{1}{\sqrt{\mathbb V(G)}}\mathbb E\left[\left(\1_{z_1=z_2}-\frac{1}{K}\right) \prod_{i=1}^{c}  \overline{P}_{G_i,\pi}\right]=  \frac{(n-2)!}{(n-|V|)!}\frac{1}{\sqrt{\mathbb V(G)}}\mathbb E\left[\1_{z_1=z_2} \prod_{i=1}^{c}  \overline{P}_{G_i,\pi}\right]\enspace . 
\end{equation}

\noindent
\underline{Case 1: $G$ contains at least one connected component that does not contain $v_1$ or $v_2$.} By Equation~\eqref{eq:indep}, we have 
$$\mathbb E[x\Psi_G] = 0.$$
So that we restrict to $G\in \mathcal G_{\leq D}$ such that each connected component contains $v_1$ or $v_2$.

\noindent
\underline{Case 2: $v_1$ and $v_2$ are not in the same connected component.} Then we know that the graph contains exactly two connected components $G_1=(V_1,E_1)$ and $G_2=(V_2,E_2)$, one of them being possibly an isolated node.  Fix $\pi\in \Pi_V$. If $G_1$ (resp. $G_2$ is an isolated node), we use the convention $P_{G_1,\pi}=1$ (resp. $P_{G_2,\pi}=1$). Developing the $\Psi_G$ and  using Equation~\eqref{eq:comp} we obtain:
\begin{align*}
\frac{(n-|V|)!}{(n-2)!}&\sqrt{\mathbb V(G)}\mathbb E[x\Psi_G] \\
&=  \mathbb E[\1_{z_1=z_2}P_{G_1,\pi}P_{G_2,\pi}] +  \mathbb E[\1_{z_1=z_2}] \mathbb E[P_{G_1,\pi}]\mathbb E[P_{G_2,\pi}]\\ &\quad  -    \mathbb E[P_{G_1,\pi}]\mathbb E[\1_{z_1=z_2}P_{G_2,\pi}]-\mathbb E[\1_{z_1=z_2}P_{G_1,\pi}]\mathbb E[P_{G_2,\pi}] \\& = \lambda^{|E|} \left[\frac{1}{K^{|V_1|+|V_2|-1}}+\frac{1}{K}\cdot \frac{1}{K^{|V_1|-1}}\cdot \frac{1}{K^{|V_2|-1}} - \frac{1}{K^{|V_1|}}\cdot\frac{1}{K^{|V_2|-1}}- \frac{1}{K^{|V_2|}}\cdot\frac{1}{K^{|V_1|-1}} \right]\\
& = 0\enspace .
\end{align*}

\noindent
\underline{Case 3: $v_1,v_2$ are in the same connected component.} Then, this implies that  $\mathbb E[x\Psi_G]=0$ unless $G$ is connected. It then follows that, for any $\pi \in \Pi_V$,
$$\frac{(n-|V|)!}{(n-2)!}\sqrt{\mathbb V(G)}\mathbb E[x\Psi_G] = \mathbb E[\1_{z_1=z_2}P_{G,\pi}] - \mathbb E[\1_{z_1=z_2}] \mathbb E[P_{G,\pi}] =   [1 - K^{-1}]  \mathbb E[P_{G,\pi}],$$
where the last equality stems from 
$$\mathbb E[\1_{z_1=z_2}P_{G,\pi}]=\lambda^{[E]}\bbP\cro{z_{i}\ \text{all equal on}\ V}=\mathbb E[P_{G,\pi}].$$
 Combining the definition of $\mathbb V(G)$ and~\eqref{eq:christophe2} in Proposition~\ref{prop:christophe}, we get 
\begin{align*}
\left|\mathbb E[x\Psi_G]|\right| & \leq   \sqrt{\frac{(n-2)!}{\mathbb E[P_{G,\pi}^2](n-|V|)!|\mathrm{Aut}(G)|}}|\mathbb E[P_{G,\pi}]| \\
& \leq \sqrt{\frac{(n-2)!}{(n-|V|)!|\mathrm{Aut}(G)|}} \left[\left(\frac{1}{K^{|V|-1}} \left(\frac{\lambda}{\sqrt{\bar q}}\right)^{|E|}\right) \land \left(\frac{\lambda^{|E|/2}}{K^{(|V|-1)/2}} \right)\right]\enspace .
\end{align*}
As a consequence, we get 
\begin{align*}
    \left|\mathbb E[x\Psi_G]|\right|     &\leq n^{(|V|-2)/2} \left[\left(\frac{1}{K^{|V|-1}} \left(\frac{\lambda}{\sqrt{\bar q}}\right)^{|E|}\right) \land \left(\frac{\lambda^{|E|/2}}{K^{(|V|-1)/2}} \right)\right]\\
    &\leq   \frac{1}{\sqrt{n}}\left[\left(\left(\frac{\sqrt{n}}{K}\right)^{|V|-1} \left(\frac{\lambda}{\sqrt{\bar q}}\right)^{|E|}\right) \land \left(\left(\sqrt{\frac{n}{K}}\right)^{|V|-1}\lambda^{|E|/2} \right)\right]\\
    &\leq \frac{1}{\sqrt{n}} \left[\left(\frac{\sqrt{n}}{K} \left(\frac{\lambda}{\sqrt{\bar q}}\right)^{r}\right) \land \left(\sqrt{\frac{n}{K}}\lambda^{r/2} \right)\right]^{|V|-1}\enspace ,
\end{align*}
where $r = r(G) = |E|/(|V|-1)$. Recall, that $G$ is a connected graph. Hence, $|E|\geq |V|-1$ and we have $1\leq r(G)\leq D$, since $|E|\leq D$. Under our signal condition in Equation~\eqref{eq:signal1}, we conclude that 
\begin{align*}
    \left|\mathbb E[x\Psi_G]|\right| &\leq \frac{1}{\sqrt{n}} D^{-8c_{\texttt{s}}|E|}.  
\end{align*}
\medskip

\noindent
\underline{Conclusion.} Combining the above, we arrive at 
\begin{align*}
 \sum_{G\in \mathcal{G}_{\leq D}} \mathbb{E}^2[x\Psi_G]  &\leq  \frac{1}{n}\sum_{G\in \mathcal G_{\leq D}: |E| \geq 1}  D^{-16c_{\texttt{s}}|E|}= \frac{1}{n}\sum_{v\in [2D],e \in [D]}\quad \quad  \sum_{G\in \mathcal G_{\leq D}: |V|=v, |E|=e} D^{-16c_{\texttt{s}}|E|}\\
    &\leq   \frac{1}{n}\sum_{v\in [2D],e \in [D]} v^{2e}  D^{-16c_{\texttt{s}}e}\leq  \frac{2}{n} D^{-15 c_{\texttt{s}}}\enspace 
\end{align*}
since $c_{\texttt{s}}\geq 8$ and $D \geq 2$.
Together with~\eqref{eq:advat}, we conclude that 
\begin{align*}
 \mathrm{Corr}^2_{\leq D}\leq  \frac{4}{n} D^{-15 c_{\texttt{s}}}\enspace . 
\end{align*}

\begin{proof}[Proof of Proposition~\ref{prop:boundbar}]

\noindent
{\bf Proof of 1):} Consider any matching $\mathbf M \not\in \mathcal M^\star$. As a consequence, there exists a connected component $G'$ of $G^{(1)}$ or in $G^{(2)}$ such that no nodes in $G'$ belongs to $\mathbf{M}^{-}$ ---recall the definition of $\mathbf{M}^{-}$ in Section~\ref{sec:graph:definition}. Without loss of generality, we assume that $G'$ is a connected component of $G^{(1)}$. By definition, we have  $\overline{P}_{G^{(1)}, \pi^{(1)}} = \overline{P}_{G', \pi^{(1)}} \overline{P}_{\underline{G}^{(1)}, \pi^{(1)}}$ where $\underline{G}^{(1)}$ is the complement of $G'$ in $G^{(1)}$. So that by Equation~\eqref{eq:indep}, we get
$$\mathbb E[\overline{P}_{G^{(1)}, \pi^{(1)}}\overline{P}_{G^{(2)}, \pi^{(2)}}] = \mathbb E[\overline{P}_{\underline{G}^{(1)}, \pi^{(1)}}\overline{P}_{G^{(2)}, \pi^{(2)}}] \mathbb E[\overline{P}_{G', \pi^{(1)}}] = 0\enspace ,$$
as $\mathbb E[\overline{P}_{G', \pi^{(1)}}]=0$.

\noindent
{\bf Proof of 2):} Write $c_1$ (resp. $c_2$) for the number of connected component of $G^{(1)}$ (resp. $G^{(2)}$) that contain at least one edge. Besides, we write $(G_1^{(1)}, \ldots, G^{(1)}_{c_1})$ and  $(G_1^{(2)}, \ldots, G^{(2)}_{c_2})$ for the decomposition of $G^{(1)}$, $G^{(2)}$ into these connected components. Then, we have 
\begin{align*}
    \mathbb{E}\left[\overline{P}_{G^{(1)}, \pi^{(1)}}\overline{P}_{G^{(2)}, \pi^{(2)}}\right] = & \sum_{S_1 \subset [c_1],\  S_2 \subset [c_2]} (-1)^{|S_1|+|S_2|}\mathbb E\left[\prod_{i\in [c_1]\setminus S_1}P_{G^{(1)}_i,\pi^{(1)}}\prod_{j\in [c_2]\setminus S_2}P_{G^{(2)}_j,\pi^{(2)}}\right] \\&  \hspace{2cm}\times \mathbb E\left[\prod_{i\in S_1}P_{G^{(1)}_i,\pi^{(1)}}\right]\mathbb E\left[\prod_{i\in S_2}P_{G^{(2)}_i,\pi^{(2)}}\right]\enspace . 
\end{align*}

The following lemma holds.
\begin{lemma}\label{lem:unioncomp}
For any $(\pi^{(1)},\pi^{(2)})\in \Pi(\mathbf{M})$ and any $S_1\subset [c_1]$ and any $S_2\subset [c_2]$, we have 
\begin{multline*}
\lefteqn{0 \leq \mathbb E\left[\prod_{i\in [c_1]\setminus S_1}P_{G^{(1)}_i,\pi^{(1)}}\prod_{j\in [c_2]\setminus S_2}P_{G^{(2)}_j,\pi^{(2)}}\right] \mathbb E\left[\prod_{i\in S_1}P_{G^{(1)}_i,\pi^{(1)}}\right]\mathbb E\left[\prod_{i\in S_2}P_{G^{(2)}_i,\pi^{(2)}}\right]}\\ \leq \lambda^{|\mathbf M_{\mathrm{PM}}^{(S_1,S_2)}|/2} \mathbb{E}\left[P_{G^{(1)}, \pi^{(1)}}P_{G^{(2)}, \pi^{(2)}}\right]\enspace .
\end{multline*}
    where $\mathbf M_{\mathrm{PM}}^{(S_1,S_2)}\subset \mathbf M_{\mathrm{PM}}$ is the set of pairs of $(v,v')$ perfectly matched nodes (to the possible exception of  $v_1$ and $v_2$ in the case where they are isolated) such that either $v$ belongs to  any  $(G^{(1)}_i)$ with $i\in S_1$ or $v'$ belongs to  any  $(G^{(1)}_i)$ with $i\in S_2$. 
\end{lemma}

Note that choosing any set of connected components $S_1$, $S_2$ amounts to choosing some specific subsets of node matches in $\mathbf M = \mathbf M_{\mathrm{PM}}\cup \mathbf M_{\mathrm{SM}}$. Since $\mathbf{M}\in \mathcal{M}^\star$, this subset of $\mathbf M$ is therefore non empty. This leads us to:
\begin{align*}
    \frac{\left|\mathbb{E}\left[\overline{P}_{G^{(1)}, \pi^{(1)}}\overline{P}_{G^{(2)}, \pi^{(2)}}\right] - \mathbb{E}\left[P_{G^{(1)}, \pi^{(1)}}P_{G^{(2)}, \pi^{(2)}}\right]\right|}{\mathbb{E}\left[P_{G^{(1)}, \pi^{(1)}}P_{G^{(2)}, \pi^{(2)}}\right]}&\leq  \sum_{\substack{S_1 \subseteq [c_1], \, S_2 \subseteq [c_2]\\ S_1\cup S_2\neq \emptyset}} \lambda^{|\mathbf M_{\mathrm{PM}}^{(S_1,S_2)}|/2}\\
    &\hspace{-2cm}\leq \sum_{\substack{\widetilde{\mathbf M}_{\mathrm{PM}} \subseteq \mathbf M_{PM} \ , \widetilde{\mathbf M}_{\mathrm{SM}} \subseteq \mathbf M_{SM} 
    \\  
    \widetilde{\mathbf M}_{\mathrm{PM}}\cup \widetilde{\mathbf M}_{\mathrm{SM}}\neq \emptyset}
    } \lambda^{|\widetilde{\mathbf M}_{\mathrm{PM}}|/2}\\
    & \hspace{-2cm} \leq \mathbf{1}_{|\mathbf M_{\mathrm{SM}}|>0}2^{|\mathbf M_{\mathrm{SM}}|}\sum_{\widetilde{\mathbf M}_{\mathrm{PM}} \subseteq \mathbf M_{PM}} \lambda^{|\widetilde{\mathbf M}_{\mathrm{PM}}|/2}+ \mathbf{1}_{|\mathbf M_{\mathrm{SM}}|=0} \sum_{\emptyset \subsetneq\widetilde{\mathbf M}_{\mathrm{PM}} \subseteq \mathbf M_{PM}} \lambda^{|\widetilde{\mathbf M}_{\mathrm{PM}}|/2}\\
    &\hspace{-2cm} \leq \mathbf{1}_{|\mathbf M_{\mathrm{SM}}|>0} 2^{|\mathbf M_{\mathrm{SM}}|}\sum_{r \geq 0} |\mathbf M_{\mathrm{PM}}|^r \lambda^{r/2} +  \mathbf{1}_{|\mathbf M_{\mathrm{SM}}|=0}\sum_{r \geq 1} |\mathbf M_{\mathrm{PM}}|^r \lambda^{r/2} \\ &  \hspace{-2cm} 
    \leq 2^{|\mathbf M_{\mathrm{SM}}|}D^{-4\texttt{c}_s} + \mathbf{1}_{|\mathbf M_{\mathrm{SM}}|>0} 2^{|\mathbf M_{\mathrm{SM}}|} \enspace ,
\end{align*}
where we used that $|\mathbf M_{\mathrm{PM}}| \leq |V^{(1)}| \leq  2D$, $D\geq 2$, and $\sqrt{\lambda}  \leq D^{-8\texttt{c}_s}$ with $c_{\texttt{s}}>1$ by Condition~\ref{eq:signal1}.



\end{proof}

\begin{proof}[Proof of Lemma~\ref{lem:unioncomp}]
Define the matching $\bar{\mathbf M}\subset \mathbf M$ where we have removed all the nodes $(v,v')$ such that either $v$ belongs to any connected component $G^{(1)}_i$ with $i\in S_1$ or $v'$ belongs to any connected component $G^{(2)}_j$ with $j\in S_2$. It is possible to choose labelings $(\bar \pi^{(1)}, \bar \pi^{(2)}) \in \Pi(\bar{ \mathbf M})$ such that $\bar \pi^{(1)}= \pi^{(1)}$ and $\bar{\pi}_2(i)= \pi_2(i)$ for $i\in U^{(2)}$, that is unmatched nodes of $G^{(2)}$. 
We have 
\begin{align}\label{eq:coolio}
 \mathbb E\left[P_{G^{(1)}, \bar \pi^{(1)}}P_{G^{(2)}, \bar \pi^{(2)}}\right]=   \mathbb E\left[\prod_{i\in [c_1]\setminus S_1}P_{G^{(1)}_i,\pi^{(1)}}\prod_{j\in [c_2]\setminus S_2}P_{G^{(2)}_j,\pi^{(2)}}\right] \mathbb E\left[\prod_{i\in S_1}P_{G^{(1)}_i,\pi^{(1)}}\right]\mathbb E\left[\prod_{i\in S_2}P_{G^{(2)}_i,\pi^{(2)}}\right] \enspace . 
\end{align}

Write $\bar G_\Delta = (\bar V_\Delta, \bar E_\Delta), \bar G_\cap=(\bar V_\cap, \bar E_\cap),  \bar G_\cup=(\bar V_\cup, \bar E_\cup)$ for the resp.~symmetric difference, intersection and union graphs corresponding to $(\bar \pi^{(1)}, \bar \pi^{(2)})$.
So  by Eq.~\eqref{eqn:old basis proof ingredients}
\begin{align}\label{eq:coolio2}
    \mathbb E\left[P_{G^{(1)}, \bar \pi^{(1)}}P_{G^{(2)}, \bar \pi^{(2)}}\right]& = \bbE\cro{\prod_{(i,j)\in \bar E_{\Delta}} (\lambda X_{ij}) \prod_{(i,j)\in  \bar E_{\cap}}[\bar q+\lambda X_{ij}(1-2q)]}\enspace  . 
\end{align}
We have chosen $\bar \pi^{(1)} = \pi^{(1)}$. Hence, we have $\bar E_\cap \subseteq E_\cap$. 
Since $\bar q+(1-2q)\lambda \geq \lambda$, we again deduce from~\eqref{eqn:old basis proof ingredients} that 
\begin{align}
   \mathbb{E}\left[P_{G^{(1)}, \pi^{(1)}}P_{G^{(2)}, \pi^{(2)}}\right] 
    &= \bbE\cro{\prod_{(i,j)\in E_{\Delta}} (\lambda X_{ij}) \prod_{(i,j)\in E_{\cap}}[\bar q+\lambda X_{ij}(1-2q)]} \nonumber \\ 
   &\geq   \bbE\cro{\prod_{(i,j)\in E_{\cup}\setminus \bar E_\cap} (\lambda X_{ij})\prod_{(i,j)\in  \bar E_{\cap}}[\bar q+\lambda X_{ij}(1-2q)]}\enspace , 
    \label{eq:coolio0}   
\end{align}

\begin{lemma}\label{lem:ineqsbm}
We have
$$\bbE\cro{\prod_{(i,j)\in  \bar E_{\Delta}} X_{ij} \prod_{(i,j)\in  \bar E_{\cap}}[\bar q+\lambda X_{ij}(1-2q)]} \leq \bbE\cro{\prod_{(i,j)\in E_{\cup} \setminus \bar E_\cap} X_{ij}\prod_{(i,j)\in  \bar E_{\cap}}[\bar q+\lambda X_{ij}(1-2q)]}\enspace .$$    
\end{lemma}


By Equation~\eqref{eq:coolio2} and Equation~\eqref{eq:coolio0}, we then deduce that 
\begin{align}\label{eq:coolio3}
    \mathbb E\left[P_{G^{(1)}, \bar \pi^{(1)}}P_{G^{(2)}, \bar \pi^{(2)}}\right]\leq \lambda^{|\bar E_\Delta| - |E_\cup\setminus \bar{E}_{\cap}|}\mathbb{E}\left[P_{G^{(1)}, \pi^{(1)}}P_{G^{(2)}, \pi^{(2)}}\right]= \lambda^{|E_{\cap}\setminus \bar E_{\cap}|}\mathbb{E}\left[P_{G^{(1)}, \pi^{(1)}}P_{G^{(2)}, \pi^{(2)}}\right]\enspace ,
\end{align}
because $|E_\cup\setminus \bar{E}_{\cap}|= |E_{\Delta}| + |E_{\cap}\setminus \bar E_{\cap}|$ and $|\bar E_\Delta|= |E_{\Delta}| + 2|E_{\cap}\setminus \bar E_{\cap}|$. The cardinality $|\mathbf M_{\mathrm{PM}}^{(S_1,S_2)}|$ is, by definition, smaller than the number of nodes in $S_1$ or in $S_2$ that are not isolated and are perfectly matched. By definition, any such perfectly matched node in $S_1$ and in $S_2$ is incident to at least an edge in $E_{\cap}\setminus \bar E_{\cap}$. We deduce that 
\[
|\mathbf M_{\mathrm{PM}}^{(S_1,S_2)}|\leq 2|E_\cap \setminus \bar E_\cap|\enspace .
\] 
Coming back to~\eqref{eq:coolio3}, we obtain 
\[
   \mathbb E\left[P_{G^{(1)}, \bar \pi^{(1)}}P_{G^{(2)}, \bar \pi^{(2)}}\right]\leq\lambda^{|\mathbf M_{\mathrm{PM}}^{(S_1,S_2)}|/2}\mathbb{E}\left[P_{G^{(1)}, \pi^{(1)}}P_{G^{(2)}, \pi^{(2)}}\right]\enspace ,
\]
which, combined with \eqref{eq:coolio}, concludes the proof. 

\end{proof}

\begin{proof}[Proof of Lemma~\ref{lem:ineqsbm}]


Let us start by developing the products of $\bar q+\lambda X_{ij}(1-2q)$. 

\begin{align*}
    \bbE\cro{\prod_{(i,j)\in \bar E_{\Delta}} X_{ij} \prod_{(i,j)\in \bar E_{\cap}}[\bar q+\lambda X_{ij}(1-2q)]} & = \bbE\cro{\sum_{S \subseteq \bar E_\cap}\prod_{(i,j)\in \bar E_{\Delta}} X_{ij} \bar q^{|\bar E_\cap| - |S|}\prod_{(i,j)\in S}[\lambda X_{ij}(1-2q)]}\ ; \\ 
    \bbE\cro{\prod_{(i,j)\in E_\cup\setminus \bar E_\cap} X_{ij} \prod_{(i,j)\in \bar E_{\cap}}[\bar q+\lambda X_{ij}(1-2q)]} &= \bbE\cro{\sum_{S \subseteq \bar E_\cap}\prod_{(i,j)\in  E_{\cup}\setminus \bar E_\cap} X_{ij} \bar q^{|\bar E_\cap| - |S|}\prod_{(i,j)\in S}[\lambda X_{ij}(1-2q)]}\enspace .  
\end{align*}

Hence, it order to conclude, if suffices to prove that, for any $S \subseteq \bar E_\cap$, we have
\begin{equation}\label{eq:toprove}
    \bbE\cro{\prod_{(i,j)\in  S\cup (E_{\cup}\setminus \bar E_\cap)} X_{ij}} \geq \bbE\cro{\prod_{(i,j)\in S\cup \bar E_{\Delta}} X_{ij} }\enspace . 
\end{equation}

To show~\eqref{eq:toprove}, we fix such a subset $S\subseteq \bar E_\cap$. Define the graphs $H$ and $\bar {H}$ respectively induced by the set of edges $S\cup (E_{\cup}\setminus \bar E_\cap)$ and $S\cup \bar E_{\Delta}$. Then denoting $|V(H)|$, $|V(\bar H)|$ their number of nodes and  $\#\mathrm{CC}_H$ and $\#\mathrm{CC}_{\bar H}$ their number of connected components; we have 
\[
  \bbE\cro{\prod_{(i,j)\in  S\cup (E_{\cup}\setminus \bar E_\cap)} X_{ij}}= K^{-|V(H)|+ \#\mathrm{CC}_H}  \ , \quad \quad \bbE\cro{\prod_{(i,j)\in S\cup \bar E_{\Delta}} X_{ij} }= K^{-|V(\bar H)|+ \#\mathrm{CC}_{\bar H}} \ , 
\]
so that~\eqref{eq:toprove} is equivalent to 
\begin{equation}\label{eq:toprove2}
    |V(H)|- \#\mathrm{CC}_H\leq |V(\bar H)|- \#\mathrm{CC}_{\bar H}\ . 
\end{equation}
Thus, we only have to establish this last inequality.

Observe that $H$ is the subgraph of $G_\cup$ where we have removed the edges $\bar E_\cap \setminus S$ and have pruned the isolated nodes. Similarly, $\bar{H}$ is the subgraph of $\bar G_{\cup}$ where we have  removed the edges $\bar E_\cap \setminus S$ and have pruned the isolated nodes. By definition of $\bar G_{\cup}$ at the beginning of the proof of Lemma~\ref{lem:unioncomp}, $G_{\cup}$ is obtained from $\bar G_{\cup}$ by merging/identifying couples of nodes. If a node has been pruned in the construction of $H$ or $\bar{H}$, then this implies that this node was only incident to edges in $\bar E_\cap \setminus S$. As a consequence, pruned nodes in $\bar{H}$ are in correspondence with those in $H$. Hence, to go from $\bar{H}$ to $H$, we merge/identify $s\geq 0$ couples of nodes in $\bar H$. The operation of merging two nodes decreases the number of nodes of the graph by one and decreases its number of connected components by at most one. Iterating the arguments $s$ times, we conclude that $  |V(H)|- \#\mathrm{CC}_H\leq |V(\bar H)|- \#\mathrm{CC}_{\bar H}$, which concludes the proof. 

\end{proof}

\begin{proof}[Proof of Proposition~\ref{prop:scalprod}]
By definition of $\Psi_G$, we have 
\begin{align*}
     \mathbb E\left[\Psi_{G^{(1)}}\Psi_{G^{(2)}}\right]  &= \sum_{\mathbf{M}\in\mathcal{M}}\sum_{(\pi^{(1)}, \pi^{(2)})\in \Pi(\mathbf{M})}\frac{1}{\sqrt{\mathbb{V}(G^{(1)})\mathbb{V}(G^{(2)})}}\mathbb{E}\left[\overline{P}_{G^{(1)}, \pi^{(1)}}\overline{P}_{G^{(2)}, \pi^{(2)}}\right]\\
     &= \sum_{\mathbf{M}\in\mathcal{M}^{\star}}\sum_{(\pi^{(1)}, \pi^{(2)})\in \Pi(\mathbf{M})}\frac{1}{\sqrt{\mathbb{V}(G^{(1)})\mathbb{V}(G^{(2)})}}\mathbb{E}\left[\overline{P}_{G^{(1)}, \pi^{(1)}}\overline{P}_{G^{(2)}, \pi^{(2)}}\right],
    \end{align*}
where the second line follows by Proposition~\ref{prop:boundbar}. 

\medskip 
\noindent
{\bf Step 1: Decomposition of the scalar product over $\mathcal M^\star \setminus \mathcal M_{\mathrm{PM}}$ and $ \mathcal M_{\mathrm{PM}}$.} The collection $\mathcal{M}_{\mathrm{PM}}$ is non-empty only if $G^{(1)} = G^{(2)}$. If $G^{(1)} = G^{(2)}$, we have
\begin{align*}
\sum_{\mathbf{M}\in\mathcal{M}_{\mathrm{PM}}}\sum_{(\pi^{(1)}, \pi^{(2)})\in \Pi(\mathbf{M})}\frac{1}{\sqrt{\mathbb{V}(G^{(1)})\mathbb{V}(G^{(2)})}}\mathbb{E}\left[\overline{P}_{G^{(1)}, \pi^{(1)}}\overline{P}_{G^{(2)}, \pi^{(2)}}\right] &=  \frac{(n-2)!|\mathrm{Aut}(G^{(1)})|\mathbb{E}\left[\overline{P}_{G^{(1)}, \pi}^2\right]}{\left(n-|V^{(1)}|\right)!\mathbb{V}(G^{(1)})}  \\
&=  \frac{\mathbb E[\overline{P}_{G^{(1)}, \pi}^2]}{\mathbb E[P_{G^{(1)}, \pi}^2]},\end{align*}
since  we have 
\begin{align}\label{eq:boubou}
    |\mathcal{M}_{\mathrm{PM}}| =|\mathrm{Aut}(G^{(1)})|~~~~\mathrm{and}~~~~|\Pi(\mathbf{M})| = \frac{(n-2)!}{\left(n-(|V^{(1)}|+ |V^{(2)}| - |\mathbf M|)\right)!}.
\end{align}
Then, it follows from the second part of  Proposition~\ref{prop:boundbar} that 
\begin{align}\nonumber
     &\left|\mathbb E\left[\Psi_{G^{(1)}}\Psi_{G^{(2)}}\right] - \mathbf 1\{G^{(1)} = G^{(2)}\} \right|\\
     &\leq  \left|\sum_{\mathbf{M}\in\mathcal{M}^{\star}\setminus \mathcal M_{\mathrm{PM}}}\sum_{(\pi^{(1)}, \pi^{(2)})\in \Pi(\mathbf{M})}\frac{1}{\sqrt{\mathbb{V}(G^{(1)})\mathbb{V}(G^{(2)})}}\mathbb{E}\left[\overline{P}_{G^{(1)}, \pi^{(1)}}\overline{P}_{G^{(2)}, \pi^{(2)}}\right]\right|  + D^{-4\texttt{c}_s[d(G^{(1)},G^{(2)})\vee 1]} \nonumber\\ 
   &  \coloneqq A  + D^{-4\texttt{c}_s[d(G^{(1)},G^{(2)})\vee 1]}\enspace . \label{eq:upper_correlation}
    \end{align}

\noindent
{\bf Step  2: Making $A$ explicit as a sum of $A(\mathbf M)$.} Observe that, for any $\mathbf M \in \mathcal M$,  $\mathbb{E}\left[P_{G^{(1)}, \pi^{(1)}}P_{G^{(2)}, \pi^{(2)}}\right]$ is identical for any $(\pi^{(1)}, \pi^{(2)})\in \Pi(\mathbf{M})$. So that by Proposition~\ref{prop:boundbar} and Equation~\eqref{eq:boubou}
\begin{align}
    A&{\sqrt{\mathbb{V}(G^{(1)})\mathbb{V}(G^{(2)})}} = \left|\sum_{\mathbf{M}\in\mathcal{M}^{\star}\setminus \mathcal{M}_{\mathrm{PM}}}\sum_{(\pi^{(1)}, \pi^{(2)})\in \Pi(\mathbf{M})}\mathbb{E}\left[\overline{P}_{G^{(1)}, \pi^{(1)}}\overline{P}_{G^{(2)}, \pi^{(2)}}\right]\right|\nonumber\\
    &\leq 
    \sum_{\mathbf{M}\in\mathcal{M}^{\star}\setminus \mathcal{M}_{\mathrm{PM}}}\frac{(n-2)!}{\left(n-(|V^{(1)}|+ |V^{(2)}| - |\mathbf M|)\right)!}\left|\mathbb{E}\left[P_{G^{(1)}, \pi^{(1)}}P_{G^{(2)}, \pi^{(2)}}\right]\right| \pa{2^{|\mathbf M_{\mathrm{SM}}|}(1+D^{-4\texttt{c}_s})+1}\nonumber \\
    &\leq 2
    \sum_{\mathbf{M}\in\mathcal{M}^{\star}\setminus \mathcal{M}_{\mathrm{PM}}}\frac{(n-2)!}{\left(n-(|V^{(1)}|+ |V^{(2)}| - |\mathbf M|)\right)!}\left|\mathbb{E}\left[P_{G^{(1)}, \pi^{(1)}}P_{G^{(2)}, \pi^{(2)}}\right]\right|2^{|\mathbf M_{\mathrm{SM}}|}\enspace ,\nonumber
    \end{align}
since $D\geq 2$.
    
Observe that $\frac{(n-2)!}{\left(n-(|V^{(1)}|+ |V^{(2)}| - |\mathbf M|)\right)!} \frac{\sqrt{(n-|V^{(1)|})!(n-|V^{(2)|})!}}{(n-2)!}\leq n^{\frac{|U^{(1)}| + |U^{(2)}|}{2}}$. Then, by Proposition~\ref{prop:christophe}, we get 
    \begin{align*}
    A &\sqrt{\left|\mathrm{Aut}(G^{(1)})\right|\left|\mathrm{Aut}(G^{(2)})\right|}\\ 
    &\leq  2\sum_{\mathbf{M}\in\mathcal{M}^{\star}\setminus \mathcal{M}_{\mathrm{PM}}}n^{\frac{|U^{(1)}| + |U^{(2)}|}{2}}
    \left[\left(\frac{\sqrt{n}}{K^{|U^{(1)}|+ |U^{(2)}|}} \left(\frac{\lambda}{\sqrt{\bar q}}\right)^{|E_{\Delta}|}\right) \land \left(\frac{\lambda^{|E_{\Delta}|/2}}{K^{(|U^{(1)}|+ |U^{(2)}|)/2}} \right)\right]2^{|\mathbf M_{\mathrm{SM}}|}\nonumber\\
       & =  2
        \sum_{\mathbf{M}\in\mathcal{M}^{\star}\setminus \mathcal{M}_{\mathrm{PM}}}\left[\left(\frac{\sqrt{n}}{K} \left(\frac{\lambda}{\sqrt{\bar q}}\right)^{r}\right) \land \left(\frac{\lambda^{r/2}}{\sqrt{K/n}} \right)\right]^U2^{|\mathbf M_{\mathrm{SM}}|}\enspace ,\nonumber
\end{align*}
where $U=U(\mathbf M) = |U^{(1)}|+ |U^{(2)}|$ and $r=r(\mathbf M) = |E_\Delta|/U$. Write $A(\mathbf M)$ for the summand in the last line.

\medskip

\noindent
{\bf Step 3: Bounding of $A$ by summing over shadows.} Recall the definition of shadows and of $\mathcal M_{\mathrm{shadow}}$ in Section~\ref{sec:graph:definition}. We now regroup the sum inside $A$ by enumerating  all possible matching  that are compatible with a  shadow. We get 
\begin{align*}
       A &\leq  \frac{2}{\sqrt{|\mathrm{Aut}(G^{(1)})| |\mathrm{Aut}(G^{(2)})|}} 
       \sum_{\substack{W^{(1)} \subset V^{(1)}, W^{(2)} \subset V^{(2)}\\ \underline{\mathbf M} \in \mathcal M\setminus \mathcal M_{\mathrm{PM}}}}\quad 
       \sum_{\mathbf M \in \mathcal M_{\mathrm{shadow}}(W^{(1)},W^{(2)}, \underline{\mathbf M})} A(\mathbf M)\enspace .
\end{align*}
Remark that $A(\mathbf M)$ is the same for all $\mathbf M \in \mathcal M_{\mathrm{shadow}}(W^{(1)},W^{(2)}, \underline{\mathbf M})$ and only depends on $U=|U^{(1)}|+ |U^{(2)}|=|W^{(1)}|+|W^{(2)}|$, and $\underline{\mathbf{M}}$. Besides, we have $\mathbf M_{\mathrm{SM}}=\underline{\mathbf M}$. 
We have the following control for the cardinality of $\mathcal M_{\mathrm{shadow}}$.
\begin{lemma}\label{lem:shadow}
For any $W^{(1)}$, $W^{(2)}$, and $\underline{\mathbf M}$, we have
    $$|\mathcal M_{\mathrm{shadow}}(W^{(1)}, W^{(2)}, \underline{\mathbf M} )| \leq \min(|\mathrm{Aut}(G^{(1)})|, |\mathrm{Aut}(G^{(2)})|)\enspace .$$
\end{lemma}
Observe that two matchings $\mathbf M$ and $\mathbf M'$ that belong $\mathcal M_{\mathrm{shadow}}(W^{(1)},W^{(2)}, \underline{\mathbf M})$ have isomorphic symmetric difference graph $G_{\Delta}$ and have a common value of $|\mathbf M_{\mathrm{SM}}|$. Hence
\begin{align}
       A &\leq  2 
       \sum_{W^{(1)} \subset V^{(1)}, W^{(2)} \subset V^{(2)}, \underline{\mathbf M} \in \mathcal M\setminus \mathcal M_{\mathrm{PM}}}  A(\mathbf M)\enspace ,\label{eqn:sum for proof sketch}
\end{align}
where $\mathbf{M}$ is any matching in   $\mathcal M_{\mathrm{shadow}}(W^{(1)},W^{(2)}, \underline{\mathbf M})$.

\noindent
{\bf Step 4: Bounding $A(\mathbf M)$.}  Recall that $|\mathbf{M}_{\mathrm{SM}}| + |U^{(1)}| + |U^{(2)}| = |V_{\Delta}|$.
For $\mathbf{M}\in \mathcal{M}^\star$, we have $|\mathbf{M}_{\mathrm{SM}}|\geq  \#\mathrm{CC}_{\Delta}$. 
Since $|E_{\Delta}|\geq |V_{\Delta}|- \#\mathrm{CC}_{\Delta}$, we deduce that $|E_{\Delta}|\geq |U|$ which, in turn, implies that $r(\mathbf M)\geq 1$.  
Also, by definition of the edit distance, we have $|E_{\Delta}| \geq d(G^{(1)}, G^{(2)})\vee 1$. Finally, we have $|E_{\Delta}| \geq |\mathbf{M}_{\mathrm{SM}}| /2$ as any semi-matched node is connected to an edge in $G_{\Delta}$. Gathering these three lower bounds, we get 
\begin{align*}
    A(\mathbf M) &\leq 2^{|\mathbf M_{\mathrm{SM}}|} D^{-8c_{\texttt{s}}|E_\Delta|} \leq 2^{|\mathbf M_{\mathrm{SM}}|}D^{-2c_{\texttt{s}}[|U|+ |\mathbf{M}_{\mathrm{SM}}|+ d(G^{(1)},G^{(2)})\vee 1]}\enspace .
\end{align*}

\noindent
{\bf Step 5: Final bound on $A$.} Plugging this bound on $A(\mathbf M)$ back in Equation~\eqref{eqn:sum for proof sketch} we get
\begin{align}
       A &\leq  2 
       \sum_{W^{(1)} \subset V^{(1)}, W^{(2)} \subset V^{(2)}, \underline{\mathbf M} \in \mathcal M\setminus \mathcal M_{\mathrm{PM}}}  2^{|\mathbf M_{\mathrm{SM}}|} D^{-2c_{\texttt{s}}[|W^{(1)}|+ |W^{(2)}|+ |\mathbf{M}_{\mathrm{SM}}|+ d(G^{(1)},G^{(2)})\vee 1]}, 
\end{align}
since for $\mathbf M\in \mathcal M_{\mathrm{shadow}}(W^{(1)},W^{(2)}, \underline{\mathbf M})$, we have $|U^{(1)}|= |W^{(1)}|$ and $|U^{(2)}|=W^{(2)}$. 
So when we enumerate over all possible sets $W^{(1)}, W^{(2)}, \underline{\mathbf M} $ that have respective cardinality $u_1$, $u_2$, and  $m$, and since these sets have bounded cardinalities resp.~by $(2D)^{u_1}, (2D)^{u_2}$ and $(2D)^{2m}$, we have
\begin{align*}
       A &\leq  2
       \sum_{u_1, u_2, m \geq 0}2^m(2D)^{u_1+u_2+2m}  D^{-2c_{\texttt{s}}\left[d(G^{(1)}, G^{(2)})\vee 1+u_1+u_2+m\right]} \leq 2D^{-c_{\texttt{s}}  d(G^{(1)},G^{(2)})\vee 1},
\end{align*}
since $c_{\texttt{s}} \geq 5$ and $D \geq 2$. Together with~\eqref{eq:upper_correlation}, this concludes the proof.  
\end{proof}

\begin{proof}[Proof of Proposition~\ref{prop:ortho}]
Since the operator norm of a symmetric matrix is bounded by the maximum $L_1$ norm of its rows, we have 
\[
\|\Gamma - \mathrm{Id}\|_{op}\leq \max_{G^{(1)}}\left|\Gamma_{G^{(1)}, G^{(1)}}-1\right|+  \sum_{G^{(2)} \in \mathcal G_{\leq D}, G^{(2)} \neq G^{(1)} }\left|\Gamma_{G^{(1)}, G^{(2)}}\right|
\] 
To bound the latter sum, we use that for a fixed template $G^{(1)}$, the number of templates $G^{(2)}\in \mathcal{G}_{\leq D}$ such that $d( G^{(1)}, G^{(2)}) = u$ is bounded by $(u+D)^{2u}$. 
We also use that if $G^{(2)} \neq G^{(1)}$, then $d(G^{(1)}, G^{(2)})\geq 1$ as they are not isomorphic. It then follows from Proposition~\ref{prop:scalprod} that
\begin{align*}
    \sum_{G^{(2)} \in \mathcal G_{\leq D}, G^{(2)} \neq G^{(1)} }\left|\Gamma_{G^{(1)}, G^{(2)}}\right|
    &\leq \sum_{G^{(2)} \in \mathcal G_{\leq D}, G^{(2)} \neq G^{(1)}} 3D^{-c_{\texttt{s}}d(G^{(1)},G^{(2)})} \\
    &\leq \sum_{2D \geq  u \geq 1} |\{G^{(2)}: d( G^{(1)}, G^{(2)}) = u \}|  3D^{-c_{\texttt{s}}u}\enspace\\
    &\leq \sum_{2D \geq u\geq 1} 3(u+D)^{2u}  D^{-c_{\texttt{s}}u}\enspace
    \leq \sum_{2D \geq u\geq 1} 3D^{-(c_{\texttt{s}}-6)u} \leq  3D^{-c_{\texttt{s}}/2}\enspace ,
\end{align*}
since $D \geq 2$ provided we have $c_{\texttt{s}} \geq 14$. 
Using again  Proposition~\ref{prop:scalprod}, to bound $\left|\Gamma_{G^{(1)}, G^{(1)}}-1\right|$  concludes the proof.
\end{proof}

\begin{proof}[Proof of Lemma~\ref{lem:shadow}] 
The proof is a straighforward variation of that of Lemma A.5 in~\cite{CGGV25}, the twist being that $|\mathrm{Aut}(G)|$ is restrict to automorphisms that let $v_1$ and $v_2$ fixed, whereas all matchings $\mathbf{M}$ of $(G^{(1)},G^{(2)})$ contain $(v_1^{(1)},v_1^{(2)})$ and $(v_2^{(1)},v_2^{(2)})$.  
\end{proof}

\section{Proof of Proposition~\ref{prop:christophe}}

Consider two templates $G^{(1)}, G^{(2)}$, some node matching $\mathbf M \in \mathcal M^*$ and two injections $(\pi^{(1)}, \pi^{(2)})\in \Pi(\mathbf{M})$. 
The expectation $ \bbE\cro{P_{G^{(1)}, \pi^{(1)}}P_{G^{(2)}, \pi^{(2)}}}$ only depends on the graphs $G^{(1)}\cro{\pi^{(1)}}$ and $G^{(2)}\cro{\pi^{(2)}}$. 
To avoid clumsy notations, and since the injections $(\pi^{(1)}, \pi^{(2)})\in \Pi(\mathbf{M})$ are fixed, we directly work in this proof with the graphs $G^{(i)}\cro{\pi^{(i)}}=:(V^{(i)},E^{(i)})$, with $i=1,2$. Also, we shall introduce new notations to account for the case where $v_1$ or $v_2$ are isolated in $G^{(1)}$ or in $G^{(2)}$.
For $i=1,2$, we define $G^{\star(i)}=(V^{\star(i)},E^{(i)})$  by removing isolated nodes. 
We recall the notation  $E_{\cap}= E^{(1)}\cap E^{(2)}$ and $E_{\Delta}= E^{(1)}\Delta E^{(2)}$. We partition $E_{\Delta}$ according to these connected components $E_{\Delta}=\cup_{\ell=1}^{\#\mathrm{CC}_{\Delta}} E_{\Delta,\ell}$. Since we work with the pruned graph, we shall redefine the node set. In particular, we define $V^{\star}_{\cap}= V^{\star(1)}\cap  V^{\star(2)}$, we partition $V^{\star}_{\cap}$ into the sets $V^{\star}_{PM}$ and $V^{\star}_{SM}$ of perfectly and semi-matched nodes. Besides we define $U^{\star(1)}= V^{\star(1)}\setminus U^{\star(2)}$, $U^{\star(2)}= V^{\star(2)}\setminus U^{\star(1)}$, and $U^{\star}= U^{\star(1)}\cup U^{\star(2)}$. Obviously, we have $|U|=|U^{\star}|$ when neither $v_1$ nor $v_2$ is isolated in both $G^{(1)}$ and $G^{(2)}$. In general, we have $|U^{\star}|\in [|U|, |U|+2]$.
 Given a subset $W$ of $V^{\star}_{\cap}$, we define $E_{\cap}[W]$ as the subset of edges in $E_{\cap}$ that connects nodes in $W$. We first focus on proving~\eqref{eq:christophe1}.

We have $\bbE\cro{Y_{ij}|X}=\lambda X_{ij}=\lambda  \mathbf{1}_{z_{i}=z_{j}}$ and 
$$\bbE\cro{Y_{ij}^2|X}= \bar q+ \lambda X_{ij}(1-2q)=\bar q \pa{\bar p\over \bar q}^{ \mathbf{1}_{z_{i}=z_{j}}},$$ with $\bar p=\bar q+\lambda (1-2q)\geq \bar q$ for $q\leq 1/2$.
Since the $Y_{ij}$ are conditionally independent given $X$, and since for $\mathbf{M}\in\mathcal{M}^\star$, no connected component of $G_{\Delta}$ is only composed of nodes from $U^{\star(j)}$, we have
\begin{align}\nonumber
   \bbE\cro{P_{G^{(1)}, \pi^{(1)}}P_{G^{(2)}, \pi^{(2)}}} 
    & = \bbE\cro{\prod_{(i,j)\in E_{\Delta}}Y_{ij}\prod_{(i,j)\in E_{\cap}}Y_{ij}^2}  \\ \label{eq:proof:ingredient:basis}
    &=   \bbE\cro{\prod_{(i,j)\in E_{\Delta}} (\lambda \mathbf{1}_{z_{i}=z_{j}}) \prod_{(i,j)\in E_{\cap}}\bar q \pa{\bar p\over \bar q}^{ \mathbf{1}_{z_{i}=z_{j}}}}\\ \nonumber
    &= \lambda^{|E_{\Delta}|} \bar q^{|E_{\cap}|} \bbE\cro{\prod_{\ell=1}^{\#\mathrm{CC}_{\Delta}}\prod_{(i,j)\in E_{\Delta,\ell}} \mathbf{1}_{z_{i}=z_{j}} \prod_{k=1}^K\prod_{\substack{(i,j)\in E_{\cap}\\ \nonumber z_{i}=z_{j}=k}}\pa{\bar p\over \bar q}}\\ \nonumber
    = \lambda^{|E_{\Delta}|} \bar q^{|E_{\cap}|}& \sum_{\varphi_{\Delta}\in [K]^{\#\mathrm{CC}_{\Delta}}} \sum_{\varphi_{PM}\in [K]^{V_{PM}}}\bbP\cro{z\sim(\varphi_{\Delta},\varphi_{PM})} \prod_{k=1}^K \pa{\bar p\over \bar q}^{|E_{\cap}\cro{\varphi_{SM}^{-1}(\ac{k})\cup\varphi_{PM}^{-1}(\ac{k})}|},
    \end{align}
where, for a given $\varphi_{\Delta}:\cro{\#\mathrm{CC}_{\Delta}}\to [K]$, the assignment $\varphi_{SM}:V^{\star}_{SM}\to [K]$ is defined by 
\begin{equation}\label{eq:phi-SM}
\varphi_{SM}(i)=\varphi_{\Delta}(\ell)\ \text{for}\ i\in V^{\star}_{SM}\cap V^{\star}_{\Delta,\ell}\ ,
\end{equation}
 and  where
\begin{align*}
\bbP\cro{z\sim(\varphi_{\Delta},\varphi_{PM})} & := \bbP\cro{z_{i}=\varphi_{\Delta}(\ell), \text{for all}\ i\in V^{\star}_{\Delta,\ell}\ \ \text{and}\ \ z_{i}=\varphi_{PM}(i),\ \text{for all}\ i\in V^{\star}_{PM}}\\
& = K^{-|V^{\star}_{\Delta}|-|V^{\star}_{PM}|} = K^{-|V^{\star}_{\cup}|}.
\end{align*}
So
\begin{equation}\label{eq:cross:formula}
   \bbE\cro{P_{G^{(1)}, \pi^{(1)}}P_{G^{(2)}, \pi^{(2)}}} = {\lambda^{|E_{\Delta}|} \bar q^{|E_{\cap}|}\over K^{|V^{\star}_{\cup}|}} \sum_{\varphi_{\Delta}\in [K]^{\#\mathrm{CC}_{\Delta}}} \sum_{\varphi_{PM}\in [K]^{V^{\star}_{PM}}}  \pa{\bar p\over \bar q}^{\sum_{k=1}^K|E_{\cap}\cro{\varphi_{SM}^{-1}(\ac{k})\cup\varphi_{PM}^{-1}(\ac{k})}|},
\end{equation}
We have in particular for $j=1,2$
\begin{equation}\label{eq:moment2}
\bbE\cro{P_{G^{(j)}, \pi^{(j)}}^2} =   {\bar q^{|E^{(j)}|}\over K^{|V^{^{\star}(j)}|}} \sum_{\varphi_{j}\in [K]^{V^{^{\star}(j)}}}  \pa{\bar p\over \bar q}^{\sum_{k=1}^K|E^{(j)}\cro{\varphi_{j}^{-1}(\ac{k})}|}.
\end{equation}
To lower bound $\bbE[P_{G^{(j)}, \pi^{(j)}}^2]$, we consider two different subfamilies of assignments $\varphi_{j}:V^{^{\star}(j)}\to[K]$ in the sum of \eqref{eq:moment2}. For some given $\varphi_{\Delta}:[\#\mathrm{CC}_{\Delta}]\to[K]$ and some $\varphi_{PM}:V^{\star}_{PM}\to [K]$, we consider
\begin{enumerate}
\item $\varphi_{j}$ defined by $\varphi_j(i)=\varphi_{\Delta}(\ell)$, for all $i\in V^{\star}_{\Delta,\ell}$; $\varphi_j(i)=\varphi_{PM}(i)$, for all $i\in V^{\star}_{PM}$;
\item $\varphi_{j}$ defined by $\varphi_j(i)=\varphi_{\Delta}(\ell)$, for all $i\in V^{\star}_{\Delta,\ell}\cap V^{\star}_{SM};\ \varphi_j(i)=\varphi_{PM}(i)$, for all $i\in V^{\star}_{PM}$, and $\varphi_j(i)=\varphi_{U^{\star(j)}}(i)$ for $i\in U^{\star(j)}$, where $\varphi_{U^{\star(j)}}\in [K]^{U^{\star(j)}}$.
\end{enumerate}

\paragraph{First lower bound.}
Let us consider some $\varphi_{\Delta}:[\#\mathrm{CC}_{\Delta}]\to[K]$ and  $\varphi_{PM}:V^{\star}_{PM}\to [K]$.
Let us restrict to assignments $\varphi_{j}:V^{^{\star}(j)}\to[K]$ defined by $\varphi_j(i)=\varphi_{\Delta}(\ell)$, for all $i\in V^{\star}_{\Delta,\ell}$; and $\varphi_j(i)=\varphi_{PM}(i)$, for all $i\in V^{\star}_{PM}$. 
We observe that $\varphi_{j}^{-1}(\ac{k}) \cap (V^{\star}_{SM}\cup V^{\star}_{PM})=\varphi_{SM}^{-1}(\ac{k})\cup\varphi_{PM}^{-1}(\ac{k})$, so 
\begin{align}
\left|E^{(j)}\cro{\varphi_{j}^{-1}(\ac{k})}\right| & = \left| E_{\cap}\cro{\varphi_{SM}^{-1}(\ac{k})\cup\varphi_{PM}^{-1}(\ac{k})}\right| + \left| \pa{E_{\Delta}\cap E^{(j)}}\cro{\varphi_{j}^{-1}(\ac{k})\cap \pa{V^{\star}_{SM}\cup U^{\star(j)}}}\right|, \label{eq:decomp:simple-double}
\end{align}
with
\begin{align}
\sum_{k=1}^K \left|\pa{E_{\Delta}\cap E^{(j)}}\cro{\varphi_{j}^{-1}(\ac{k})\cap \pa{V^{\star}_{SM}\cup U^{\star(j)}}} \right|
& = \sum_{\ell =1}^{\#\mathrm{CC}_{\Delta}}\sum_{k=1}^K \left|\pa{E_{\Delta,\ell}\cap E^{(j)}}\cro{\varphi_{j}^{-1}(\ac{k})\cap \pa{V^{\star}_{SM}\cup U^{\star(j)}}}\right|\nonumber\\
& = \sum_{\ell =1}^{\#\mathrm{CC}_{\Delta}} \left|\pa{E_{\Delta,\ell}\cap E^{(j)}} \right|
 = \left|E_{\Delta}\cap E^{(j)}\right|, \label{eq:partition:pure}
\end{align}
since $\varphi_{j}$ is constant on each $V^{\star}_{\Delta,\ell}$.
So, we can lower-bound \eqref{eq:moment2} by
\begin{align}
\bbE\cro{P_{G^{(j)}, \pi^{(j)}}^2} & \geq  {\bar q^{|E^{(j)}|}\over K^{|V^{^{\star}(j)}|}} \pa{\bar p \over \bar q}^{\left|E_{\Delta}\cap E^{(j)}\right|}  \sum_{\varphi_{\Delta}\in [K]^{\#\mathrm{CC}_{\Delta}}} \sum_{\varphi_{PM}\in [K]^{V^{\star}_{PM}}}  \pa{\bar p\over \bar q}^{\sum_{k=1}^K|E_{\cap}\cro{\varphi_{SM}^{-1}(\ac{k})\cup\varphi_{PM}^{-1}(\ac{k})}|}\enspace . \nonumber
\end{align}
Since $|E_{\Delta}|=|E_{\Delta}\cap E^{(1)}|+|E_{\Delta}\cap E^{(2)}|$ and
$2|E_{\cap}|=|E^{(1)}|+|E^{(2)}|-|E_{\Delta}\cap E^{(1)}|- |E_{\Delta}\cap E^{(2)}|$,
we get a first lower-bound
\begin{align}
\prod_{j=1,2}\bbE\cro{P_{G^{(j)}, \pi^{(j)}}^2}^{1/2} & \geq {\bar q^{|E_{\cap}|}\,\bar p^{|E_{\Delta}|/2}\over K^{\pa{|V^{\star(1)}|+|V^{\star(2)}|}/2}} \sum_{\varphi_{\Delta}\in [K]^{\#\mathrm{CC}_{\Delta}}} \sum_{\varphi_{PM}\in [K]^{V^{\star}_{PM}}}  \pa{\bar p\over \bar q}^{\sum_{k=1}^K\left|E_{\cap}\cro{\varphi_{SM}^{-1}(\ac{k})\cup\varphi_{PM}^{-1}(\ac{k})}\right|} \nonumber\\
\label{eq:restrictedLB1}
&=  {\bar q^{|E_{\cap}|}\,\bar p^{|E_{\Delta}|/2}\over K^{|V^{\star}_{\cup}|-\pa{|U^{\star(1)}|+|U^{\star(2)}|}/2}} \sum_{\varphi_{\Delta}\in [K]^{\#\mathrm{CC}_{\Delta}}} \sum_{\varphi_{PM}\in [K]^{V^{\star}_{PM}}}  \pa{\bar p\over \bar q}^{\sum_{k=1}^K\left|E_{\cap}\cro{\varphi_{SM}^{-1}(\ac{k})\cup\varphi_{PM}^{-1}(\ac{k})}\right|}.
\end{align}

\paragraph{Second lower bound.} 
Let us consider again some $\varphi_{\Delta}:[\#\mathrm{CC}_{\Delta}]\to[K]$,  $\varphi_{PM}:V^{\star}_{PM}\to [K]$ and $\varphi_{U^{\star(j)}}: U^{\star(j)}\to [K]$.
We now restrict to assignments 
 $\varphi_{j}:V^{*(j)}\to[K]$ defined by $\varphi_j(i)=\varphi_{\Delta}(\ell)$, for all $i\in V^{*}_{\Delta,\ell}\cap V^{\star}_{SM};\ \varphi_j(i)=\varphi_{PM}(i)$, for all $i\in V^{\star}_{PM}$, and $\varphi_j(i)=\varphi_{U^{\star(j)}}(i)$ for $i\in U^{\star(j)}$.
We still have the decomposition \eqref{eq:decomp:simple-double}, but we do not have the identity \eqref{eq:partition:pure} anymore. 
Yet, we have $\left|E^{(j)}\cro{\varphi_{j}^{-1}(\ac{k})}\right| \geq \left| E_{\cap}\cro{\varphi_{SM}^{-1}(\ac{k})\cup\varphi_{PM}^{-1}(\ac{k})}\right|$. Since $\mathbf{M}\in\mathcal{M}^\star$, no connected component of $G_{\Delta}$ is only composed of nodes from $U^{\star(j)}$. 
Since $\bar p \geq \bar q$ when $q\leq 1/2$,  we can lower-bound \eqref{eq:moment2} by
\begin{align}
\bbE\cro{P_{G^{(j)}, \pi^{(j)}}^2} & \geq  {\bar q^{|E^{(j)}|}\over K^{|V^{\star(j)}|}} \nonumber
\sum_{\varphi_{U^{\star(j)}}\in [K]^{U^{\star(j)}}} \sum_{\varphi_{\Delta}\in [K]^{\#\mathrm{CC}_{\Delta}}} \sum_{\varphi_{PM}\in [K]^{V^{\star}_{PM}}}  \pa{\bar p\over \bar q}^{\sum_{k=1}^K|E_{\cap}\cro{\varphi_{SM}^{-1}(\ac{k})\cup\varphi_{PM}^{-1}(\ac{k})}|}\\
& \geq  {\bar q^{|E^{(j)}|}\over K^{|V^{\star(j)}|-|U^{\star(j)}|}}   \sum_{\varphi_{\Delta}\in [K]^{\#\mathrm{CC}_{\Delta}}} \sum_{\varphi_{PM}\in [K]^{V^{\star}_{PM}}}  \pa{\bar p\over \bar q}^{\sum_{k=1}^K|E_{\cap}\cro{\varphi_{SM}^{-1}(\ac{k})\cup\varphi_{PM}^{-1}(\ac{k})}|}\enspace , \nonumber
\end{align}
and get a second lower bound
\begin{align}
\prod_{j=1,2}\bbE\cro{P_{G^{(j)}, \pi^{(j)}}^2}^{1/2} & \geq {\bar q^{(|E^{(1)}|+|E^{(2)}|)/2}\over K^{\pa{|V^{\star(1)}|+|V^{\star(2)}|-|U^{\star(1)}|-|U^{\star(2)}|}/2}} \nonumber\\
&\quad \times \sum_{\varphi_{\Delta}\in [K]^{\#\mathrm{CC}_{\Delta}}} \sum_{\varphi_{PM}\in [K]^{V^{\star}_{PM}}}  \pa{\bar p\over \bar q}^{\sum_{k=1}^K\left|E_{\cap}\cro{\varphi_{SM}^{-1}(\ac{k})\cup\varphi_{PM}^{-1}(\ac{k})}\right|}\nonumber \\
& = {\bar q^{|E_{\cap}|+|E_{\Delta}|/2} \over K^{|V^{\star}_{\cup}|-|U^{\star(1)}|-|U^{\star(2)}|}} \sum_{\varphi_{\Delta}\in [K]^{\#\mathrm{CC}_{\Delta}}} \sum_{\varphi_{PM}\in [K]^{V^{\star}_{PM}}}  \pa{\bar p\over \bar q}^{\sum_{k=1}^K\left|E_{\cap}\cro{\varphi_{SM}^{-1}(\ac{k})\cup\varphi_{PM}^{-1}(\ac{k})}\right|}. \label{eq:restrictedLB2}
\end{align}

\paragraph{Conclusion.}
Combining \eqref{eq:cross:formula}, \eqref{eq:restrictedLB1} and \eqref{eq:restrictedLB2}, we then conclude that
\begin{align*}
{ \bbE\cro{P_{G^{(1)}, \pi^{(1)}}P_{G^{(2)}, \pi^{(2)}}} \over \prod_{j=1,2}\bbE\cro{P_{G^{(j)}, \pi^{(j)}}^2}^{1/2}} 
& \leq \pa{\lambda^{|E_{\Delta}|}\over \bar p^{|E_{\Delta}|/2} K^{(|U^{\star(1)}|+|U^{\star(2)}|)/2}} \wedge \pa{\lambda^{|E_{\Delta}|}\over \bar q^{|E_{\Delta}|/2} K^{|U^{\star(1)}|+|U^{\star(2)}|}}\\
&\leq \pa{\lambda^{|E_{\Delta}|/2}\over K^{(|U^{\star(1)}|+|U^{\star(2)}|)/2}} \wedge \pa{\lambda^{|E_{\Delta}|}\over \bar q^{|E_{\Delta}|/2} K^{|U^{\star(1)}|+|U^{\star(2)}|}},
\end{align*}
where the last line follows from $\bar p= \lambda+q (1-q-2\lambda)\geq \lambda$ when $q+2\lambda \leq 1$. Since $|U^{\star(1)}|+|U^{\star(2)}|\geq |U^{(1)}|+|U^{(2)}|$, this concludes the proof of~\eqref{eq:christophe1}. 


\medskip 

Let us turn to~\eqref{eq:christophe2}. Since $G$ is connected, neither $v_1$ nor $v_2$ are isolated.
We know that 
\[
\mathbb{E}[P_{G,\pi}]= \frac{\lambda^{|E|}}{K^{|V|-1}}\ . 
\]
Then, we lower bound the second moment $\mathbb{E}[P^2_{G,\pi}]$ by relying on~\eqref{eq:moment2} and considering two subfamilies of assignments as previously: one subfamily where the $z_i$'s are identical on the vertices of $G$ and one where the $z_i$'s are let arbitrary. This respectively leads us to 
 $\mathbb{E}[P^2_{G,\pi}]\geq \bar{p}^{|E|}K^{-|V|+ 1}$ and $\mathbb{E}[P^2_{G,\pi}]\geq \bar{q}^{|E|}$. Hence
\begin{align*}
0\leq { \bbE\cro{P_{G, \pi}} \over\bbE\cro{P_{G, \pi}^2}^{1/2}} 
& \leq \pa{\lambda^{|E|}\over \bar p^{|E|/2} K^{(|V|-1)/2}} \wedge \pa{\lambda^{|E|}\over \bar q^{|E|/2} K^{|V|-1}}\\
&\leq \pa{\lambda^{|E|/2}\over K^{(|V|-1)/2}} \wedge \pa{\lambda^{|E|}\over \bar q^{|E|/2} K^{(|V|-1)}}\enspace ,
\end{align*}
where we used again that $\bar{p}\geq \lambda$.
This concludes the proof of Proposition~\ref{prop:christophe}.

%

\section{Expectation and variance of $m$-clique counting: proof of Lemma~\ref{lem:cliques}}\label{sec:S}

Let us prove Lemma~\ref{lem:cliques}.
Without loss of generality, we assume henceforth that $(i,j)=(1,2)$. We recall that $G=(V,E)$ is a clique on $V=\ac{v_{1},\ldots,v_{m}}$, where we have removed the edge $(v_{1},v_{2})$, and 
$$S_{12}=\sum_{\pi \in \Pi_{1,2}} P_{G,\pi}(Y),$$
where $\Pi_{1,2}$ is the set of injections $\pi$ from $V$ to $\ac{1\ldots,n}$ such that $\pi(v_{1})=1$ and $\pi(v_{2})=2$.
We will  repeatedly  use the identity
$$|E|={m(m-1)\over 2}-1={(m+1)(m-2)\over 2}.$$

We first compute the mean and upper bound the variance under $\bbP_{12}=\bbP\cro{\cdot|z_{1}=z_{2}}$, and then under
$\bbP_{\not 12}=\bbP\cro{\cdot|z_{1}\neq z_{2}}$. 

\paragraph{Mean under $\bbP_{12}$.}
We have 
\begin{align*}
\bbE_{12}\cro{P_{G,\pi}(Y)}&= \lambda^{|E|} \,\bbE_{12}\cro{\prod_{(i,j)\in \pi(E)}\mathbf 1_{z_{i}=z_{j}}}\\
&=\lambda^{(m+1)(m-2)\over 2} \,\bbP_{12}\cro{z_{i}=z_{1}:i=3,\ldots,m}={\lambda^{(m+1)(m-2)\over 2} \over K^{m-2}}.
\end{align*}
Hence
\begin{equation*}
\bbE_{12}\cro{S_{12}}={(n-2)!\over (n-m)!} \pa{\lambda^{m+1\over 2} \over K}^{m-2}.
\end{equation*}

\paragraph{Variance under $\bbP_{12}$.}
Let $\pi^{(1)},\pi^{(2)}\in \Pi_{12}$ such that $|\text{range}(\pi^{(1)})\cap \text{range}(\pi^{(2)})|=2+u$, with $u\in \ac{0,\ldots,m-2}$.
Following~\eqref{eq:proof:ingredient:basis}, we have 
$$\bbE_{12}\cro{P_{G,\pi^{(1)}}P_{G,\pi^{(2)}}}=\lambda^{|E_{\Delta}|} \bar q^{|E_{\cap}|} \bbE_{12}\cro{\prod_{(i,j)\in E_{\Delta}}\mathbf 1_{z_{i}=z_{j}}\prod_{\substack{(i,j)\in E_{\cap}\\ z_{i}=z_{j}}}{\bar p\over \bar q}}.$$
\smallskip

\noindent\underline{Case $u=0$.} For $u=0$, we have  $E_{\cap}=\emptyset$. We also have $|V_{\Delta}|=|V_{\cup}|=2m-2$ and $|E_{\Delta}|=2|E|=(m+1)(m-2)$, so,
  when $u=0$,
\begin{align}
\bbE_{12}\cro{P_{G,\pi^{(1)}}P_{G,\pi^{(2)}}} & = \lambda^{(m+1)(m-2)} \bbP_{12}\cro{z_{i}=z_{1}:i=3,\ldots, 2m-2}\nonumber\\
& = {\lambda^{(m+1)(m-2)} \over K^{2m-4}}=
\bbE_{12}\cro{P_{G,\pi^{(1)}}}\bbE_{12}\cro{P_{G,\pi^{(2)}}}.\label{eq:var:u=0}
\end{align}
\smallskip

\noindent\underline{Case $1\leq u\leq m-3$.}
We now have $|V_{\Delta}|=|V_{\cup}|=2m-2-u$, $|E_{\cap}|=(u+3)u/2$, and $|E_{\Delta}|=2(|E|-|E_{\cap}|)=(m+1)(m-2)-(u+3)u$. Hence
\begin{align}
\bbE_{12}\cro{P_{G,\pi^{(1)}}P_{G,\pi^{(2)}}} & = \lambda^{(m+1)(m-2)-(u+3)u} {\bar q}^{(u+3)u\over 2} 
\pa{\bar p\over \bar q}^{(u+3)u\over 2} 
\bbP_{12}\cro{z_{i}=z_{1}:i=3,\ldots, 2m-2-u}\nonumber\\
&= {\lambda^{(m+1)(m-2)} \over K^{2m-4}}
\pa{\bar p\over \lambda^2}^{(u+3)u\over 2} {K}^{u} \leq {\bar p^{(m+1)(m-2)-{(u+3)u\over 2}} \over K^{2m-4}}
{K}^{u}\enspace ,\label{eq:var:generic:u}
\end{align}
since $\bar{p}= \bar q + \lambda (1-2q)\geq \lambda$ when $2\lambda+q\leq 1$.

\noindent\underline{Case $u= m-2$.} When $u=m-2$ then $\text{range}(\pi^{(1)})= \text{range}(\pi^{(2)})$
and $E_{\Delta}=\emptyset$. We define $V_{k}(z):=\ac{i\in\ac{1,\ldots,m}: z_{i}=k}$.
We have
\begin{align*}
\bbE_{12}\cro{P_{G,\pi^{(1)}}P_{G,\pi^{(2)}}}& = \bar q^{|E|} \bbE_{12}\cro{\prod_{k=1}^{K}\prod_{\substack{(i,j)\in \pi^{(1)}(E)\\ z_{i}=z_{j}=k}}{\bar p\over \bar q}}
 =  \bar q^{(m+1)(m-2)\over 2} \bbE_{12}\cro{\prod_{k=1}^{K}\prod_{\substack{1\leq i<j\leq m\\ (i,j)\neq (1,2)\\ z_{i}=z_{j}=k}}{\bar p\over \bar q}}\\
& =  \bar q^{(m+1)(m-2)\over 2} \bbE_{12}\cro{\pa{\bar p\over \bar q}^{{1\over 2}\sum_{k=1}^K |V_{k}(z)|(|V_{k}(z)|-1)-1}}\\
& =   \bar q^{(m+1)(m-2)\over 2} \bbE_{12}\cro{\pa{\bar p\over \bar q}^{{1\over 2}\sum_{k=1}^K |V_{k}(z)|^2-{m+2\over 2}}}.
\end{align*}
 Let $\ell(z):=|\ac{k:V_{k}(z)\neq \emptyset}|$. We have $\ell(z)\leq m-1$ under $\bbP_{12}$, since $z_{1}=z_{2}$ a.s.
For $\ell(z)=\ell^\star$, the sum $\sum_{k=1}^K |V_{k}(z)|^2$ is maximized  when $\ell^\star-1$ groups have 1 node and the remaining group has $m-(\ell^\star-1)$ nodes, so 
$$\sum_{k=1}^K |V_{k}(z)|^2\leq \ell(z)-1+(m-(\ell(z)-1))^2=m^2-(\ell(z)-1)(2m-\ell(z)).$$
Since  $\bar p\geq \bar q$ (since $q\leq 1/2$), we deduce
\begin{align}
\bbE_{12}\cro{P_{G,\pi^{(1)}}P_{G,\pi^{(2)}}}
& \leq   \bar q^{(m+1)(m-2)\over 2} \pa{\bar p\over \bar q}^{(m+1)(m-2)\over 2} \bbE_{12}\cro{\pa{\bar q\over \bar p}^{{(\ell(z)-1)(2m-\ell(z))\over 2}}}\nonumber\\
& \leq   \bar p^{(m+1)(m-2)\over 2} \,\bbE_{12}\cro{\pa{\bar q\over \bar p}^{(\ell(z)-1){(m+1)\over 2}}}\nonumber\\
& \leq   \bar p^{(m+1)(m-2)\over 2} \sum_{\ell=1}^{m-1}\pa{\bar q\over \bar p}^{(\ell-1){(m+1)\over 2}}\bbP_{12}\cro{\ell(z)=\ell}.\label{eq:var:u=m-2:intermediate}
\end{align}
We observe that, for $\ell\leq m-1\leq K$,
\begin{align*}
\bbP_{12}\cro{\ell(z)=\ell}& \leq \binom{m-2}{\ell-1} {K(K-1)\cdots (K-\ell+1)\over K^{\ell}} \pa{\ell \over K}^{m-2-(\ell-1)} \leq  \binom{m-2}{\ell-1}\pa{m-2 \over K}^{m-2-(\ell-1)},
\end{align*}
so plugging in \eqref{eq:var:u=m-2:intermediate}, we get
\begin{align}
\bbE_{12}\cro{P_{G,\pi^{(1)}}P_{G,\pi^{(2)}}}
&\leq   \bar p^{(m+1)(m-2)\over 2} \sum_{\ell=1}^{m-1} \binom{m-2}{\ell-1}\pa{m-2 \over K}^{m-2-(\ell-1)} \pa{\bar q\over \bar p}^{(\ell-1){(m+1)\over 2}}\nonumber \\
&\leq  \bar p^{(m+1)(m-2)\over 2} \pa{{m-2 \over K}+\pa{\bar q\over \bar p}^{{m+1\over 2}}}^{m-2}.
\label{eq:var:u=m-2}
\end{align}

\noindent\underline{Combining the three terms.}
Combining \eqref{eq:var:u=0},\eqref{eq:var:generic:u}, and \eqref{eq:var:u=m-2}
\begin{align*}
\text{var}_{12}(S_{12}) &= \sum_{u=0}^{m-2} \sum_{\substack{\pi^{(1)},\pi^{(2)}\in \Pi_{12}\\ |\text{range}(\pi^{(1)})\cap \text{range}(\pi^{(2)})|=2+u}}
\pa{\bbE_{12}\cro{P_{G,\pi^{(1)}}P_{G,\pi^{(2)}}}-\bbE_{12}\cro{P_{G,\pi^{(1)}}}\bbE_{12}\cro{P_{G,\pi^{(2)}}}}\\
& \leq   \sum_{u=1}^{m-2}   \binom{m-2}{u}^{2} {(n-2)!u!\over (n-2m+u+2)!}  \bbE_{12}\cro{P_{G,\pi^{(1)}}P_{G,\pi^{(2)}}}\\
& \leq {(n-2)! (m-2)!\over (n-m)!} \bar p^{(m+1)(m-2)\over 2} \pa{{m-2 \over K}+\pa{\bar q\over \bar p}^{{m+1\over 2}}}^{m-2} \\ &\quad 
+\sum_{u=1}^{m-3}   \binom{m-2}{u}^{2} {(n-2)!u!\over (n-2m+u+2)!} {\bar p^{(m+1)(m-2)-{(u+3)u\over 2}} \over K^{2m-4-u}}
, 
\end{align*}
so
\begin{align*}
 \text{var}_{12}(S_{12}) & \leq {(n-2)! (m-2)!\over (n-m)!} \bar p^{(m+1)(m-2)\over 2} \pa{{m-2 \over K}+\pa{\bar q\over \bar p}^{{m+1\over 2}}}^{m-2} \\
&\quad +{(n-2)!(m-2)!\over (n-m)!}{\bar p^{(m+1)(m-2)} \over K^{2m-4}}\sum_{u=1}^{m-3}   \binom{m-2}{u} {(n-m)!\over (m-2-u)!(n-2m+u+2)!} \pa{K\over \bar p^{{m\over 2}} }^{u}.
\end{align*}
We have 
\begin{align*}
\sum_{u=1}^{m-3}   \binom{m-2}{u} {(n-m)!\over (m-2-u)!(n-2m+u+2)!} \pa{K\over \bar p^{{m+1\over 2}} }^{u}
&\leq (n-m)^{m-2} \sum_{u=1}^{m-3}  \binom{m-2}{u}  \pa{K\over (n-m)\bar p^{m+1\over 2}}^{u}\\
&\leq  \pa{n-m+{K\over \bar p^{m+1\over 2}}}^{m-2}.
\end{align*}
So,
\begin{align*}
 \text{var}_{12}(S_{12}) & \leq {(n-2)! (m-2)!\over (n-m)!} \bar p^{(m+1)(m-2)\over 2} \pa{{m-2 \over K}+\pa{\bar q\over \bar p}^{{m+1\over 2}}}^{m-2} \\
&\quad +{(n-2)!(m-2)!\over (n-m)!}{\bar p^{(m+1)(m-2)} \over K^{2m-4}}\pa{n-m+{K\over \bar p^{m+1\over 2}}}^{m-2}\bar p^{1/2}.
\end{align*}
We have proved~\eqref{eq:var:12}.

\paragraph{Mean under $\bbP_{\not 12}$.}
We have
\begin{equation*}
\bbE_{\not 12}\cro{P_{G,\pi}(Y)}= \lambda^{|E|} \,\bbE_{\not 12}\cro{\prod_{(i,j)\in \pi(E)}\mathbf 1_{z_{i}=z_{j}}}=0,
\end{equation*}
since $z_{1}\neq z_{2}$ a.s. under  $\bbP_{\not 12}$.

\paragraph{Variance under $\bbP_{\not 12}$.}
Let $\pi^{(1)},\pi^{(2)}\in \Pi_{12}$ such that $|\text{range}(\pi^{(1)})\cap \text{range}(\pi^{(2)})|=2+u$, with $u\in \ac{0,\ldots,m-2}$. Following~\eqref{eq:proof:ingredient:basis}, we have 
$$\bbE_{\not 12}\cro{P_{G,\pi^{(1)}}P_{G,\pi^{(2)}}}=\lambda^{|E_{\Delta}|} \bar q^{|E_{\cap}|} \bbE_{\not 12}\cro{\prod_{(i,j)\in E_{\Delta}}\mathbf 1_{z_{i}=z_{j}}\prod_{\substack{(i,j)\in E_{\cap}\\ z_{i}=z_{j}}}{\bar p\over \bar q}}.$$
\smallskip

\noindent\underline{Case $0\leq u\leq m-3$.}
We have $1,2\in V_{\Delta}$ and $|E_{\Delta}|=(m+1)(m-2)-(u+3)u\geq 4$, so 
$\prod_{(i,j)\in E_{\Delta}}\mathbf 1_{z_{i}=z_{j}}=0$ a.s. under $\bbP_{\not 12}$ since 
 $z_{1}\neq z_{2}$.
 Hence $\bbE_{\not 12}\cro{P_{G,\pi^{(1)}}P_{G,\pi^{(2)}}}=0$ in this case.
 \smallskip

\noindent\underline{Case $u= m-2$.} When $u=m-2$ then $\text{range}(\pi^{(1)})= \text{range}(\pi^{(2)})$
and $E_{\Delta}=\emptyset$. Following the same lines as for deriving \eqref{eq:var:u=m-2:intermediate}, we get
$$\bbE_{\not 12}\cro{P_{G,\pi^{(1)}}P_{G,\pi^{(2)}}} 
=   \bar q^{(m+1)(m-2)\over 2} \bbE_{\not 12}\cro{\pa{\bar p\over \bar q}^{{1\over 2}\sum_{k=1}^K |V_{k}(z)|^2-{m\over 2}}}
 \leq   \bar p^{(m+1)(m-2)\over 2} \,  \bbE_{\not 12}\cro{\pa{\bar q\over \bar p}^{{(\ell(z)-1)(2m-\ell(z))\over 2}{-1}}}.$$
We have $2\leq \ell(z)\leq m$ a.s. under $\bbP_{\not 12}$, since $z_{1}\neq z_{2}$ a.s.
Hence

\begin{align*}
\bbE_{\not 12}\cro{P_{G,\pi^{(1)}}P_{G,\pi^{(2)}}}
& \leq   \bar p^{(m+1)(m-2)\over 2} \sum_{\ell=2}^{m-1}\pa{\bar q\over \bar p}^{(\ell-1){(m+1)\over 2}{-1}}\bbP_{\not 12}\cro{\ell(z)=\ell}\\
&\quad + \bar p^{(m+1)(m-2)\over 2}\pa{\bar q\over \bar p}^{{(m-1)m\over 2}{-1}}\bbP_{\not 12}\cro{\ell(z)=m}
\end{align*}
Arguing as in the previous variance bound, we observe that, for $\ell\leq m\leq K$, we have 
\begin{align*}
\bbP_{\not 12}\cro{\ell(z)=\ell}&  \leq  \binom{m-2}{\ell-2}\pa{m-2 \over K}^{m-2-(\ell-2)},
\end{align*}
so  we get
\begin{align*}
\bbE_{\not 12}&\cro{P_{G,\pi^{(1)}}P_{G,\pi^{(2)}}}\\
&\leq   \bar p^{(m+1)(m-2)\over 2} \sum_{\ell=2}^{m-1} \binom{m-2}{\ell-2}\pa{m-2 \over K}^{m-2-(\ell-2)} \pa{\bar q\over \bar p}^{(\ell-1){(m+1)\over 2}{-1}}+ \bar p^{(m+1)(m-2)\over 2}\pa{\bar q\over \bar p}^{{(m-2)(m+1)\over 2}}\\
&\leq  \bar p^{(m+1)(m-2)\over 2} \pa{{m-2 \over K}+\pa{\bar q\over \bar p}^{{m+1\over 2}}}^{m-2}\pa{\bar q\over \bar p}^{{m-1\over 2}}+ \bar p^{(m+1)(m-2)\over 2}\pa{\bar q\over \bar p}^{{(m-2)(m+1)\over 2}}.
\end{align*}

\noindent\underline{Final bound on the variance.}

\begin{align}
\text{var}_{\not 12}(S_{12}) &=  \sum_{\substack{\pi^{(1)},\pi^{(2)}\in \Pi_{12}\\ \text{range}(\pi^{(1)})= \text{range}(\pi^{(2)})}}
\pa{\bbE_{\not 12}\cro{P_{G,\pi^{(1)}}P_{G,\pi^{(2)}}}-\bbE_{\not 12}\cro{P_{G,\pi^{(1)}}}\bbE_{\not 12}\cro{P_{G,\pi^{(2)}}}}\nonumber\\
& \leq {(n-2)! (m-2)!\over (n-m)!} \bar p^{(m+1)(m-2)\over 2} \pa{\pa{{m-2 \over K}+\pa{\bar q\over \bar p}^{{m+1\over 2}}}^{m-2} \pa{\bar q\over \bar p}^{{m-1\over 2}}+ \pa{\bar q\over \bar p}^{{(m-2)(m+1)\over 2}} }\enspace . \nonumber
\end{align}
We  have proved~\eqref{eq:var:not12}.

\section{Counting self-avoiding paths: Proof of Lemma~\ref{lem:paths}}\label{sec:lem:paths}

\subsection{Results under the distribution $\bbP_{12}$}

\paragraph{Mean under $\bbP_{12}$.}
We have
\begin{equation*}
\bbE_{12}\cro{P_{G,\pi}(Y)}= \lambda^{|E|} \,\bbE_{ 12}\cro{\prod_{(i,j)\in \pi(E)}\mathbf 1_{z_{i}=z_{j}}}= \frac{1}{K^{|V|-2}} \lambda^{|E|} = \frac{1}{K^{m-2}} \lambda^{m-1},
\end{equation*}
since $z_{1}= z_{2}$ a.s. under  $\bbP_{12}$.

\paragraph{Variance under $\bbP_{12}$.} As for the proof of Lemma~\ref{lem:cliques}, we consider $\pi^{(1)},\pi^{(2)}\in \Pi_{12}$ such that $|\text{range}(\pi^{(1)})\cap \text{range}(\pi^{(2)})|=2+u$, with $u\in \ac{0,\ldots,m-2}$.
Following~\eqref{eq:proof:ingredient:basis}, we have 
$$\bbE_{12}\cro{P_{G,\pi^{(1)}}P_{G,\pi^{(2)}}}=\lambda^{|E_{\Delta}|} \bar q^{|E_{\cap}|} \bbE_{12}\cro{\prod_{(i,j)\in E_{\Delta}}\mathbf 1_{z_{i}=z_{j}}\prod_{\substack{(i,j)\in E_{\cap}\\ z_{i}=z_{j}}}{\bar p\over \bar q}}.$$
\smallskip

\noindent\underline{Case $u=0$.} For $u=0$, we have   $E_{\cap}=\emptyset$. We also have $|V_{\Delta}|=|V_{\cup}|=2m-2$ and $|E_{\Delta}|=2(m-1)$, so,
  when $u=0$,
\begin{align} \nonumber
\bbE_{12}\cro{P_{G,\pi^{(1)}}P_{G,\pi^{(2)}}} & = \lambda^{2(m-1)} \bbP_{12}\cro{z_{i}=z_{1}:i=3,\ldots, 2m-2}\nonumber\\
& = {\lambda^{2(m-1)} \over K^{2m-4}}= \label{eq:upper_cross_u:0}
\bbE_{12}\cro{P_{G,\pi^{(1)}}}\bbE_{12}\cro{P_{G,\pi^{(2)}}}\enspace .
\end{align}
\smallskip

\noindent\underline{Case $1\leq u\leq m-3$.}
We now have $|V_{\cup}|=2m-2-u$.
By definition of $u$, there are $m-u-2$ unmatched nodes in $\pi^{(2)}(V)\setminus \pi^{(1)}(V)$. Note that all the nodes in $\pi^{(2)}(V)\setminus \pi^{(1)}(V)$ belong to $G_{\Delta}$. Since the line graph $G$ is connected, all the nodes in $\pi^{(2)}(V)\setminus \pi^{(1)}(V)$ are connected through $G_{\Delta}$ to a node in $\pi^{(1)}(V)\cap \pi^{(2)}(V)$. Define $E_{\Delta}^{(2)}\subset E_{\Delta}$ as the subset of edges $(i,j)$ such that either $i\in \pi^{(2)}(V)\setminus \pi^{(1)}(V)$ or $j\in \pi^{(2)}(V)\setminus \pi^{(1)}(V)$. Then, we have 
\[
\bbE_{12}\left[\prod_{(i,j)\in E^{(2)}_{\Delta} } \mathbf 1_{z_{i}=z_{j}}  | z_{\pi^{(1)}(V)}  \right] \leq \frac{1}{K^{m-u-2}} \ , 
\]
almost surely. It then follows that 
\[
\bbE_{12}\cro{P_{G,\pi^{(1)}}P_{G,\pi^{(2)}}}\leq \lambda^{|E_{\Delta}|} \bar q^{|E_{\cap}|}\frac{1}{K^{m-u-2}}
 \bbE_{12}\cro{\prod_{(i,j)\in E_{\Delta}\setminus  E^{(2)}_{\Delta} }\mathbf 1_{z_{i}=z_{j}}\prod_{\substack{(i,j)\in E_{\cap}\\ z_{i}=z_{j}}}{\bar p\over \bar q}}.
 \]
Then, we consider $E_{\Delta}^{(1)}= E_{\Delta} \cap (\pi^{(1)}(V)\times \pi^{(1)}(V))\subset E_{\Delta}\setminus  E^{(2)}_{\Delta}$ as the subset of edges
that only arise through the labeling $\pi^{(1)}$ of $G$. We have 
\begin{align}\label{eq:upper_E_12}
\bbE_{12}\cro{P_{G,\pi^{(1)}}P_{G,\pi^{(2)}}}\leq \lambda^{|E_{\Delta}|} \bar q^{|E_{\cap}|}\frac{1}{K^{m-u-2}}
 \bbE_{12}\cro{\prod_{(i,j)\in E^{(1)}_{\Delta} }\mathbf 1_{z_{i}=z_{j}}\prod_{\substack{(i,j)\in E_{\cap}\\ z_{i}=z_{j}}}{\bar p\over \bar q}}.
\end{align}
Note that $E^{(1)}_{\Delta}$ and $E_{\cap}$ form a partition of the edges of the labelled line graph on $\pi^{(1)}(V)$. If $E_{\cap}=\emptyset$, we arrive at 
\[
\bbE_{12}\cro{P_{G,\pi^{(1)}}P_{G,\pi^{(2)}}}\leq \lambda^{|E_{\Delta}|} \frac{1}{K^{2m-u-4}}= \lambda^{2(m-1)} \frac{1}{K^{2m-u-4}} \enspace . 
\]
If $E_{\cap}\neq \emptyset$, then the graph $G_{\cap}$ is a collection of $t\geq 1$ connected components with $|V_{\cap}|=|E_{\cap}|+t \geq 2$ nodes in total. 
The term $\prod_{(i,j)\in E^{(1)}_{\Delta} }\mathbf 1_{z_{i}=z_{j}}$ is not equal to $0$ if and only if some constraints on the communities $z_i$ of the nodes  are satisfied. Namely, this sets the constraint that the nodes from each connected component of $G_\Delta^{(1)}$  must belong to the same community. Also, by definition of $\mathbb{P}_{12}$, we have $z_1=z_2$. Furthermore, by invariance of the choice of the communities, we can assume that $z_1$ is fixed. By considering separately the cases where $V_{\cap}\cap \{1,2\}=\emptyset, \{1\}, \{2\}, \{1,2\}$, we deduce that $m-2-|V_{\cap}|+t+1= m-1-|E_{\cap}|$ nodes in $G$ have their communities fixed by the constraints, and only $|E_{\cap}|-1$ can still be assigned freely - namely their label can still be chosen freely without setting the term $\prod_{(i,j)\in E^{(1)}_{\Delta} }\mathbf 1_{z_{i}=z_{j}}$ to $0$. This set of nodes - that we write $V_u$ - is composed of (i) exactly one node pro connected component in the graph $G_\Delta^{(1)}$  that one can choose freely and (ii) of all perfectly matched nodes, i.e.~the nodes that are not in $G_\Delta^{(1)}$, except $v_1$ and $v_2$. This yields
\[
\bbE_{12}\cro{\prod_{(i,j)\in E^{(1)}_{\Delta} }\mathbf 1_{z_{i}=z_{j}}\prod_{\substack{(i,j)\in E_{\cap}\\ z_{i}=z_{j}}}{\bar p\over \bar q}}\leq
\left(\frac{1}{K}\right)^{m-1-|E_{\cap}|}\mathbb{E}_{12}\left[\prod_{\substack{(i,j)\in E_{\cap}\\ z_{i}=z_{j}}}{\bar p\over \bar q}~~~\Big| \prod_{(i,j)\in E^{(1)}_{\Delta} }\mathbf 1_{z_{i}=z_{j}} = 1\right]\ .
\]
Write $E_\cap^{=} = \{(i,j)\in E_{\cap}: z_{i}=z_{j}\}$. Now we consider the labels of the $|E_{\cap}|-1$ nodes in $V_u$ that have their labels unconstrained. We order them as $(v'_1,\ldots, v'_{|E_{\cap}|-1})$ according to their order in the path $v_1,v_3, \ldots, v_m, v_2$ so that either (i) node $v_i'$ is connected to a connected component of $G_\Delta^{(1)}$ to which $v_{i-1}'$ belongs, or (ii) $v_i'$ is connected to $v'_{i-1}$ (with the convention $v'_0=v_1$). The communities of these nodes are sampled independently which leads to the set $E_\cap^{=}$ of identical edges. If the community of $v_i'$ is the same as that of $v_{i-1}'$ (resp.~of $v_1$ for $i=1$) - which happens with probability $1/K$ - then one edge is added in $E_\cap^{=}$ - namely the edge connecting $v_i'$ with either the neighboring  connected component represented by $v_{i-1}'$, or with $v_{i-1}'$ itself. Finally, the last edge in $E_{\cap}$ belongs to  $E_{\cap}^{=}$, if $v'_{|E_{\cap}|-1}$ and $v_2$ (and therefore $v_1$) belong to the same community. As a consequence, we can decompose $E_{\cap}^{=}$ as $E_{\cap}^{=}=N_1+N_2$ where $N_1\sim \mathrm{Bin}(|E_{\cap}|-1,1/K)$ and $N_2\in \{0,1\}= \mathbf{1}\{z_{v'_{|E_{\cap}|-1}}=z_{v_1}\}$. One easily checks that $\mathbb{P}_{12}[N_2=1|N_1=|E_{\cap}|-1]=1$ and that  $\mathbb{P}_{12}[N_2=1|N_1=|E_{\cap}|-2]=0$. Next, for any $a\in \{1,\ldots,|E_{\cap}|-3\}$, conditionally to $N_1=a$, we have $\mathbb{P}_{12}[N_2=1|N_1=|E_{\cap}|-2]\leq 1/(K-1)$. Indeed, if the community of second-to last node is the same as that of $v_1$, then this probability is equal to $0$, whereas, if the community of second-to last node differs from that of $v_1$, this probability is equal to $1/(K-1)$.  As a consequence, we get
\[\mathbb{E}_{12}\left[\prod_{\substack{(i,j)\in E_{\cap}\\ z_{i}=z_{j}}}{\bar p\over \bar q}~~~\Big| \prod_{(i,j)\in E^{(1)}_{\Delta} }\mathbf 1_{z_{i}=z_{j}} = 1\right] = \mathbb{E}_{12}\left[\left({\bar p\over \bar q}\right)^{|E_\cap^{=}|}~~~\Big| \prod_{(i,j)\in E^{(1)}_{\Delta} }\mathbf 1_{z_{i}=z_{j}} = 1\right] \leq \mathbb{E}\left[ \left({\bar p\over \bar q}\right)^{N_1+N_2} \right] .\]
We have to bound the exponential moment of $N_1+ N_2$. With probability at most $(K-1)^{-(|E_\cap|-1)}$, $N_1+N_2$ is equal to $|E_{\cap}|$. Conditionally to $N_1 < |E_{\cap}|-1$, $N_1+N_2$ is stochastically upper bounded by a Binomial distribution with parameters $|E_{\cap}|$ and $1/(K-1)$. 
For $b\in \mathbb R$, we therefore derive that  $\mathbb E[e^{a(N_1+N_2)}] \leq  \mathbb E[e^{bN_3}] + (K-1)^{-(|E_\cap|-1)} e^{b|E_\cap|}$, where $N_3 \sim  \mathrm{Bin}(|E_{\cap}|,1/(K-1))$. As a consequence, we obtain
\begin{align*}
\bbE_{12}\cro{\prod_{\substack{(i,j)\in E_{\cap}\\ z_{i}=z_{j}}}{\bar p\over \bar q}~~~\Big| \prod_{(i,j)\in E^{(1)}_{\Delta} }\mathbf 1_{z_{i}=z_{j}} = 1}  \leq
\left[\left(\frac{1}{K}\right)^{|E_{\cap}|-1}\left(\frac{\bar p}{\bar q}\right)^{|E_{\cap}|} + \exp\left(|E_{\cap}|\log\left(\frac{\bar{p}}{(K-1)\bar{q}}+1-\frac{1}{K-1}\right)\right)\right]\enspace ,
\end{align*}
 Coming back to~\eqref{eq:upper_E_12}, we get
\begin{align*}
\bbE_{12}\cro{P_{G,\pi^{(1)}}P_{G,\pi^{(2)}}}&\leq \lambda^{|E_{\Delta}|} \bar p^{|E_{\cap}|}\frac{1}{K^{2m-u-4}}+ \lambda^{|E_{\Delta}|} \frac{1}{K^{2m-u-3}} \bar q^{|E_{\cap}|} \left(2\frac{\bar{p}}{\bar{q}}+   K \right)^{|E_{\cap}|}\\
&\leq \frac{\lambda^{2(m-1)}}{K^{2m-4}}\left[K^{u} \left(\frac{\bar p}{\lambda^2}\right)^{|E_{\cap}|}+ K^{u-1}\left[\left(4 \frac{\bar{p}}{\lambda^2}\right)^{|E_{\cap}|}+ \left[\frac{2K\bar{q}}{\lambda^2} \right]^{|E_{\cap}|}  \right]\right]\enspace . 
\end{align*}

Next, we observe that, if $u\leq m-3$, we have $|E_\cap|\in [1,u]$, whereas, for $u=m-2$, we have $|E_{\cap}|\in [1,m-1]$. This yields
\begin{align}\label{eq:upper_cross_u:middle}
\bbE_{12}\cro{P_{G,\pi^{(1)}}P_{G,\pi^{(2)}}}&\leq 
 \frac{\lambda^{2(m-1)}}{K^{2m-4}}\left[ 2\left(4K\frac{\bar p}{\lambda^2}\right)^{u}+ 
 \frac{1}{K}\left(\frac{2K^2\bar{q}}{\lambda^2} \right)^{u} + 3K^u \right]\enspace \ , 
\end{align}
for $u\leq m-3$ and 
\begin{align}\label{eq:upper_cross_u:max}
\bbE_{12}\cro{P_{G,\pi^{(1)}}P_{G,\pi^{(2)}}}&\leq 
 \frac{\lambda^{2(m-1)}}{K^{2m-4}}\left[ \frac{2}{K} \left(\frac{4K\bar p}{\lambda^2}\right)^{m-1}+ 
 \frac{1}{K^2}\left(\frac{2K^2\bar{q}}{\lambda^2} \right)^{m-1} + 3K^{m-2} \right]\enspace \ , 
\end{align}
for $u=m-2$.

\noindent\underline{Combining the three terms.}
Combining \eqref{eq:upper_cross_u:0}, \eqref{eq:upper_cross_u:middle}, and~\eqref{eq:upper_cross_u:max}, we obtain 
\begin{align*}
\text{var}_{12}(R_{12}) &= \sum_{u=0}^{m-2} \sum_{\substack{\pi^{(1)},\pi^{(2)}\in \Pi_{12}\\ |\text{range}(\pi^{(1)})\cap \text{range}(\pi^{(2)})|=2+u}}
\pa{\bbE_{12}\cro{P_{G,\pi^{(1)}}P_{G,\pi^{(2)}}}-\bbE_{12}\cro{P_{G,\pi^{(1)}}}\bbE_{12}\cro{P_{G,\pi^{(2)}}}}\\
& \leq   \sum_{u=1}^{m-2}   \binom{m-2}{u}^{2} {(n-2)!u!\over (n-2m+u+2)!}  \bbE_{12}\cro{P_{G,\pi^{(1)}}P_{G,\pi^{(2)}}}\\
&\leq \left(\frac{(n-2)!}{(n-m)!}\right)^2\sum_{u=1}^{m-2} (m-2)^{2u}  \frac{1}{(n-m)^u}\bbE_{12}\cro{P_{G,\pi^{(1)}}P_{G,\pi^{(2)}}} \\ 
&\leq \left(\frac{(n-2)!}{(n-m)!}\right)^2 \frac{\lambda^{2(m-1)}}{K^{2m-4}}\Bigg[\sum_{u=1}^{m-3} \left( 2\left(4Km^2\frac{\bar p}{ (n-m) \lambda^2}\right)^{u}+  \frac{1}{K}\left(\frac{2m^2K^2\bar{q}}{(n-m) \lambda^2} \right)^{u} + 3\left(m^2\frac{K}{n-m}\right)^u \right)   
\\ & \quad \quad\quad \quad  + m^{2m-4}\left( \frac{2(n-m)}{K} \left(\frac{4K\bar p}{(n-m)\lambda^2}\right)^{m-1}+ \frac{n-m}{K^2}\left(\frac{2K^2\bar{q}}{(n-m)\lambda^2} \right)^{m-1} + 3\left(\frac{K}{n-m}\right)^{m-2}\right)\Bigg]\\
&\leq \left(\frac{(n-2)!}{(n-m)!}\right)^2  \frac{\lambda^{2(m-1)}}{K^{2m-4}}m \Bigg[\frac{8Km^2 \bar p}{ (n-m) \lambda^2}+  \frac{2m^2K\bar{q}}{(n-m) \lambda^2} + 3\frac{m^2 K}{n-m}+ 3\left(\frac{m^2 K}{n-m}\right)^{m-2}\\ &\hspace{3cm} +  \frac{2nm^2}{K} \left(\frac{4Km^2\bar p}{(n-m)\lambda^2}\right)^{m-1} + \frac{n m^2}{K^2}\left(\frac{2m^2K^2\bar{q}}{(n-m)\lambda^2} \right)^{m-1}
\Bigg] \ . 
\end{align*}
This concludes the proof of~\eqref{eq:var:12:path}.

\subsection{Results under the distribution $\bbP_{\not 12}$}

\paragraph{Mean under $\bbP_{\not 12}$.}
We have
\begin{equation*}
\bbE_{\not 12}\cro{P_{G,\pi}(Y)}= \lambda^{|E|} \,\bbE_{\not 12}\cro{\prod_{(i,j)\in \pi(E)}\mathbf 1_{z_{i}=z_{j}}}=0,
\end{equation*}
since $z_{1}\neq z_{2}$ a.s. under  $\bbP_{\not 12}$.

\paragraph{Variance under $\bbP_{\not 12}$.}
Let $\pi^{(1)},\pi^{(2)}\in \Pi_{12}$ such that $|\text{range}(\pi^{(1)})\cap \text{range}(\pi^{(2)})|=2+u$, with $u\in \ac{0,\ldots,m-2}$. Following~\eqref{eq:proof:ingredient:basis}, we again have 
$$\bbE_{\not 12}\cro{P_{G,\pi^{(1)}}P_{G,\pi^{(2)}}}=\lambda^{|E_{\Delta}|} \bar q^{|E_{\cap}|} \bbE_{\not 12}\cro{\prod_{(i,j)\in E_{\Delta}}\mathbf 1_{z_{i}=z_{j}}\prod_{\substack{(i,j)\in E_{\cap}\\ z_{i}=z_{j}}}{\bar p\over \bar q}}.$$
\smallskip

If $u=0$, we have $E_{\cap}=\emptyset$ and since $z_1\neq z_2$, we deduce that 
$\bbE_{\not 12}\cro{P_{G,\pi^{(1)}}P_{G,\pi^{(2)}}}=0$. 

 \medskip

\noindent\underline{Case $1\leq u\leq m-2$.}
First, if $E_{\cap}=\emptyset$, we easily derive as above that $\bbE_{\not 12}\cro{P_{G,\pi^{(1)}}P_{G,\pi^{(2)}}}=0$. Hence, we assume henceforth that $|E_{\emptyset}|\geq 1$.
Arguing as for $\bbE_{12}$ in the previous subsection, we derive from~\eqref{eq:upper_E_12} that 
\begin{align*}
\bbE_{\not 12}\cro{P_{G,\pi^{(1)}}P_{G,\pi^{(2)}}}\leq \lambda^{|E_{\Delta}|} \bar q^{|E_{\cap}|}\frac{1}{K^{m-u-2}}
 \bbE_{\not 12}\cro{\prod_{(i,j)\in E^{(1)}_{\Delta} }\mathbf 1_{z_{i}=z_{j}}\prod_{\substack{(i,j)\in E_{\cap}\\ z_{i}=z_{j}}}{\bar p\over \bar q}}\ , 
\end{align*}
and we also have in a similar way as for $\bbE_{12}$ in the previous subsection
\[
\bbE_{\not 12}\cro{\prod_{(i,j)\in E^{(1)}_{\Delta} }\mathbf 1_{z_{i}=z_{j}}\prod_{\substack{(i,j)\in E_{\cap}\\ z_{i}=z_{j}}}{\bar p\over \bar q}}\leq
\left(\frac{1}{K}\right)^{m-1-|E_{\cap}|}\mathbb{E}\left[ \left({\bar p\over \bar q}\right)^{N'_1+N'_2} \right]\ , 
\]
where $N'_1\sim \mathrm{Bin}(|E_{\cap}|-1,1/K)$ and $N'_2$ satisfies $\bbP[N'_2=1|N_1=|E_{\cap}|-1]=0$ (as now $v_1$ and $v_2$ do not have the same labels) and $\bbP [N'_2=1|N'_1<|E_{\cap}|-1]\leq  \frac{1}{K-1}$, so that $N'_1+N'_2$ is stochastically dominated by Binomial distribution $\mathrm{Bin}(|E_{\cap}|,1/(K-1))$. 
This yields 
\begin{align*}
\bbE_{\not 12}\cro{P_{G,\pi^{(1)}}P_{G,\pi^{(2)}}}&\leq \lambda^{|E_{\Delta}|} \frac{1}{K^{2m-u-3}} \bar q^{|E_{\cap}|} \left(2\frac{\bar{p}}{\bar{q}}+   K \right)^{|E_{\cap}|}\\
&\leq \frac{\lambda^{2(m-1)}}{K^{2m-4}}\cdot  K^{u-1}\left[\left(4 \frac{\bar{p}}{\lambda^2}\right)^{|E_{\cap}|}+ \left[\frac{2K\bar{q}}{\lambda^2} \right]^{|E_{\cap}|}  \right]\enspace . 
\end{align*}
Recall that $|E_{\cap}|\in [1,u]$ if $u\leq m-3$ and $|E_{\cap}|\in [1,u+1]$ for $u=m-2$. This yields 
\begin{align}\label{eq:upper_cross_u:middle:not12}
\bbE_{\not 12}\cro{P_{G,\pi^{(1)}}P_{G,\pi^{(2)}}}&\leq 
 \frac{\lambda^{2(m-1)}}{K^{2m-4}}\left[ \frac{1}{K}\left(4K\frac{\bar p}{\lambda^2}\right)^{u}+ 
 \frac{1}{K}\left(\frac{2K^2\bar{q}}{\lambda^2} \right)^{u} + 2K^{u-1} \right]\enspace \ , 
\end{align}
for $u\leq m-3$ and 
\begin{align}\label{eq:upper_cross_u:max:not12}
\bbE_{\not 12}\cro{P_{G,\pi^{(1)}}P_{G,\pi^{(2)}}}&\leq 
 \frac{\lambda^{2(m-1)}}{K^{2m-4}}\left[ \frac{1}{K^2} \left(\frac{4K\bar p}{\lambda^2}\right)^{m-1}+ 
 \frac{1}{K^2}\left(\frac{2K^2\bar{q}}{\lambda^2} \right)^{m-1} + 2K^{m-2} \right]\enspace \ , 
\end{align}
for $u=m-2$.

\noindent\underline{Final bound on the variance.}
Combining \eqref{eq:upper_cross_u:middle:not12} and~\eqref{eq:upper_cross_u:max:not12},  we obtain 
\begin{align*}
\text{var}_{\not 12}(T_{12}) &= \sum_{u=0}^{m-2} \sum_{\substack{\pi^{(1)},\pi^{(2)}\in \Pi_{12}\\ |\text{range}(\pi^{(1)})\cap \text{range}(\pi^{(2)})|=2+u}}
\pa{\bbE_{\not 12}\cro{P_{G,\pi^{(1)}}P_{G,\pi^{(2)}}}-\bbE_{\not 12}\cro{P_{G,\pi^{(1)}}}\bbE_{\not 12}\cro{P_{G,\pi^{(2)}}}}\\
&\leq \sum_{u=1}^{m-2}  \binom{m-2}{u}^{2} {(n-2)!u!\over (n-2m+u+2)!}\bbE_{\not 12}\cro{P_{G,\pi^{(1)}}P_{G,\pi^{(2)}}} \\ 
&\leq \left(\frac{(n-2)!}{(n-m)!}\right)^2  \frac{\lambda^{2(m-1)}}{K^{2m-4}}\Bigg[\sum_{u=1}^{m-3} \left(\frac{1}{K} \left(4Km^2\frac{\bar p}{ (n-m) \lambda^2}\right)^{u}+  \frac{1}{K}\left(\frac{2m^2K^2\bar{q}}{(n-m) \lambda^2} \right)^{u} + 2\left(m^2\frac{K}{n-m}\right)^u \right)   
\\ & \quad \quad\quad \quad \quad \quad + m^{2m-4}\left( \frac{n}{K^2} \left(\frac{4K\bar p}{(n-m)\lambda^2}\right)^{m-1}+ \frac{n}{K^2}\left(\frac{2K^2\bar{q}}{(n-m)\lambda^2} \right)^{m-1} + 2\left(\frac{K}{n-m}\right)^{m-2}\right)\Bigg]\\
&\leq \left(\frac{(n-2)!}{(n-m)!}\right)^2 \frac{\lambda^{2(m-1)}}{K^{2m-4}}m \Bigg[\frac{4m^2 \bar p}{ (n-m) \lambda^2}+  \frac{2m^2K\bar{q}}{(n-m) \lambda^2}\\ 
&\hspace{1.5cm}+ 2m^2\frac{K}{n-m}+ 2\left(\frac{mK}{n-m}\right)^{m-2}+  \frac{nm^2}{K} \left(\frac{4Km^2\bar p}{(n-m)\lambda^2}\right)^{m-1}  + \frac{n m^2}{K^2}\left(\frac{2m^2K^2\bar{q}}{(n-m)\lambda^2} \right)^{m-1}
\Bigg] \ . 
\end{align*}
We  have proved~\eqref{eq:var:not12:path}.

\section{Proof of Proposition~\ref{prop:key:blow-up}}\label{sec:key:blow-up}

In this proof, we need additional notation. Let $V_{\mathrm{cyc},0}$ denote the set of cycle nodes  having the same community as $v_1$ and $v_2$. For $k=1,\dots,I$ let $V_{\mathrm{cyc},k}$ be the set of cycle nodes of community $k$. For short, we write 
$V_{\mathrm{cyc},\neq 0}=\cup_{k=1}^IV_{\mathrm{cyc},k}$. 

We consider separately cycle edges and fastener edges. For $k=0,\ldots, I$, we write $E^{\uparrow}_{\mathrm{cyc},k}\subset E_{\mathrm{cyc}}$ for the collection of cycle edges that are incident to a node in $V_{\mathrm{cyc},k}$ and a node from a different community --- these edges will be called \textit{boundary} edges. For $k=0,\ldots, I$, we define $V_{\mathrm{fst},k} := V_{\mathrm{cyc},k} \cap V_{\mathrm{fst}}$, the subset of fastener nodes within community $k$. Note that $|V_{\mathrm{fst},k}|$ also corresponds to the number of fastener edges within community $k$. For short, we also write $V_{\mathrm{fst},\neq 0}=\cup_{k=1}^I V_{\mathrm{fst},k}$.

\bigskip

First, we consider the specific  and trivial case where $I=1$ and $V_{\mathrm{cyc},1}=V_{\mathrm{cyc}}$. Then, we obviously have 
\[
|E^{\neq}|= |V_{\mathrm{fst}}|=a\kappa\gamma \geq \gamma+1 \geq (\gamma+a)I\ , 
\] 
since $\kappa\geq 2/a$.  We assume henceforth $V_{\mathrm{cyc},k}\neq V_{\mathrm{cyc}}$ for any $k$.

The proof mainly divides into the two following lemmas. The first one states that the boundary edges of any community that does not contain $v_1$ or $v_2$ --- namely the sets of edges $E_{\mathrm{cyc},k}^{\uparrow}$ for $k \in \{1, \ldots, I\}$ --- is of size at least $2\gamma$.

\begin{lemma}\label{lem:Ecyc_k}
If $|V_{\mathrm{cyc},k}| < \kappa \gamma$, i.e. $V_{\mathrm{cyc},k}\neq V_{\mathrm{cyc}}$, we have 
\[
|E_{\mathrm{cyc},k}^{\uparrow}|\geq 2\gamma\enspace . 
\]
\end{lemma}
The second lemma lower bounds the sum of twice the number of fasteners $V_{\mathrm{fst},\neq 0}$ that do not belong to the same community as $v_1$ or $v_2$, plus the number of boundary edges for the community of $v_1,v_2$, namely $E_{\mathrm{cyc},0}^{\uparrow}$. If the nodes $V_{\mathrm{cyc},\neq 0}$ were sampled uniformly, then $|V_{\mathrm{fst},\neq 0}|$ would be of the order of $a|V_{\mathrm{cyc},\neq 0}|$. However, the number of fasteners in $V_{\mathrm{fst},\neq 0}$ can be much smaller for unfavorable configurations such that fastener nodes are in the same community as $v_1$ or $v_2$. Nevertheless, if this happens, then the number of boundary edges in $E_{\mathrm{cyc},0}^{\uparrow}$ must be high accordingly to compensate this drop in the number of fasteners in $V_{\mathrm{fst},\neq 0}$. 
\begin{lemma}\label{lem:eq:out_0}
    We have 
 \begin{equation}\label{eq:objective_fastener}
|E_{\mathrm{cyc},0}^{\uparrow}|+ 2 |V_{\mathrm{fst},\neq 0}| \geq 2 a |V_{\mathrm{cyc},\neq 0}| \geq 2a I\enspace . 
\end{equation}
\end{lemma}

Let us explain how Proposition~\ref{prop:key:blow-up} follows from Lemmas~\ref{prop:key:blow-up} and~\ref{lem:Ecyc_k}.
First, we observe that  an edge between two communities is either a fastener edge and is therefore incident to $V_{\mathrm{fst}, \neq 0}$ or is a cycle edge and therefore both arises in $E_{\mathrm{cyc},k}^{\uparrow}$ and $E_{\mathrm{cyc},k'}^{\uparrow}$ for $k\neq k'$ being the two communities to which the nodes of the edge belong.
Hence, we have the following decomposition
\[
2|E^{\neq}|\geq \sum_{k=0}^{I} |E_{\mathrm{cyc},k}^{\uparrow}| + 2|V_{\mathrm{fst}, \neq 0}|\ ,
\]
and  Lemmas~\ref{lem:Ecyc_k} and~\ref{lem:eq:out_0} ensure  that 
\[
2|E^{\neq}|\geq 2\gamma I + 2a I \ , 
\]
which concludes the proof of the Proposition~\ref{prop:key:blow-up}.

\begin{proof}[Proof of Lemma~\ref{lem:Ecyc_k}]
We consider three cases, 
(i) at least a layer does not intersect $V_{\mathrm{cyc},k}$; (ii)  each layer intersects  $V_{\mathrm{cyc},k}$ and $|V_{\mathrm{cyc},k}| \leq (\kappa-1) \gamma$; (iii)   each layer intersects  $V_{\mathrm{cyc},k}$ and $|V_{\mathrm{cyc},k}| > (\kappa-1) \gamma$. 

\medskip

\noindent
\underline{\bf Case (i): at least a layer does not intersect $V_{\mathrm{cyc},k}$.} In this case there exist two layers $i_1$ and $i_2$ (not necessarily distinct) such that 
\[
L_{i_1}\cap V_{\mathrm{cyc},k}\neq \emptyset;\quad  L_{i_1-1}\cap V_{\mathrm{cyc},k}= \emptyset;\quad  L_{i_2}\cap V_{\mathrm{cyc},k}\neq \emptyset;\quad  L_{i_2+1}\cap V_{\mathrm{cyc},k}= \emptyset ,
\]
with the convention that $L_0=L_{\kappa}$ and $L_{\kappa+1}= L_1$. 
There are $\gamma$ edges between $L_{i_1-1}$ and one node of $V_{\mathrm{cyc},k}$ in $L_{i_1}$. Similarly, there are $\gamma$ edges between $L_{i_2+1}$ and one node of $V_{\mathrm{cyc},k}$ in $L_{i_2}$. Since $\kappa \geq 3$, these edges are distinct and we deduce $|E_{\mathrm{cyc},k}^{\uparrow}|\geq 2\gamma$. 

\medskip

\noindent
\underline{\bf Case (ii): each layer intersects  $V_{\mathrm{cyc},k}$, with $|V_{\mathrm{cyc},k}| \leq (\kappa-1) \gamma$.} Since in this case each layer intersects $V_{\mathrm{cyc},k}$, it follows from the definition of the graph that each node $v$ in $V_{\mathrm{cyc}}\setminus V_{\mathrm{cyc},k}$ is connected to $V_{\mathrm{cyc},k}$ by at least two edges corresponding to nodes of $V_{\mathrm{cyc},k}$ in the previous layer and in the following layer of that of $v$. As a consequence, 
\[
|E_{\mathrm{cyc},k}^{\uparrow}| \geq 2|V_{\mathrm{cyc}}\setminus V_{\mathrm{cyc},k}|= 2[\kappa \gamma- |V_{\mathrm{cyc},k}|]\geq 2\gamma\ , 
\]
as soon as $|V_{\mathrm{cyc},k}| \leq (\kappa-1) \gamma$.

\medskip

\noindent
\underline{\bf Case (iii): each layer intersects  $V_{\mathrm{cyc},k}$, with $|V_{\mathrm{cyc},k}| > (\kappa-1) \gamma$.}
It remains to consider the case where $|V_{\mathrm{cyc},k}|\in [(\kappa-1)\gamma+1, \kappa\gamma-1]$ which is non empty only if $\gamma>1$. Then, we consider the complementary set $V'= V_{\mathrm{cyc}}\setminus V_{\mathrm{cyc},k}$ which satisfies $|V'|\leq\gamma-1\leq \kappa(\gamma-1)$. Since $V_{\mathrm{cyc},k}$ and $V'$ have a symmetric role for counting $|E_{\mathrm{cyc},k}^{\uparrow}|$, 
 we can apply the same arguments as in Case~(i) and~(ii), switching the role of $V_{\mathrm{cyc},k}$ and $V'$. Accordingly,  $|E_{\mathrm{cyc},k}^{\uparrow}|\geq 2\gamma$, and the proof of Lemma~\ref{lem:Ecyc_k} is complete.

\end{proof}

\begin{proof}[Proof of Lemma~\ref{lem:eq:out_0}]

Write $\mathcal{L}_0= \{\omega\in [\kappa]: L_{\omega}\cap V_{\mathrm{cyc},0}\neq \emptyset\}$ the collection of layers that intersect $V_{\mathrm{cyc},0}$. To start with, we consider three cases (a) $\mathcal{L}_0=\emptyset$, (b) $\mathcal{L}_0=[\kappa]$, and (c) $|\mathcal{L}_0|\in [2;\kappa-1]$.

\noindent
\underline{\bf Case (a): $\mathcal{L}_0=\emptyset$.} In this case no fastener node belongs to the community of $v_1, v_2$, therefore we have 
\[
|V_{\mathrm{fst},\neq 0}|=  |V_{\mathrm{fst}}|= a |V_{\mathrm{cyc}}|\geq a  |V_{\mathrm{cyc},\neq 0}| \ , 
\]
and~\eqref{eq:objective_fastener} holds.

\medskip 

\noindent
\underline{\bf Case (b): $\mathcal{L}_0=[\kappa]$.} In this case we observe as in the proof of Lemma~\ref{lem:Ecyc_k} (case (ii)), that any node $v$ in $V_{\mathrm{cyc}}\setminus V_{\mathrm{cyc},0}$ is connected to two nodes in $V_{\mathrm{cyc},0}$, one in the previous layer and one in the following layer. Thus, we get 
\[
|E_{\mathrm{cyc},0}^{\uparrow}|\geq 2 \sum_{k=1}^I |V_{\mathrm{cyc},k}| =  2|V_{\mathrm{cyc},\neq 0}| \geq 2a|V_{\mathrm{cyc},\neq 0}|\ , 
\]
and~\eqref{eq:objective_fastener} holds.

\medskip 

\noindent
\underline{\bf Case (c): $|\mathcal{L}_0|\in [2;\kappa-1]$.} This is the intermediary case between cases (a) and (b). Consider $[\kappa]$ as the nodes of a cycle graph and $\mathcal{L}_0$ as a subset of nodes of this graph. Write $\#\mathrm{CC}(\mathcal L_0)\geq 1$ for the number of connected components of the subgraph induced by $\mathcal{L}_0$. 
\smallskip

\noindent
\underline{Step 1: lower bound of $|E_{\mathrm{cyc},0}^{\uparrow}|$.} We partition $\mathcal{L}_0$ into $\mathcal{L}_{0,\mathrm{in}}$, $\mathcal{L}_{0,\mathrm{bd}}$, and $\mathcal{L}_{0,\mathrm{sg}}$ where $\mathcal{L}_{0,\mathrm{in}}$ are internal layers of $\mathcal{L}_0$ that is, they are not at the boundary of a connected component of $\mathcal{L}_0$, $\mathcal{L}_{0,\mathrm{sg}}$ correspond to connected components that are singletons, and $\mathcal{L}_{0,\mathrm{bd}}$ stands for all the the remaining boundary layers. 

Fix any $\omega\in \mathcal{L}_{0,\mathrm{in}}$ and any node $v\in L_\omega\setminus V_{\mathrm{cyc},0}$. This node $v$ is connected by at least two edges to $V_{\mathrm{cyc},0}$ (at least one node in the previous layer and one node in the following layer). Fix any $\omega\in\mathcal{L}_{0,\mathrm{bd}}$. For any node $v\in L_\omega\setminus V_{\mathrm{cyc},0}$, we also observe that $v$ is connected by at least one edge to $V_{\mathrm{cyc},0}$. Also, for any $v\in L_{\omega}\cap V_{\mathrm{cyc},0}$, there are $\gamma$ edges between $v$ and nodes in a layer that does not belong to $\mathcal{L}_0$. Finally, we consider any $\omega\in \mathcal{L}_{0,\mathrm{sg}}$  and any $v\in L_{\omega}\cap V_{\mathrm{cyc},0}$. This node is connected to $2\gamma$ nodes outside the layers in $\mathcal{L}_0$. Gathering these bounds, we get
\begin{align}\nonumber
|E_{\mathrm{cyc},0}^{\uparrow}|&\geq 2 \sum_{\omega\in\mathcal{L}_{0,\mathrm{in}}}|L_{\omega}\setminus V_{\mathrm{cyc},0}|+ \sum_{\omega\in\mathcal{L}_{0,\mathrm{bd}}}[|L_{\omega}\setminus V_{\mathrm{cyc},0}| + \gamma]  +  2\sum_{\omega\in\mathcal{L}_{0,\mathrm{sg}}}\gamma  \\
& \geq 2 \sum_{\omega\in\mathcal{L}_{0}}|L_{\omega}\setminus V_{\mathrm{cyc},0}| + 2\#\mathrm{CC}(\mathcal L_0)\  , \label{eq:lower:E:cyc,0,up}
\end{align}
where we used that, for any $\omega\in \mathcal{L}_{0}$, $|L_{\omega}\setminus V_{\mathrm{cyc},0}|\leq \gamma - 1$, and where we recall that $\#\mathrm{CC}(\mathcal L_0)$ is the number of connected components of $\mathcal{L}_0$. 

\smallskip

\noindent
\underline{Step 2: lower bound of $|V_{\mathrm{fst},\neq 0}|$.} We now focus on the set $[\kappa]\setminus \mathcal{L}_0$ of layers that do not intersect $V_{\mathrm{cyc},0}$. We start from 
\[
|V_{\mathrm{fst},\neq 0}|\geq \sum_{\omega\in [\kappa]\setminus \mathcal{L}_0} |V_{\mathrm{fst}}\cap L_{\omega}|.
\]
Although $|V_{\mathrm{fst}}\cap L_{\omega}|\in [\lfloor a\gamma\rfloor, \lceil a\gamma\rceil]$ for each $\omega$, we have to be slightly more careful to control $|V_{\mathrm{fst},\neq 0}|$. The set $[\kappa]\setminus \mathcal{L}_0$ decomposes into $\#\mathrm{CC}(\mathcal L_0)$ connected components, since the set $\mathcal L_0$ decomposes in $\#\mathrm{CC}(\mathcal L_0)$ connected components. We write $r_1,\ldots, r_{\#\mathrm{CC}(\mathcal L_0)}$ for their respective number of layers. Into the $i$-th connected components, there are at least $\lfloor a r_i\gamma\rfloor$ fasteners. From this and from $\sum_i r_i \geq \kappa - |\mathcal L_0|$, we deduce that 
\begin{align}\nonumber
|V_{\mathrm{fst},\neq 0}|&\geq \sum_{i=1}^{\#\mathrm{CC}(\mathcal L_0)} \lfloor a r_i \gamma\rfloor \geq a \left[\kappa - |\mathcal{L}_0|\right]\gamma - \#\mathrm{CC}(\mathcal L_0)\\
&= a \sum_{\omega\in [\kappa]\setminus \mathcal L_0}|L_{\omega}\setminus V_{\mathrm{cyc},0}|   - \#\mathrm{CC}(\mathcal L_0)\enspace . \label{eq:fastener_open}
\end{align}
Combining~\eqref{eq:lower:E:cyc,0,up} and \eqref{eq:fastener_open}, we obtain 
\begin{align*}
|E_{\mathrm{cyc},0}^{\uparrow}|+ 2 |V_{\mathrm{fst},\neq 0}| &\geq  2a \left[\sum_{\omega\in [\kappa]}|L_{\omega}\setminus V_{\mathrm{cyc},0}| \right] =  2a \left[\sum_{\omega\in [\kappa]}|L_{\omega}\cap V_{\mathrm{cyc},\neq 0}| \right]\\
& \geq 2a |V_{\mathrm{cyc},\neq 0}|\ ,
\end{align*}
which concludes the proof of Lemma~\ref{lem:eq:out_0}.
\end{proof}

\section{Proof of Proposition~\ref{prop:mean:variance:blow_up}}\label{sec:mean:variance:blow_up}

This proposition is a straightforward consequence of the following result. 

\begin{proposition}\label{lem:blow_up}
Assume that $q\leq 1/4$, $q+2\lambda \leq 1$, $n\geq 2\kappa\gamma +4$, and set $\bar p=\bar q+\lambda(1-2q)$. We also assume that $a\kappa \gamma$ is an even integer and $\kappa\geq 3\vee 2/a$. 
Then, for any $1\leq i< j\leq n$, we have 
\begin{align}
\bbE_{ij}\cro{R_{ij}}&={(n-2)!\over (n-\kappa\gamma-2)!} \left({\lambda^{ \gamma + a } \over K}\right)^{\kappa \gamma }, \label{eq:mean:12:blow-up}\\
\bbE_{\not ij}\cro{R_{ij}}&= 0, \label{eq:mean:not12:blow-up}\\
 \mathrm{var}_{\not ij}(R_{ij})\vee \mathrm{var}_{ij}(R_{ij}) & \leq \bbE^2_{12}[R_{12}](\kappa\gamma)^2 \Bigg[\frac{2(\kappa \gamma)^{3} K^2 } {n}  \left(\frac{\bar{q}}{\lambda^2}\vee \frac{\bar{q}}{\bar p } \right)^{\gamma + a}  + \frac{2(\kappa \gamma)^{3} K } {n} \left(\frac{\bar{p}}{\lambda^2}\vee 1 \right)^{\gamma + a} \nonumber  \\& \quad \quad  + \left[\frac{2(\kappa \gamma)^{3} K^2 } {n}   \left(\frac{\bar{q}}{\lambda^2}\vee \frac{\bar{q}}{\bar p } \right)^{\gamma + a}\right]^{\kappa \gamma}  + \left[\frac{2(\kappa \gamma)^{3} K } {n} \left(\frac{\bar{p}}{\lambda^2}\vee 1 \right)^{\gamma + a}\right]^{\kappa \gamma}   \Bigg] . 
\label{eq:var:12:blow-up}
\end{align}
\end{proposition}

\subsection{Notation and preliminaries}

To control the expectation and the variance of $R_{ij}$, we need to introduce some graph notation as in Appendix~A of~\cite{carpentier2025phase}.
For a motif $G=(V,E)$ and a labeling $\pi\in\Pi_{ij}$, we define the labeled graph $\pi(G)$ as the graph with node set $\ac{\pi(v):v\in V}$ and edge set $\ac{(\pi(v),\pi(v')):(v,v')\in E}$. Given two labelings $\pi^{(1)}$ and $\pi^{(2)}\in \Pi_{ij}$, the labeled merged graph $G_{\cup}=(V_{\cup},E_{\cup})$ is defined as the union of $\pi^{(1)}(G^{(1)})$ and $\pi^{(2)}(G^{(2)})$, with the convention that two same edges are merged into a single edge. Similarly,  we define the  intersection graph  $G_{\cap}=(V_{\cap},E_{\cap})$ and the  symmetric difference graph $G_{\Delta}=(V_{\Delta},E_{\Delta})$ so that $E_{\Delta}=E_{\cup}\setminus E_{\cap}$. Here, $V_{\cap}$ (resp. $V_{\Delta}$) is the set of nodes induced by the edges $E_{\cap}$ (resp. $E_{\Delta}$) so that $G_{\cap}$ (resp. $G_{\Delta}$) does not contain any isolated node.

The two following expressions are the starting point of our analysis. They easily derive from the definition~\eqref{eq:definition:Y} of $Y$:
\begin{align}
 \bbE\cro{P_{G, \pi}}&=    \bbE\cro{\prod_{(v,v')\in E }Y_{\pi(v)\pi(v')}}   \label{eq:proof:ingredient:basis:mean}
    =   \bbE\cro{\prod_{(v,v')\in E} (\lambda \mathbf{1}_{z_{\pi(v)}=z_{\pi(v')}})},\\
   \bbE\cro{P_{G, \pi^{(1)}}P_{G, \pi^{(2)}}} 
   &   = \bbE\cro{\prod_{(i,j)\in E_{\Delta}}Y_{ij}\prod_{(i,j)\in E_{\cap}}Y_{ij}^2}   \nonumber
    =   \bbE\cro{\prod_{(i,j)\in E_{\Delta}} (\lambda \mathbf{1}_{z_{i}=z_{j}}) \prod_{(i,j)\in E_{\cap}}\bar q \pa{\bar p\over \bar q}^{ \mathbf{1}_{z_{i}=z_{j}}}}.
\end{align}

Without loss of generality,  we assume in the following that $(i,j)=(1,2)$.

\subsection{Results under the distribution $\bbP_{12}$}

In what follows, we consider the graph $G := G_{\kappa,\gamma,a}=(V,E)$.  We will study respectively the expectation and variance of the associated polynomial, both under $\bbP_{12}$ and under $\bbP_{\not 12}$.

\paragraph{Mean under $\bbP_{12}$.}
Since the graph $G$ is connected, we have from Equation~\eqref{eq:proof:ingredient:basis:mean}
\begin{equation}\label{eq:mean:12:pi:blow-up}
\bbE_{12}\cro{P_{G,\pi}(Y)}= \lambda^{|E|} \,\bbE_{ 12}\cro{\prod_{(i,j)\in \pi(E)}\mathbf 1_{z_{i}=z_{j}}}= \frac{1}{K^{|V|-2}} \lambda^{|E|} = \frac{1}{K^{\kappa \gamma}} \lambda^{(\gamma+a)\kappa\gamma},
\end{equation}
since $z_{1}= z_{2}$ a.s. under  $\bbP_{12}$. The identity~\eqref{eq:mean:12:blow-up} follows.

\paragraph{Variance under $\bbP_{12}$.}
Let us turn to proving~\eqref{eq:var:12:blow-up}, which is the main part of the proof. Fix $\pi^{(1)},\pi^{(2)}\in\Pi_{12}$. 
We start by controlling $\bbE_{12}\cro{P_{G,\pi^{(1)}}P_{G,\pi^{(2)}}}$. Let $u\in\{0,\dots,|V_{\mathrm{cyc}}|\}$ be the integer so that
\[
|\mathrm{range}(\pi^{(1)})\cap\mathrm{range}(\pi^{(2)})| = 2+u\enspace.
\]
Recall that  $V_{\cap},V_{\Delta},E_{\cap},E_{\Delta}$ stand for the common and symmetric-difference vertex- and edge-sets. 
Following \eqref{eq:proof:ingredient:basis}, we have 
\[
\bbE_{12}\cro{P_{G,\pi^{(1)}}P_{G,\pi^{(2)}}}
=\lambda^{|E_{\Delta}|} \bar q^{|E_{\cap}|} \,
\bbE_{12}\cro{\prod_{(i,j)\in E_{\Delta}}\mathbf 1_{z_{i}=z_{j}}
\prod_{\substack{(i,j)\in E_{\cap}\\ z_{i}=z_{j}}}{\bar p\over \bar q}}.
\]

\noindent\underline{Case $u=0$.} Then, we have $E_{\cap}=\emptyset$, $|V_{\Delta}|=|V_{\cup}|=2\kappa \gamma+2$ and $|E_{\Delta}|=2\kappa \gamma(\gamma+a)$. Therefore
\begin{align}\nonumber
\bbE_{12}\cro{P_{G,\pi^{(1)}}P_{G,\pi^{(2)}}}
&= \lambda^{2\kappa \gamma(\gamma+a)} \,\bbP_{12}\cro{z_i=z_1\ \text{for }i\in V_{\Delta}} \\
& =  \lambda^{2\kappa \gamma(\gamma+a)} K^{-2\kappa\gamma}  \nonumber \\ 
& = \bbE_{12}\cro{P_{G,\pi^{(1)}}}\bbE_{12}\cro{P_{G,\pi^{(2)}}}\enspace , \label{eq:upper_cross_u:0:blow-up}
\end{align}
by~\eqref{eq:mean:12:pi:blow-up}.

\noindent\underline{Case $1\le u\le \kappa \gamma$.} In this case $|V_{\cup}|=2\kappa \gamma+2-u$ and there are $\kappa \gamma -u$ vertices that appear only in $\pi^{(2)}(V)\setminus\pi^{(1)}(V)$. All these vertices belong to $G_{\Delta}$ and each such vertex is connected (within $G_\Delta$) to at least one vertex in the intersection $\pi^{(1)}(V)\cap\pi^{(2)}(V)$ because  $G$ is connected. Let $E_{\Delta}^{(2)}\subset E_{\Delta}$ be the set of edges of $G_{\Delta}$ that arise from $\pi^{(2)}[G]$. In particular, $E_{\Delta}^{(2)}$ contains all the edges that have at least one endpoint in $\pi^{(2)}(V)\setminus\pi^{(1)}(V)$. Conditionally on the communities  of the vertices of $\pi^{(1)}(V)$, the product of indicators on $E_{\Delta}^{(2)}$ enforces the community of the  $\kappa \gamma-u$ vertices in $\pi^{(2)}(V)\setminus\pi^{(1)}(V)$. Hence,  we have 
\[
\bbE_{12}\left[\prod_{(i,j)\in E^{(2)}_{\Delta}} \mathbf 1_{z_{i}=z_{j}} \,\Bigg|\, z_{\pi^{(1)}(V)}\right] \le \frac{1}{K^{\kappa \gamma-u}}\qquad\text{a.s.}
\]
Therefore (conditioning and removing $E^{(2)}_\Delta$), we get 
\[
\bbE_{12}\cro{P_{G,\pi^{(1)}}P_{G,\pi^{(2)}}}
\le \lambda^{|E_{\Delta}|} \bar q^{|E_{\cap}|}\frac{1}{K^{\kappa \gamma-u}}
 \bbE_{12}\cro{\prod_{(i,j)\in E_{\Delta}\setminus  E^{(2)}_{\Delta} }\mathbf 1_{z_{i}=z_{j}}\prod_{\substack{(i,j)\in E_{\cap}\\ z_{i}=z_{j}}}{\bar p\over \bar q}}.
\]
Now, define $E^{(1)}_{\Delta}:= E_{\Delta}\setminus E^{(2)}_{\Delta}$. Observe that  $E^{(1)}_{\Delta}$ together with $E_{\cap}$ form a partition of the edges of the labeled copy of $G$ coming from $\pi^{(1)}$. Given the community assignments $(z_i)_{i\in \pi^{(1)}(V)}$, denote $\ell(z)$ the number of distinct communities in $\pi^{(1)}(V)$, and $|E^{\neq}|$ the number of edges in the graph between distinct communities. We have
\begin{align}
\bbE_{12}\cro{P_{G,\pi^{(1)}}P_{G,\pi^{(2)}}} \nonumber 
&\le \lambda^{|E_{\Delta}|}\bar p^{|E_{\cap}|}\frac{1}{K^{\kappa \gamma-u}}
\;\bbE_{12}\cro{\left({\bar q \over \bar p}\right)^{|E^{\neq}|} \prod_{(i,j)\in E^{(1)}_{\Delta}}\mathbf 1_{z_{i}=z_{j}}}\\ 
&\le \lambda^{|E_{\Delta}|}\bar p^{|E_{\cap}|}\frac{1}{K^{\kappa \gamma-u}}
\;\bbE_{12}\cro{\left({\bar q \over \bar p}\right)^{(\ell(z)-1)(\gamma + a)} \prod_{(i,j)\in E^{(1)}_{\Delta}}\mathbf 1_{z_{i}=z_{j}}} ,\label{eq:up_E_12:blow-up}
\end{align}
where, importantly, we apply Proposition~\ref{prop:key:blow-up} in the second line. Consider the  graph $\overline{\mathcal{G}}^{(1)}_{\Delta}=(\pi^{(1)}(V),E^{(1)}_{\Delta})$.  In Equation~\eqref{eq:up_E_12:blow-up} above, the condition $\prod_{(i,j)\in E^{(1)}_{\Delta}}\mathbf 1_{z_{i}=z_{j}}=1$ enforces that the nodes of each of the connected components of $\overline{\mathcal{G}}^{(1)}_{\Delta}$ are in the same community. Let $\mathrm{cc}$ denote the number of connected components of $\overline{\mathcal{G}}^{(1)}_{\Delta}$ when we remove those containing $\pi^{(1)}(v_1)=1$ or  $\pi^{(1)}(v_2)=2$. Then, conditionally to $\prod_{(i,j)\in E^{(1)}_{\Delta}}\mathbf 1_{z_{i}=z_{j}}=1$, we have $(\ell(z)-1)\in [0,\mathrm{cc}]$ and for any $x\in [0,\mathrm{cc}]$, we have 
\begin{equation}\label{eq:l(z)}
\mathbb{P}_{12}\cro{\ell(z)-1= x\,\Bigg|\, \prod_{(i,j)\in E^{(1)}_{\Delta}}\mathbf 1_{z_{i}=z_{j}}=1}\leq \mathrm{cc}^{\mathrm{cc}} K^{-(\mathrm{cc}-x)} \ . 
\end{equation}
Since the condition $\prod_{(i,j)\in E^{(1)}_{\Delta}}\mathbf 1_{z_{i}=z_{j}}=1$ enforces that the nodes of each of the connected components of $\overline{\mathcal{G}}^{(1)}_{\Delta}$ are in the same community, we get
$$\mathbb{P}_{12}\cro{\prod_{(i,j)\in E^{(1)}_{\Delta}}\mathbf 1_{z_{i}=z_{j}}=1}\leq {1\over K^{\kappa\gamma+1-cc-1}}={1\over K^{\kappa\gamma-cc}}$$
and, combining with \eqref{eq:l(z)}
\begin{align*}
    \bbE_{12}\cro{\mathbf 1_{l(z)-1 = x}\left({\bar q \over \bar p}\right)^{(\ell(z)-1)(\gamma + a)} \prod_{(i,j)\in E^{(1)}_{\Delta}}\mathbf 1_{z_{i}=z_{j}}} &\leq  \frac{1}{K^{\kappa \gamma -\mathrm{cc}}}  \cdot \mathrm{cc}^{\mathrm{cc}} K^{-(\mathrm{cc}-x)} \left(\frac{\bar q }{\bar p}\right)^{x (\gamma +a)}\\ 
    &\quad\quad = \frac{1}{K^{\kappa \gamma }}  \cdot \mathrm{cc}^{\mathrm{cc}} K^{x} \left(\frac{\bar q }{\bar p}\right)^{x (\gamma +a)}\enspace .
\end{align*}    
Since $G$ is connected,  each connected component of $\overline{\mathcal{G}}^{(1)}_{\Delta}$ must contain at least one node in $\pi^{(1)}(V)\cap \pi^{(2)}(V)$. As a consequence, we have $\mathrm{cc}\leq u$. Combining Equation~\eqref{eq:up_E_12:blow-up} with the last equation, and using  $|E_{\Delta}|+ 2|E_{\cap}|= 2|E|$, we get 
\begin{align*}
\bbE_{12}\cro{P_{G,\pi^{(1)}}P_{G,\pi^{(2)}}} 
  & \leq \lambda^{|E_{\Delta}|}\bar p^{|E_{\cap}|}\frac{1}{K^{2\kappa \gamma-u}} 
  u^u \sum_{x = 0 }^{u} \left(K \left(\frac{\bar q }{\bar p}\right)^{\gamma +a}\right)^{x }\ \\
  & \leq \lambda^{2|E|}\left(\frac{\bar p}{\lambda^2}\right)^{|E_{\cap}|}\frac{1}{K^{2\kappa \gamma-u}} 
  u^u \sum_{x = 0 }^{u} \left(K \left(\frac{\bar q }{\bar p}\right)^{\gamma +a}\right)^{x }\ .
\end{align*}
\begin{lemma}\label{lem:cap}
We have 
\[
|E_{\cap}|\leq (\gamma + a ) u.
\]
\end{lemma}
\begin{proof}[Proof of Lemma~\ref{lem:cap}]
This result is a consequence of Proposition~\ref{prop:key:blow-up}. Indeed, consider a community assignment such that the $2+u$ nodes in $\pi^{(1)}(V)\cap \pi^{(2)}(V)$ belong to same community, whereas the $\kappa \gamma-u$ other nodes in $\pi^{(1)}(V)\setminus \pi^{(2)}(V)$ belong to distinct communities. Hence, for that assignment, we have $\kappa\gamma - u + 1$ communities and $|E^{\neq}|= |E|-|E_{\cap}|$. Then, Proposition~\ref{prop:key:blow-up} states that 
\[
|E|-|E_{\cap}|\geq (\gamma+a)[\kappa\gamma -u]\  .
\]
Since $|E|=(\gamma+a )\kappa \gamma$, Lemma~\ref{lem:cap} holds. 
\end{proof}

We deduce from the previous lemma that 
\begin{align}\nonumber
\bbE_{12}\cro{P_{G,\pi^{(1)}}P_{G,\pi^{(2)}}}&\leq \bbE^2_{12}\cro{P_{G,\pi}(Y)}  \left[K \left(\frac{\bar p}{\lambda^2}\vee 1\right)^{\gamma + a}\right]^u    u^u \sum_{x = 0 }^{u} \left(K \left(\frac{\bar q }{\bar p}\right)^{\gamma +a}\right)^{x }\\ 
&\leq \bbE^2_{12}\cro{P_{G,\pi}(Y)}  u^{u+1} \left[\left[K^2  \left(\frac{\bar{q}}{\lambda^2}\vee \frac{\bar{q}}{\bar p } \right)^{\gamma + a}\right]^u  + \left[K \left(\frac{\bar{p}}{\lambda^2}\vee 1 \right)^{\gamma + a}\right]^u   \right] \enspace . \label{eq:upper:covariance:u:blow-up}
\end{align}

\noindent\underline{Combining the terms.}
Combining \eqref{eq:upper_cross_u:0:blow-up} and  \eqref{eq:upper:covariance:u:blow-up}, we get 
\begin{align*}
\text{var}_{12}(R_{12}) &= \sum_{u=0}^{\kappa \gamma} \sum_{\substack{\pi^{(1)},\pi^{(2)}\in \Pi_{12}\\ |\text{range}(\pi^{(1)})\cap \text{range}(\pi^{(2)})|=2+u}}
\pa{\bbE_{12}\cro{P_{G,\pi^{(1)}}P_{G,\pi^{(2)}}}-\bbE_{12}\cro{P_{G,\pi^{(1)}}}\bbE_{12}\cro{P_{G,\pi^{(2)}}}}\\
& \leq   \sum_{u=1}^{\kappa \gamma}   \binom{\kappa \gamma}{u}^{2} {(n-2)!u!\over (n-2\kappa\gamma+u - 2)!}  \bbE_{12}\cro{P_{G,\pi^{(1)}}P_{G,\pi^{(2)}}}\\
&\leq \left(\frac{(n-2)!}{(n-\kappa \gamma-2)!}\right)^2\sum_{u=1}^{\kappa \gamma} (\kappa \gamma)^{2u}  \frac{1}{(n-\kappa \gamma-2)^u}\bbE_{12}\cro{P_{G,\pi^{(1)}}P_{G,\pi^{(2)}}} \\ 
&\leq \left(\frac{(n-2)!}{(n-\kappa \gamma-2)!}\right)^2 \bbE^2_{12}\cro{P_{G,\pi}(Y)}\sum_{u=1}^{\kappa\gamma } \frac{u^{u+1}(\kappa \gamma)^{2u} } {(n-\kappa \gamma-2)^u} \Bigg[\left[K^2  \left(\frac{\bar{q}}{\lambda^2}\vee \frac{\bar{q}}{\bar p } \right)^{\gamma + a}\right]^u  + \left[K \left(\frac{\bar{p}}{\lambda^2}\vee 1 \right)^{\gamma + a}\right]^u   \Bigg]\\
&\leq \bbE^2_{12}[R_{12}](\kappa\gamma)^2 \Bigg[\frac{2(\kappa \gamma)^{3} K^2 } {n}  \left(\frac{\bar{q}}{\lambda^2}\vee \frac{\bar{q}}{\bar p } \right)^{\gamma + a}  + \frac{2(\kappa \gamma)^{3} K } {n} \left(\frac{\bar{p}}{\lambda^2}\vee 1 \right)^{\gamma + a} \\& \quad\quad\quad \quad \quad \quad \quad \quad   + \left[\frac{2(\kappa \gamma)^{3} K^2 } {n}   \left(\frac{\bar{q}}{\lambda^2}\vee \frac{\bar{q}}{\bar p } \right)^{\gamma + a}\right]^{\kappa \gamma}  + \left[\frac{2(\kappa \gamma)^{3} K } {n} \left(\frac{\bar{p}}{\lambda^2}\vee 1 \right)^{\gamma + a}\right]^{\kappa \gamma}   \Bigg] \ . 
\end{align*}
This concludes the bound on $\mathrm{var}_{12}(R_{12})$ in Equation~\eqref{eq:var:12:blow-up}.

\subsection{Results under the distribution $\bbP_{\not 12}$}

\paragraph{Mean under $\bbP_{\not 12}$.}
We have from Equation~\eqref{eq:proof:ingredient:basis:mean}
\begin{equation*}
\bbE_{\not 12}\cro{P_{G,\pi}(Y)}= \lambda^{|E|} \,\bbE_{\not 12}\cro{\prod_{(i,j)\in \pi(E)}\mathbf 1_{z_{i}=z_{j}}}=0,
\end{equation*}
since $z_{1}\neq z_{2}$ a.s. under  $\bbP_{\not 12}$. Equation~\eqref{eq:mean:not12:blow-up} follows.

\paragraph{Variance under $\bbP_{\not 12}$.}
We follow the same approach as for  $\bbP_{12}$. 
Let $\pi^{(1)},\pi^{(2)}\in \Pi_{12}$ such that $|\text{range}(\pi^{(1)})\cap \text{range}(\pi^{(2)})|=2+u$, with $u\in \ac{0,\ldots,\kappa \gamma}$. Following~\eqref{eq:proof:ingredient:basis}, we again have 
$$\bbE_{\not 12}\cro{P_{G,\pi^{(1)}}P_{G,\pi^{(2)}}}=\lambda^{|E_{\Delta}|} \bar q^{|E_{\cap}|} \bbE_{\not 12}\cro{\prod_{(i,j)\in E_{\Delta}}\mathbf 1_{z_{i}=z_{j}}\prod_{\substack{(i,j)\in E_{\cap}\\ z_{i}=z_{j}}}{\bar p\over \bar q}}.$$
\smallskip

If $u=0$, we have $E_{\cap}=\emptyset$ and since $z_1\neq z_2$, we deduce that 
$\bbE_{\not 12}\cro{P_{G,\pi^{(1)}}P_{G,\pi^{(2)}}}=0$.  Then, for $u\in [1, \kappa \gamma]$, we argue as for $\bbE_{12}$. In particular, the counterpart of~\eqref{eq:up_E_12:blow-up} still holds up to the difference that under $\bbE_{\not 12}$, the number of communities that are distinct of $z_1$ or of $z_2$ is $\ell(z)-2$. Hence, we have 
\begin{align*}
\bbE_{\not 12}\cro{P_{G,\pi^{(1)}}P_{G,\pi^{(2)}}}
&\le \lambda^{|E_{\Delta}|}\bar p^{|E_{\cap}|}\frac{1}{K^{\kappa \gamma-u}}
\;\bbE_{\not 12}\cro{\left({\bar q \over \bar p}\right)^{(\ell(z)-2)_+(\gamma + a)} \prod_{(i,j)\in E^{(1)}_{\Delta}}\mathbf 1_{z_{i}=z_{j}}} \enspace . 
\end{align*}
Write again $\mathrm{cc}$ as the number of connected components of the graph $\overline{\mathcal{G}}^{(1)}_{\Delta}=(\pi^{(1)}(V),E^{(1)}_{\Delta})$ when we remove those containing $\pi^{(1)}(v_1)=1$ or  $\pi^{(1)}(v_2)=2$. Then, conditionally to $\prod_{(i,j)\in E^{(1)}_{\Delta}}\mathbf 1_{z_{i}=z_{j}}=1$, we have $(\ell(z)-2)\in [0,\mathrm{cc}]$ and for any $x\in [0,\mathrm{cc}]$ 
\[
\mathbb{P}_{\not 12}\cro{\ell(z)-2= x\,\Bigg|\, \prod_{(i,j)\in E^{(1)}_{\Delta}}\mathbf 1_{z_{i}=z_{j}}=1}\leq \mathrm{cc}^{\mathrm{cc}} K^{-(\mathrm{cc}-x)} \ . 
\]
Hence, we arrive as previously at 
\begin{align*}
\bbE_{\not 12}\cro{P_{G,\pi^{(1)}}P_{G,\pi^{(2)}}}&
 \leq \lambda^{2|E|}\left(\frac{\bar p}{\lambda^2}\right)^{|E_{\cap}|}\frac{1}{K^{2\kappa \gamma-u}} 
  u^u \sum_{x = 0 }^{u} \left(K \left(\frac{\bar q }{\bar p}\right)^{\gamma +a}\right)^{x }\ , 
\end{align*}
and the counterpart of~\eqref{eq:upper:covariance:u:blow-up} holds: 
\begin{align*}\nonumber
\bbE_{\not 12}\cro{P_{G,\pi^{(1)}}P_{G,\pi^{(2)}}}
&\leq \bbE^2_{12}\cro{P_{G,\pi}(Y)}  u^{u+1} \left[\left[K^2  \left(\frac{\bar{q}}{\lambda^2}\vee \frac{\bar{q}}{\bar p } \right)^{\gamma + a}\right]^u  + \left[K \left(\frac{\bar{p}}{\lambda^2}\vee 1 \right)^{\gamma + a}\right]^u   \right] \enspace . 
\end{align*}
Then, summing over all $\pi^{(1)}$ and $\pi^{(2)}$, we conclude as previously 
the bound on $\mathrm{var}_{\not 12}(R_{12})$
 in Equation~\eqref{eq:var:12:blow-up}.

\section{Proof of Lemma \ref{lem:rstar}}\label{sec:signal}

The signal condition showing up in the low-degree lower bound is of the form $\lambda\leq D^{-c}\lambda_{c}$, where
\begin{equation}\label{eq:transition}
\sup_{r\geq 1} \ac{n\lambda_{c}^{2r}\over K \lambda_{c}^r+K^2 \bar q^r} = 1\ .
\end{equation}
Lemma~\ref{lem:rstar}, stated  page~\pageref{lem:rstar}, makes explicit the phase transition in the Condition~\eqref{eq:transition}. 


\bigskip

\noindent{\bf Proof of Lemma \ref{lem:rstar}}

Let $\rho:=q/\lambda$ and
$$\phi(r):= {\lambda^r\over 1+K \rho^r}= {\lambda^{2r}\over \lambda^r+K q^r}.$$
Let $r^*$ such that $\phi'(r^*)=0$.
We have
\begin{equation}\label{eq:r*}
\rho^{r^*}= {1\over K}\times{\log(1/\lambda)\over  \log(\lambda/\rho)}:={u\over K}.
\end{equation}
We have
$$\phi(r^*)= {\lambda^{r^*}\over 1+K \rho^{r^*}}=  {\lambda^{r^*}\over 1+u},$$
so if  $\lambda_{*}$ is such that $\phi(r^*)=K/n$, then 
\begin{equation}\label{eq:lambda*}
\lambda_{*}^{r^*}={(1+u)K\over n}.
\end{equation}
Since $q=\lambda_{*} \rho$, we then get
\begin{equation}\label{eq:q*}
q^{r^*}={u(1+u)\over n},\quad \text{with}\ \ u={\log(1/\lambda_{*})\over  \log(\lambda_{*}^2/q)}.
\end{equation}
Below, we write $\lambda$ for $\lambda_{*}$
We have
$$u={r^*\log(1/\lambda)\over r^* \log(\lambda/\rho)}={\log(n/K)-\log(1+u)\over \log(K^2/n)+\log(1+1/u)},
$$
and hence
\begin{equation}\label{eq:u*}
u \log(K^2/n)+u\log(1+1/u)+\log(1+u)=\log(n/K).
\end{equation}
We seek for  a solution $u>0$ and $r^*\geq 1$ to \eqref{eq:q*} and \eqref{eq:u*}.
We denote $K=n^{(1+\delta_{n})/2}$.

\paragraph{Case $K\leq (1-\epsilon)\sqrt{n}$.} 
Then $\log(K^2/n)\leq  2\log(1-\epsilon)<0$ and  
$$-2u\log\pa{1\over 1-\epsilon}+u\log(1+1/u)+\log(1+u)\geq {1\over 2}\log(n) $$
which has no positive solution in $u$ for $n$ large, since the left hand side is bounded from above.
In this case, the supremum in~\eqref{eq:transition} is achieved for $r=1$ and the critical value is the KS threshold
$$\lambda_{c}=\lambda_{KS}= {K\over 2n} \pa{1+\sqrt{1+2nq}}.$$

\paragraph{Case $(1+\epsilon)\sqrt{n}\leq K\leq n^{1-\epsilon}$.} Hence ${2\log(1+\epsilon)\over \log(n)}\leq \delta_{n}\leq 1-\epsilon/2$. 
Then, we have
$$u={1-\delta\over 2\delta} -{u\log(1+1/u)+\log(1+u)\over \log(n)}={1-\delta\over 2\delta}-O\pa{\log(1/\delta)\over \log(n)}={1-\delta\over 2\delta}\pa{1-O\pa{\delta\log(1/\delta)\over \log(n)}}.$$
We observe that $r^*\geq 1$ iff $q\geq u(1+u)/n$,
i.e. 
\begin{equation}\label{eq:outKS}
q\geq {1-\delta^2\over 4\delta^2}\times{1-O\pa{\delta\log(1/\delta)\over \log(n)}\over n}.
\end{equation}
Writing $q=n^{-\alpha_{n}}$ and solving $q^{r^*}= u(1+u)/n$
gives 
$$n^{1-\alpha_{n}r^*}={1-\delta^2\over 4\delta^2}\times\pa{1-O\pa{\delta\log(1/\delta)\over \log(n)}}$$
and
\begin{align}
{1\over r^*}&= \alpha_{n} \pa{1-{\log((1-\delta^2)/4\delta^2)-O(\delta\log(1/\delta)/\log(n))\over \log(n)}}^{-1}\nonumber\\
&=\alpha_{n}\pa{1+{\log((1-\delta^2)/4\delta^2)\over \log(n)}+O\pa{\delta\log(1/\delta)\over \log(n)^2}}.\label{eq:r*value}
\end{align}
From \eqref{eq:lambda*}, $q^{r^*}= u(1+u)/n$ and $\log(1+u)=\log(1/\delta)=\log\log(n)$ we get
\begin{align*}
\lambda_{c}&= \pa{(1+u) n^{-(1-\delta)/2}}^{\alpha_{n}\pa{1+{\log(u(1+u))\over \log(n)}+O\pa{\pa{\log\log(n)\over \log(n)}^2}}}\\
&=n^{-\alpha_{n}(1-\delta)/2}\times\pa{(1+u)^{1+\delta}\over u^{1-\delta}}^{\alpha_{n}/2}(1+o(1))\\
&\sim q^{(1-\delta)/2} \pa{(1+\delta)^{1+\delta}\over (1-\delta)^{1-\delta}(2\delta)^{2\delta}}^{\alpha_{n}/2}.
\end{align*}

\end{document}